\documentclass{article}

\usepackage[accepted]{icml2023}
\usepackage{microtype}
\usepackage{graphicx}
\usepackage{subfigure}
\usepackage{booktabs} 

\usepackage{hyperref}
\usepackage{array}
\usepackage{multirow}

\usepackage{icml2023}

\usepackage{amsmath}
\usepackage{amssymb}
\usepackage{mathtools}
\usepackage{amsthm}

\usepackage[capitalize,noabbrev]{cleveref}

\theoremstyle{plain}
\newtheorem{theorem}{Theorem}[section]

\newtheorem{lemma}[theorem]{Lemma}

\theoremstyle{definition}
\newtheorem{definition}[theorem]{Definition}
\newtheorem{assumption}[theorem]{Assumption}
\theoremstyle{remark}

\usepackage[textsize=tiny]{todonotes}

\usepackage{color}

\icmltitlerunning{Patch-level Routing in MoE is Provably Sample-efficient for CNN}

\begin{document}

\twocolumn[
\icmltitle{Patch-level Routing in Mixture-of-Experts is Provably Sample-efficient for Convolutional Neural Networks}



\icmlsetsymbol{equal}{*}

\begin{icmlauthorlist}
\icmlauthor{Mohammed Nowaz Rabbani Chowdhury}{xxx}
\icmlauthor{Shuai Zhang}{xxx}
\icmlauthor{Meng Wang}{xxx}
\icmlauthor{Sijia Liu}{yyy,zzz}
\icmlauthor{Pin-Yu Chen}{qqq}
\end{icmlauthorlist}

\icmlaffiliation{xxx}{Department of Electrical, Computer, and Systems Engineering, Rensselaer Polytechnic Institute, NY, USA}
\icmlaffiliation{yyy}{Department of
Computer Science and Engineering, Michigan State University,
MI, USA}
\icmlaffiliation{zzz}{MIT-IBM Watson AI Lab, IBM Research, MA, USA}
\icmlaffiliation{qqq}{IBM Research, Yorktown Heights, NY,
USA}

\icmlcorrespondingauthor{Mohammed Nowaz Rabbani Chowdhury}{chowdm2@rpi.edu}
\icmlcorrespondingauthor{Meng Wang}{wangm7@rpi.edu}

\icmlkeywords{Machine Learning, ICML}

\vskip 0.3in
]



\printAffiliationsAndNotice{} 

\begin{abstract}
In deep learning, mixture-of-experts (MoE) activates one or   few experts (sub-networks) on a per-sample or per-token basis, resulting in  significant computation reduction. The recently proposed \underline{p}atch-level routing in \underline{MoE} (pMoE) divides each input into $n$ patches (or tokens)  and  sends $l$ patches ($l\ll n$) to each  expert  through prioritized routing. pMoE has demonstrated great empirical success in reducing training and inference costs while maintaining test accuracy. However, the theoretical explanation of pMoE and the general MoE remains   elusive. Focusing on a supervised classification task using a mixture of  two-layer convolutional neural networks (CNNs), we show  for the first time that pMoE provably reduces the required number of training samples to achieve desirable generalization (referred to as the sample complexity) by a factor in the polynomial order of $n/l$, and outperforms its single-expert counterpart of the same or even larger capacity. The advantage results from the discriminative routing property, which is justified in both theory and practice that pMoE routers can filter label-irrelevant patches and route similar class-discriminative patches to the same expert. Our experimental results on MNIST, CIFAR-10, and CelebA support our theoretical findings on pMoE's generalization and show that pMoE can avoid learning  spurious correlations.
\end{abstract}

\section{Introduction}
Deep learning has demonstrated exceptional empirical success in many applications at the cost of high computational and data requirements. To address this issue, mixture-of-experts (MoE) only activates partial regions of a  neural network  for each data point and significantly reduces the computational complexity of deep learning \textit{without hurting the performance} in applications such as machine translation and natural image classification \citep{shazeer2017outrageously,yang2019condconv}.
\begin{figure}[t]
\vskip 0.1in
    \centering
    \includegraphics[width=0.7\linewidth]{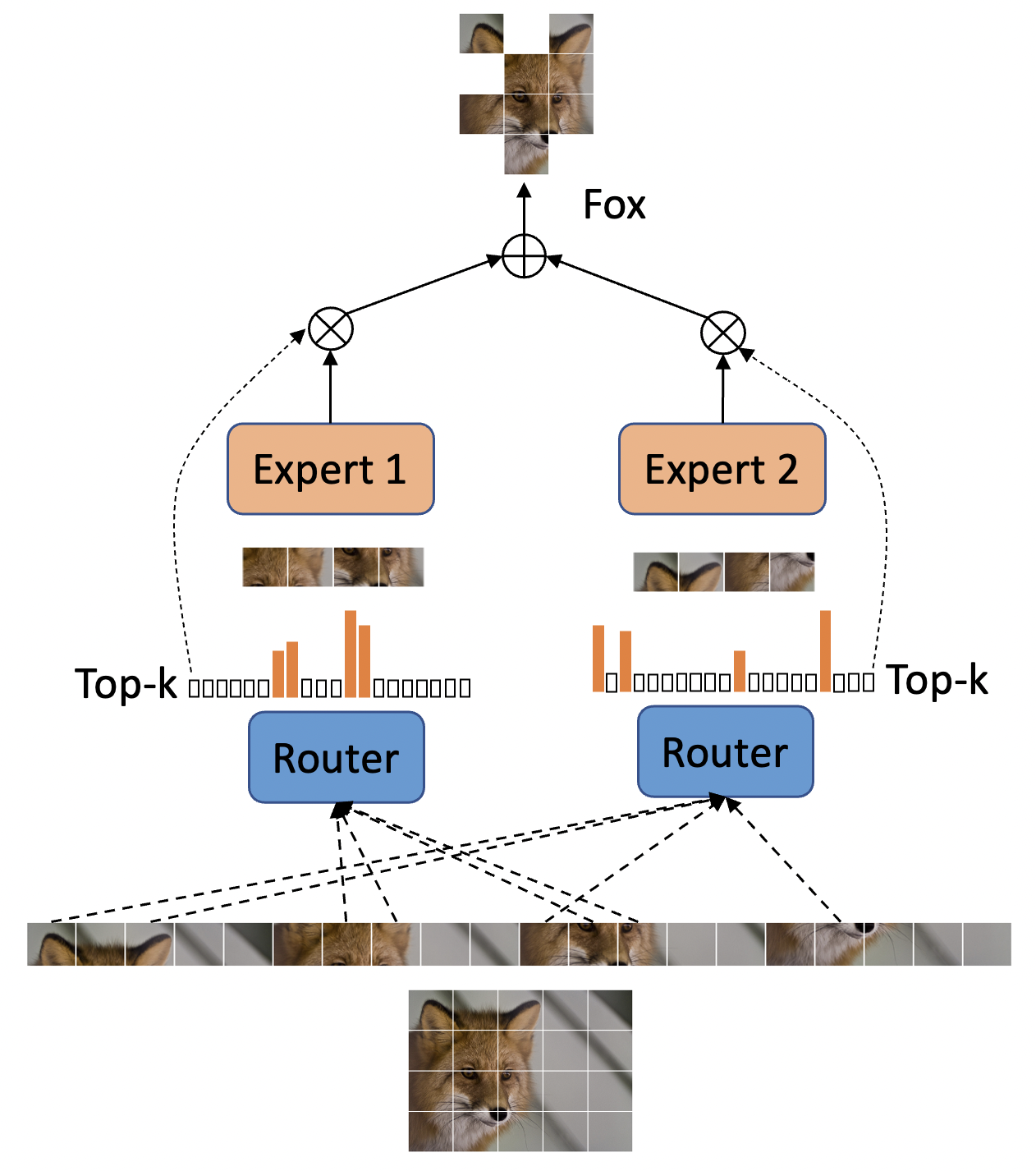}
    \caption{An illustration of pMoE. The image is divided into $20$ patches while the router selects $4$ of them for each expert.}
    \label{pMoE_arch}
\vskip -0.1in
\end{figure}
A conventional 
MoE model contains multiple experts (subnetworks of the backbone architecture) and one learnable router
that routes each input sample to a few but not all the experts \citep{ramachandran2018diversity}.  
 Position-wise MoE has been introduced in language models \citep{shazeer2017outrageously,lepikhin2020gshard,fedus2022switch}, where the routing decisions are made on embeddings of different positions of the input separately rather than routing the entire text-input. \citet{riquelme2021scaling} extended  it to vision models where the routing decisions are made on image patches. \citet{zhou2022mixtureofexperts} further extended where the MoE layer has one router for each expert such that the router selects partial patches for the corresponding expert and discards the remaining patches. 
  We termed this routing mode as \textit{patch-level routing} and the MoE layer as \textit{patch-level MoE} (pMoE) layer (see Figure \ref{pMoE_arch} for  an illustration of a pMoE).  
 Notably, pMoE  achieves the \textit{same} test accuracy in vision tasks with 20\% less training compute, and 50\% less inference compute compared to its single-expert (i.e., one expert which is receiving all the patches of an input) counterpart of the same capacity \citep{riquelme2021scaling}.

Despite the empirical success of MoE, it remains elusive in theory, why can MoE maintain test accuracy while significantly reducing the amount of computation? To the best of our knowledge, only one recent work by \citet{chen2022towards} shows theoretically that a conventional sample-wise MoE achieves higher test accuracy than convolutional neural networks (CNN) in a special setup of a binary classification task on data from linearly separable clusters. However, the sample-wise analyses by \citet{chen2022towards} do not extend to patch-level MoE, which employ different routing strategies than conventional MoE, and their data model might not characterize some practical datasets. 
This paper addresses the following  question theoretically:

\begin{center}
    \textit{How much computational resource does pMoE  save from the single-expert counterpart while maintaining the same generalization guarantee?}
\end{center}

In this paper, we consider a supervised binary classification task where each input sample consists of $n$ equal-sized patches  including  \textit{class-discriminative} patterns that determine the labels and  \textit{class-irrelevant} patterns that do not affect the labels. The neural network contains a pMoE layer\footnote{In practice, pMoEs are usually placed in the last layers of deep models. Our analysis can be extended to this case   as long as the input to the pMoE layer satisfies our data model (see Section \ref{data_model}).} and multiple experts, each of which is a two-layer CNN\footnote{We consider CNN as expert due to its wide applications, especially in vision tasks. Moreover, the pMoE in \cite{riquelme2021scaling,zhou2022mixtureofexperts} uses two-layer Multi-Layer Perceptrons (MLPs) as experts in vision transformer (ViT), which operates on image patches. Hence, the MLPs in \cite{riquelme2021scaling,zhou2022mixtureofexperts} are effectively non-overlapping CNNs.} of the same architecture. The router sends $l$ ($l\ll n$) patches to each expert. Although we consider a simplified neural network model to facilitate the formal analysis of pMoE, the insights are applicable to more general setups. Our major results include:

\textbf{1.}   To the best of our knowledge, this paper provides \textbf{the first theoretical generalization analysis of pMoE}. Our analysis reveals that pMoE with two-layer CNNs as experts can achieve the same generalization performance  as conventional CNN while reducing  the sample complexity (the required number of training samples to learn a proper model) and model complexity. Specifically,   we prove that as long as $l$ is larger than a certain threshold, pMoE reduces the sample complexity and model complexity by a factor   polynomial in $n/l$, indicating an improved generalization with a smaller $l$.

\textbf{2. Characterization of the desired property of the pMoE router.} We show that a desired pMoE router can dispatch the same class-discriminative patterns to the same expert and discard some class-irrelevant patterns. This discriminative property   allows the experts to learn the class-discriminative patterns with reduced interference from irrelevant patterns, which in turn reduces the sample complexity and model complexity. We also prove theoretically that a separately trained pMoE router has the desired property and empirically verify this property on practical pMoE routers.

\textbf{3. Experimental demonstration of reduced sample complexity by pMoE in deep CNN models.} In addition to verifying our theoretical findings on synthetic data prepared from the MNIST dataset \citep{lecun2010mnist}, we demonstrate the sample efficiency of pMoE in learning some benchmark vision datasets (e.g., CIFAR-10 \citep{krizhevsky2009learning} and CelebA \citep{liu2015deep}) by replacing the last convolutional layer of a ten-layer wide residual network (WRN) \citep{zagoruyko2016wide} with a pMoE layer. These experiments not only verify our theoretical findings but also demonstrate the applicability of pMoE in reducing sample complexity in deep-CNN-based vision models, complementing the existing empirical success of   pMoE with vision transformers. 

\section{Related Works}
\textbf{Mixture-of-Experts.} MoE was first introduced in the 1990s with dense sample-wise routing, i.e. each input sample is routed to all the experts \citep{jacobs1991adaptive,jordan1994hierarchical,chen1999improved,NIPS2000_9fdb62f9,rasmussen2001infinite}. 
Sparse sample-wise routing was later introduced \citep{bengio2013estimating,eigen2013learning}, where each input sample activates few of the experts in an MoE layer both for joint training \citep{ramachandran2018diversity,yang2019condconv} and separate training of the router and experts \citep{collobert2001parallel,collobert2003scaling,ahmed2016network,gross2017hard}. Position/patch-wise MoE (i.e., pMoE) recently demonstrated success 
in large language and vision models \citep{shazeer2017outrageously,lepikhin2020gshard,riquelme2021scaling,fedus2022switch}. To solve the issue of load imbalance \citep{lewis2021base}, 
\citet{zhou2022mixtureofexperts} introduces the \textit{expert-choice routing} in pMoE, where 
each expert uses 
one router to select a   fixed number of patches from the input. This paper analyzes the sparse patch-level MoE with expert-choice routing under both joint-training and separate-training setups. 

\textbf{Optimization and generalization analyses of neural networks (NN).} 
Due to the significant nonconvexity of   deep learning problem, 
the existing generalization analyses are   limited to linearized or shallow neural networks. 
The   Neural-Tangent-Kernel (NTK) approach 
\citep{jacot2018neural,lee2019wide,du2019gradient,allen2019convergence,zou2020gradient,chizat2019lazy,ghorbani2021linearized}  considers strong over-parameterization 
and approximates the neural network by the first-order Taylor expansion. The NTK results are independent of  the input data, and performance gaps in 
the representation power and generalization ability exist between the practical NN and the NTK results \citep{yehudai2019power,ghorbani2019limitations,ghorbani2020neural,li2020learning,malach2021quantifying}. 
Nonlinear neural networks are analyzed recently through   higher-order Taylor expansions \citep{allen2019learning,bai2019beyond,arora2019fine,ji2019polylogarithmic} or employing a model estimation approach 
 from Gaussian input data \citep{zhong2017recovery,zhong2017learning,zhang2020improved,zhang2020fast,fu2020guaranteed,li2022learning},  but these results are limited to two-layer networks with   few papers on  three-layer networks \citep{allen2019learning,allen2019can,allen2020backward,li2022generalization}.

The above works   consider   arbitrary input data or Gaussian input. 
To better characterize the practical generalization performance, some recent works analyze   structured data models using approaches such as feature mapping \citep{li2018learning}, where some of the initial model weights are close to data features, and   feature learning \citep{daniely2020learning,shalev2020computational,shi2021theoretical,allen2022feature,li2023a}, where some weights gradually learn features during training. 
Among them, 
\citet{allen2020towards,brutzkus2021optimization,karp2021local} analyze CNN on learning structured data  composed of class-discriminative patterns that determine the labels and other label-irrelevant patterns. This paper extends the data models in \citet{allen2020towards,brutzkus2021optimization,karp2021local} to a more general setup, and our analytical approach  
 is a combination of feature learning  in routers  and feature mapping in experts for pMoE.

\section{Problem Formulation}

This paper considers the supervised binary classification\footnote{Our results can be extended to multiclass classification problems. See Section \ref{multi_class_extension} in the Appendix for details.} problem where given $N$ i.i.d. training samples   $\{(x_i,y_i)\}_{i=1}^N$ generated by an unknown distribution $\mathcal{D}$, the objective is to learn a neural network model that maps $x$ to  $y$ for any $(x,y)$   sampled from  $\mathcal{D}$. Here, the input  $x\in\mathbb{R}^{nd}$ has $n$ disjoint patches, i.e.,  $x^{\intercal}=[x^{(1)\intercal},x^{(2)\intercal},...,x^{(n)\intercal}]$, where $x^{(j)}\in\mathbb{R}^d$ denotes the $j$-th patch of $x$. $y \in \{+1, -1\}$ denotes the corresponding label. 

\subsection{Neural Network Models}
\label{models}

We consider a pMoE architecture that includes 
$k$ experts and the corresponding $k$ routers. Each router selects $l$ out of $n$ ($l <n$) patches for each expert separately. Specifically, the router for each expert $s$ ($s\in[k]$) contains a  trainable gating kernel $w_s\in\mathbb{R}^d$. Given a sample $x$, 
the router computes a routing value $g_{j,s}(x)=\langle w_s, x^{(j)}\rangle$   for each patch $j$. Let  $J_s(x)$ denote  the index set of top-$l$ values of $g_{j,s}$ among all the patches $j\in[n]$. Only patches with indices in $J_s(x)$ are routed to the expert $s$, multiplied by a gating value $G_{j,s}(x)$, which are selected differently in different pMoE models.

Each expert is a two-layer CNN with the same architecture. Let $m$ denote the total number of neurons in all the experts. Then each expert contains $(m/k)$ neurons. Let $w_{r,s}\in\mathbb{R}^d$ and $a_{r,s} \in\mathbb{R}$ denote the hidden layer and output layer weights for neuron $r$ ($r\in [m/k])$ in expert $s$ ($s\in[k]$), respectively.  The activation function is the rectified linear unit (ReLU), where $\textbf{ReLU}(z)=\text{max}(0,z)$.

Let $\theta=\{a_{r,s}, w_{r,s}, w_s, \forall s \in [k], \forall r \in [m/k]\}$ include all the trainable weights. The pMoE model denoted as $f_{M}$, is defined as follows: 
\begin{equation}\label{mcnn_bin}
   \resizebox{.99\hsize}{!}{$f_{M}(\theta, x)=\overset{k}{\underset{s=1}{\sum}}\overset{\frac{m}{k}}{\underset{r=1}{\sum}}\cfrac{a_{r,s}}{l}\underset{j \in J_s(w_s,x)}{\sum}\textbf{ReLU}(\langle w_{r,s},x^{(j)}\rangle)G_{j,s}(w_s,x)$}
\end{equation}

An illustration of (\ref{mcnn_bin}) is given in Figure \ref{arch_eqn_1}.
\begin{figure}[h]
\vskip 0.1in
    \centering
    \includegraphics[width=0.7\linewidth]{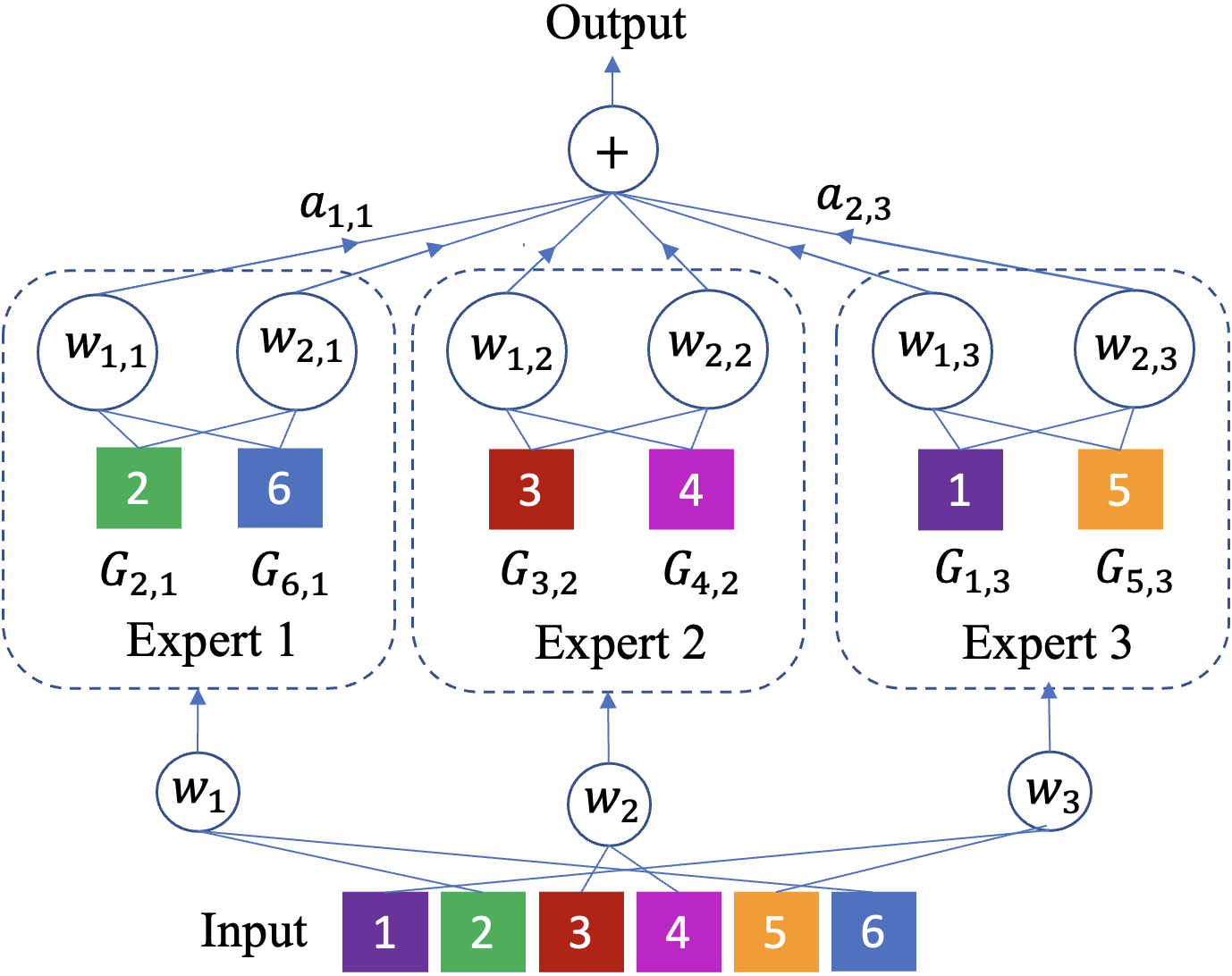}
    \caption{An illustration of the pMoE model in (\ref{mcnn_bin}) with $k=3, m=6, n=6$, and $l=2$. \iffalse Here, each input is divided into fixed number ($n$) of patches. For each expert, the router contains a gating-kernel (denoted as $\{w_1, w_2, ... w_k\}$) which selects a fixed number of patches ($l$) to be routed to the expert ($l<n$). While routing a patch to an expert, the router provides a gating-value (i.e., $G_{j,s}$) which is multiplied to the outputs of the neurons in the expert for that patch.\fi}
    \label{arch_eqn_1}
\vskip -0.1in
\end{figure}

The learning problem solves the following empirical risk minimization problem with the logistic loss function, 
\begin{equation}\label{erm}
    \displaystyle\underset{\theta}{\text{min}}:\hspace{0.4cm}L(\theta)=\cfrac{1}{N}\overset{N}{\underset{i=1}{\sum}}\log{(1+e^{-y_if_M(\theta, x_i)})}
\end{equation}
We consider two different training modes of pMoE, 
\textit{Separate-training} and \textit{Joint-training} of the routers and the experts. We also consider 
 the conventional CNN architecture for comparison. 

\noindent (I) \textbf{Separate-training pMoE}: Under the setup of the so-called \textit{hard mixtures of experts} \citep{collobert2003scaling,ahmed2016network,gross2017hard}, the router weights $w_s$ are trained  first  and then fixed when training the weights of the experts. 
In this case, the gating values 
 are set as \begin{equation}
    G_{j,s}(w_s,x) \equiv 1, \ \forall j, s, x
\end{equation}
We select $k=2$ in this case to simplify the analysis.

\noindent (II) \textbf{Joint-training pMoE}:  The routers and the experts are learned jointly, see, e.g., \citep{lepikhin2020gshard,riquelme2021scaling,fedus2022switch}. Here, the gating values are  softmax functions with  
\begin{equation}
    G_{j,s}(w_s,x)=e^{g_{j,s}(x)}/(\sum_{i\in J_s(x)} g_{i,s}(x))
\end{equation}

\noindent (III) \textbf{CNN single-expert counterpart}: The conventional two-layer CNN with $m$ neurons, denoted as $f_{C}$, satisfies,
\begin{equation}\label{cnn_bin}
    \displaystyle f_{C}(\theta, x)=\overset{m}{\underset{r=1}{\sum}}a_{r}\left(\cfrac{1}{n}\overset{n}{\underset{j=1}{\sum}}\textbf{ReLU}(\langle w_r,x^{(j)}\rangle)\right)
\end{equation}
Eq. (\ref{cnn_bin}) can be viewed as a special case of (\ref{mcnn_bin}) 
when there is only one expert ($k=1$), and all the patches are sent to the expert ($l=n$) with gating values $G_{j,s}\equiv1$. 

Let $\tilde{\theta}$ denote the parameters of the  learned model by solving (\ref{mcnn_bin}). The predicted label for a test sample $x$ by the learned model is $\textrm{sign}(f_{M}(\tilde{\theta},x))$. 
The generalization accuracy, 
i.e., the fraction of correct predictions of all test samples equals  $\underset{(x,y)\sim\mathcal{D}}{\mathbb{P}}\left[yf_{M}(\theta,x)>0\right]$.
This paper studies both separate and joint training of pMoE and compares their performance with CNN, from the perspective of sample complexity to achieve a desirable generalization accuracy.

\subsection{Training Algorithms}\label{alg}

In the following algorithms, we fix the output layer weights $a_{r,s}$ and $a_{r}$ at their initial values randomly sampled from the standard Gaussian distribution  $\mathcal{N}(0,1)$ and do not update them during the training. This is a typical simplification when analyzing NN, as used in  \citep{li2018learning,brutzkus2018sgd,allen2019learning,arora2019fine}.


(I) Separate-training pMoE: The routers are separately trained using $N_r$ training samples ($N_r<N$), denoted  by $\{(x_i, y_i)\}_{i=1}^{N_r}$ without loss of generality. The gating kernels $w_1$ and $w_2$ are obtained by solving the following minimization problem:
\begin{equation}\label{router_erm}
    \displaystyle \underset{w_1,w_2}{\text{min}}:\hspace{0.1cm}l_r(w_1,w_2)=-\cfrac{1}{N_r}\hspace{0.1cm}\overset{N_r}{\underset{i=1}{\sum}} y_i  \langle w_1-w_2, \sum_{j=1}^n x_i^{(j)}\rangle
\end{equation}
To solve (\ref{router_erm}), 
we implement the mini-batch SGD with batch size $B_r$ for $T_r=N_r/B_r$ iterations, starting from the  
random initialization as follows:
\begin{equation}\label{eqn:router_ini}
    w_s^{(0)}\sim\mathcal{N}(0,\sigma_r^2\mathbb{I}_{d\times d}), \forall s\in [2]
\end{equation}
where, $\sigma_r=\Theta\big(1\big/(n^2 \log{(\textrm{poly}(n))}\sqrt{d})\big)$.

After learning the routers, we train the hidden-layer weights $w_{r,s}$ by solving (\ref{erm}) while fixing  $w_1$ and $w_2$.  We implement mini-batch SGD of batch size $B$ for $T=N/B$ iterations 
starting from the initialization 
\begin{equation}\label{eqn:expert_ini}
    w_{r,s}^{(0)}\sim \mathcal{N}(0,\frac{1}{m}\mathbb{I}_{d\times d}), 
    \forall s \in [2], \forall r\in[m/2] 
\end{equation}

(II) Joint-training pMoE:   $w_s$ and $w_{r,s}$ in (\ref{mcnn_bin}) are updated simultaneously by mini-batch SGD of batch size $B$ for $T=N/B$ iterations starting from the initialization in (\ref{eqn:router_ini}) and (\ref{eqn:expert_ini}).

(III) CNN:      $w_{r}$  in (\ref{cnn_bin}) are updated by mini-batch SGD of batch size $B$ for $T=N/B$ iterations starting from the initialization in (\ref{eqn:expert_ini}).

\section{Theoretical Results}

\subsection{Key Findings At-a-glance}

Before defining the data model assumptions and rationale in Section \ref{data_model} and presenting the formal results in \ref{sec: theory}, we first summarize our key findings.
 We assume that the data patches are sampled from either \textit{class-discriminative} patterns that determine the labels or a possibly infinite number of \textit{class-irrelevant} patterns that have no impact on the label. The parameter $\delta$ (defined in \eqref{eqn:delta}) is inversely related to the separation among patterns, i.e.,  $\delta$ decreases when (i) the separation among class-discriminative patterns increases, and/or (ii) the separation between class-discriminative and class-irrelevant patterns increases. The key findings are as follows. 

\textbf{(I). A properly trained patch-level router  sends  class-discriminative patches of one class to the same expert while dropping some class-irrelevant patches}.  We prove that 
separate-training pMoE routes class-discriminative patches 
of the class with label $y=+1$ (or the class with label $y=-1$) to the expert 1 (or the expert 2) respectively, and the class-irrelevant patterns that are sufficiently away from class-discriminative patterns 
are not routed to any expert (Lemma \ref{router_lemma}). This discriminative routing property is also verified empirically for joint-training pMoE (see section \ref{experiment_mnist}).  Therefore, pMoE effectively 
reduces the interference by irrelevant patches when each expert       learns the  class-discriminative patterns. 
Moreover, we   show empirically that pMoE can remove class-irrelevant patches that are spuriously correlated with class labels and thus can avoid learning from spuriously correlated features of the data.

\textbf{(II). Both the sample complexity and the required number of hidden nodes of pMoE reduce by a polynomial factor of $n/l$ over CNN.} 
We  prove that as long as $l$, the number of patches per expert, is greater than a threshold (that decreases as the separation between class-discriminative and class-irrelevant patterns increases), 
the sample complexity and the required number of neurons of learning pMoE are  $\Omega(l^8)$ and $\Omega(l^{10})$ respectively. In contrast, the sample and model complexities 
of the CNN are $\Omega(n^8)$ and $\Omega(n^{10})$ respectively, indicating improved generalization by pMoE.

\textbf{(III). Larger separation among class-discriminative and class-irrelevant patterns reduces the sample complexity and model complexity of pMoE.} Both the sample complexity and the required number of neurons of pMoE is polynomial in $\delta$, which decreases when the separation among patterns increases.

\subsection{Data Model Assumptions and Rationale}\label{data_model}
 
The input $x$ is comprised of one class-discriminative pattern and $n-1$ class-irrelevant patterns, and the label $y$ is determined by the class-discriminative pattern only.

\textbf{Distributions of class-discriminative patterns}: The unit vectors $o_1$ and $o_2\in\mathbb{R}^d$ denote the \textit{class-discriminative} patterns that determine the labels. The separation between $o_1$ and $o_2$ is measured as $\delta_d:=\langle o_1, o_2\rangle \in (-1,1)$.  $o_1$ and $o_2$ are equally distributed in the samples, 
and each sample has exactly one of them. 
If   $x$ contains $o_1$ (or $o_2$), then $y$ is $+1$ (or $-1$). 

\textbf{Distributions of class-irrelevant patterns.} \textit{Class-irrelevant} patterns are unit vectors in $\mathbb{R}^d$ belonging to $p$ disjoint pattern sets 
$S_1, S_2, ...., S_p$, and these patterns distribute equally for both classes.   
 $\delta_r$ measures the separation between class-discriminative patterns and class-irrelevant patterns, where $|\langle o_i,  q\rangle|\le\delta_r$, $\forall i\in[2]$, $\forall q \in S_j$, $j=1,...,p$. 
 Each $S_j$ belongs to a ball with a diameter of $\Theta(\sqrt{(1-\delta_r^2)/dp^2)}$.  Note that NO separation among class-irrelevant patterns themselves is required. 

\textbf{The rationale of  our data model.} 
The data distribution $\mathcal{D}$ 
captures the locality of the label-defining features in image data. It is motivated by and extended from 
the data distributions in recent  theoretical frameworks \citep{yu2019learning,brutzkus2021optimization,karp2021local,chen2022towards}. 
Specifically, \citet{yu2019learning} and \citet{brutzkus2021optimization} require orthogonal patterns, 
i.e., $\delta_r$ and $\delta_d$ are 
both $0$, and there are only a fixed number of non-discriminative patterns. \citet{karp2021local} and \citet{chen2022towards} assume that 
$\delta_d=-1$ and 
a possibly infinite number of patterns 
drawn from zero-mean Gaussian distribution. In our model, 
$\delta_d$ takes any value in $(-1,1)$, and the class-irrelevant patterns can be drawn from $p$ pattern sets that contain an infinite number of patterns that are not necessarily 
  Gaussian   or   orthogonal.  

Define  
\begin{equation}\label{eqn:delta}
\delta =1/(1-\max (\delta_d^2, \delta_r^2))
\end{equation}
$\delta$ decreases if (1) $o_1$ and $o_2$
are more separated from each other, and (2) Both  $o_1$ and $o_2$ are more separated from any set $S_i$, $i \in [p]$.
We also define an integer $l^*$ ($l^*\leq n$) that measures \textit{the maximum number of class-irrelevant patterns per sample that are sufficiently closer to $o_1$ than $o_2$, and vice versa}. Specifically, a  class-irrelevant pattern $q$ is 
 called $\delta^\prime$-closer ($\delta^\prime>0$) to $o_1$ than $o_2$, if $\langle o_1-o_2, q\rangle>\delta^\prime$ holds. 
 Similarly,  $q$ is $\delta^\prime$-closer to $o_2$ than $o_1$ if $\langle o_2-o_1,q\rangle>\delta^\prime$. Then, let $l^*-1$ be the maximum number of  class-irrelevant patches  that are either $\delta^\prime$-closer to $o_1$ than $o_2$ or vice versa with $\delta^\prime=\Theta(1-\delta_d)$ in any $x$ sampled from $\mathcal{D}$.  $l^*$ depends on $\mathcal{D}$ and $\delta_d$. When $\mathcal{D}$ is fixed, a smaller $\delta_d$ corresponds to a larger separation between $o_1$ and $o_2$ and leads to a small $l^*$. In contrast to linearly separable data in \citep{yu2019learning,brutzkus2018sgd,chen2022towards}, our data model is \textbf{NOT} linearly separable as long as $l^*=\Omega(1)$ (see section \ref{non_linearly_separable} in Appendix for the proof).
 
\subsection{Main Theoretical Results}\label{sec: theory}

\subsubsection{Generalization Guarantee of Separate-training pMoE}\label{separate_pMoE_th}
Lemma \ref{router_lemma} shows that as long as the number of patches per expert, $l$, is greater than  $l^*$, 
then the separately learned routers by solving (\ref{router_erm}) always send $o_1$  
to expert 1 and  $o_2$ 
to expert 2.  Based on this discriminative property of the learned routers,  
Theorem \ref{Thm_mcnn_bin} then quantifies the sample complexity and network size of separate-training pMoE to achieve a desired generalization error $\epsilon$. 
Theorem \ref{Thm_single_cnn} quantifies the sample and model complexities of CNN 
for comparison.

\begin{lemma}[Discriminative Property of Separately Trained Routers]\label{router_lemma}
For every $l\geq l^*$, w.h.p. over the random initialization defined in (\ref{eqn:router_ini}), after doing mini-batch SGD with batch-size $B_r= \Omega\left(n^2/(1-\delta_d)^2\right)$ and learning rate $\eta_r=\Theta(1/n)$, for $T_r=\Omega\left(1/(1-\delta_d)\right)$ iterations, the returned $w_1$ and $w_2$ satisfy 
\begin{equation*}
    \underset{j\in[n]}{\text{arg}}(x^{(j)}=o_1) \in J_1(w_1, x), \quad \forall  (x,y=+1) \sim \mathcal{D}
\end{equation*}
\begin{equation*}
    \underset{j\in[n]}{\text{arg}}(x^{(j)}=o_2) \in J_2(w_2, x), \quad \forall  (x,y=-1) \sim \mathcal{D}
\end{equation*}
i.e., the learned routers always send  $o_1$ 
to expert 1 and $o_2$ 
to expert 2. 

\end{lemma}

The main idea in proving Lemma \ref{router_lemma} is to show that the gradient in each iteration has a large component along the directions of $o_1$ and $o_2$. Then after enough iterations, the inner product of $w_1$ and $o_1$ (similarly, $w_2$ and $o_2$) is sufficiently large.  The intuition of requiring $l\ge l^*$ is that 
because there are at most $l^*-1$ class-irrelevant patches sufficiently closer to $o_1$ than $o_2$ (or vice versa), then sending $l\ge l^*$ patches to one expert will ensure that one of them is $o_1$ (or $o_2$). Note that the batch size $B_r$ and the number of iterations $T_r$ depend on $\delta_d$, the separation between $o_1$ and $o_2$, but are independent of the separation between class-discriminative and class-irrelevant patterns.

We then show that the separate-training pMoE reduces both the sample complexity and the required model size (Theorem \ref{Thm_mcnn_bin}) compared to the CNN (Theorem \ref{Thm_single_cnn}).
\begin{theorem}[Generalization guarantee of separate-training pMoE] \label{Thm_mcnn_bin}
For every $\epsilon>0$ and $l\ge l^*$, for every $m\ge M_S=\Omega\left(l^{10}p^{12}\delta^6\big/\epsilon^{16}\right)$ with at least   $N_S=\Omega(l^8 p^{12}\delta^6/\epsilon^{16})$ training samples, after performing minibatch SGD with the batch size $B=\Omega\left(l^4p^6\delta^3\big/\epsilon^{8}\right)$ and the learning rate $\eta=O\big(1\big/(m \textrm{poly}(l,p,\delta,1/\epsilon, \log m))\big)$ for $T=O\left(l^4p^6\delta^3\big/\epsilon^{8}\right)$ iterations, it holds w.h.p. that
\begin{center}
    $\underset{(x,y)\sim\mathcal{D}}{\mathbb{P}}\left[yf_{M}(\theta^{(T)},x)>0\right]\ge1-\epsilon$
\end{center}
\end{theorem}
Theorem \ref{Thm_mcnn_bin} implies that 
to achieve generalization error $\epsilon$ 
by a separate-training pMoE, we need $N_S=\Omega(l^8 p^{12}\delta^6/\epsilon^{16})$ training samples and $M_S=\Omega\left(l^{10}p^{12}\delta^6\big/\epsilon^{16}\right)$ hidden nodes. Therefore, 
both $N_S$ and $M_S$ increase polynomially with the number of patches $l$ sent to each expert. Moreover, both  $N_S$ and $M_S$ are polynomial  in  $\delta$  defined in (\ref{eqn:delta}), indicating an improved generalization performance with stronger separation among patterns.

The proof of Theorem \ref{Thm_mcnn_bin} is inspired by \citet{li2018learning}, which analyzes the generalization performance of fully-connected neural networks (FCN) on structured data, but we have new technical contributions in analyzing pMoE models.  In addition to analyzing the pMoE routers (Lemma \ref{router_lemma}), which do not appear in the FCN analysis, our analyses also significantly relax the separation requirement on the data, compared with that by \citet{li2018learning}. 
For example, \citet{li2018learning} requires the separation between the two classes, measured by the smallest $\ell_2$-norm distance of two points in different classes,  being $\Omega(n)$ to  obtain a sample complexity bound of poly($n$) for the binary classification task.  In contrast, the separation between the two classes in our data model  is $\min\{\sqrt{2(1-\delta_d)}, 2\sqrt{1-\delta_r}\}$, much less than  $\Omega(n)$ required by \citet{li2018learning}. 

\begin{theorem}[Generalization guarantee of CNN]\label{Thm_single_cnn}
For every $\epsilon>0$, for every $m\ge  M_C=\Omega\left(n^{10}p^{12}\delta^6\big/\epsilon^{16}\right)$ with at least  $N_C=\Omega(n^8p^{12}\delta^6/\epsilon^{16})$ training samples, 
after performing minibatch SGD with  the batch size $B=\Omega\left(n^4p^6\delta^3\big/\epsilon^{8}\right)$ and the learning rate $\eta=O\big(1\big/(m \textrm{poly}(n,p,\delta,1/\epsilon, \log m))\big)$ for $T=O\left(n^4p^6\delta^3\big/\epsilon^{8}\right)$ iterations, it holds w.h.p. that 
\begin{center}
    $\underset{(x,y)\sim\mathcal{D}}{\mathbb{P}}\left[yf_{C}(\theta^{(T)},x)>0\right]\ge1-\epsilon$
\end{center} 
\end{theorem}
Theorem \ref{Thm_single_cnn} implies that to achieve a generalization error $\epsilon$ 
using CNN in (\ref{cnn_bin}),  
we need $N_C=\Omega(n^8p^{12}\delta^6/\epsilon^{16})$ training samples and $M_C=\Omega\left(n^{10}p^{12}\delta^6\big/\epsilon^{16}\right)$ neurons.


\textbf{Sample-complexity gap between single CNN and mixture of CNNs.} From Theorem \ref{Thm_mcnn_bin} and Theorem \ref{Thm_single_cnn}, 
the sample-complexity ratio of the CNN to the separate-training pMoE is $N_C/N_S=\Theta\big((n/l)^8\big)$. 
Similarly, the required number of neurons is reduced by a factor of  $M_C/M_S=\Theta\big((n/l)^{10}\big)$ in separate-training pMoE\footnote{The bounds for the sample complexity and model size in Theorem \ref{Thm_mcnn_bin} and Theorem \ref{Thm_single_cnn} are sufficient but not necessary. Thus, rigorously speaking, one can not compare 
sufficient conditions only. In our analysis, however, the bounds for MoE and CNN are derived with   exactly the same technique with the only difference to handle the routers. 
Therefore, it is fair to compare these two bounds to show the advantage of pMoE.}.  

\subsubsection{Generalization Guarantee of Joint-training pMoE with Proper Routers}

Theorem \ref{Thm_mcnn_j} characterizes the generalization performance of joint-training pMoE 
assuming   the routers are properly trained in the sense  that after some SGD iterations, for each class 
at least one of the $k$ experts
receives all class-discriminative patches 
of that class with the largest gating-value (see Assumption \ref{router_asmptn}). 
\begin{assumption}\label{router_asmptn}
There exists an integer $T^\prime<T$ such that for all  $t\ge T^\prime$, it holds that:
\begin{align*}
    &\text{There exists an expert } s\in[k] \text{ s.t. } \forall (x,y=+1)\sim \mathcal{D},\\ 
    &\hspace{1cm}j_{o_1}\in J_s(w_s^{(t)},x), \text{ and } G_{j_{o_1},s}^{(t)}(x)\ge G_{j,s}^{(t)} (x)\\
    &\text{and an expert } s\in[k] \text{ s.t. } \forall (x,y=-1)\sim \mathcal{D},\\ 
    &\hspace{1cm}j_{o_2}\in J_s(w_s^{(t)},x), \text{ and } G_{j_{o_2},s}^{(t)}(x)\ge G_{j,s}^{(t)} (x)
\end{align*}
where $j_{o_1}$ ($j_{o_2}$) denotes the index of the class-discriminative pattern $o_1$ ($o_2$), $G_{j,s}^{(t)}(x)$ is the  gating output of patch $j\in J_s(w_s^{(t)},x)$ of sample $x$ for expert $s$ at the iteration $t$, and $w_s^{(t)}$ is  the gating kernel for expert $s$ at iteration $t$. 
\end{assumption}
Assumption \ref{router_asmptn} is required in proving Theorem \ref{Thm_mcnn_j} because of the difficulty of
  tracking the dynamics of the routers in joint-training pMoE. Assumption \ref{router_asmptn} is verified on empirical experiments in Section \ref{experiment_mnist}, while its theoretical proof is left for future work.

\begin{table*}[t]
\caption{Computational complexity of pMoE and CNN.}
\label{table1}
\vskip 0.1in
\begin{center}
\begin{small}
\begin{sc}
\renewcommand{\arraystretch}{1.4}
\begin{tabular}{m{1.2in}|>{\centering}m{1.2in}|>{\centering}m{1.7in}|>{\centering}m{0.8in}|>{\centering\arraybackslash}m{0.8in}}
\toprule
\multirow{3}{1.2in}{\textbf{Complexity to achieve $\epsilon$ error} (Complx/Iter $\times$ T)} & \multicolumn{3}{c|}{\textbf{pMoE}} & \multirow{2}{*}{\textbf{CNN}} \\\cline{2-4}
& \textbf{Separate-training} & \multicolumn{2}{c|}{\textbf{Joint-training}} & \\\cline{2-5}
& $O(Bml^5d/\epsilon^8)$ & \multicolumn{2}{c|}{$O(Bmk^2l^3d/\epsilon^8)$} & $O(Bmn^5d/\epsilon^8)$\\\cline{1-5}
\multirow{3}{1.2in}{Complexity per Iteration (\textbf{Complx/Iter})} & \multirow{3}{*}{$O(Bmld)$} & \textbf{Router} & \textbf{Expert} &  \multirow{3}{*}{$O(Bmnd)$} \\\cline{3-4}
&& $O(Bknd)$ (Forward pass) & \multirow{2}{*}{$O(Bmld)$} &\\
&& $O(Bkl^2d)$ (Backward pass) &&\\\cline{1-5}
Iteration required to converge with $\epsilon$ error (\textbf{T}) & $O(l^4/\epsilon^8)$ & \multicolumn{2}{c|}{$O(k^2l^2/\epsilon^8)$} & $O(n^4/\epsilon^8)$\\
\bottomrule
\end{tabular}
\renewcommand{\arraystretch}{1}
\end{sc}
\end{small}
\end{center}
\vskip -0.1in
\end{table*}

\begin{theorem}[Generalization guarantee of joint-training pMoE]\label{Thm_mcnn_j}
Suppose Assumption \ref{router_asmptn} hold. Then for every $\epsilon>0$, for every  $m \ge M_J=\Omega\left(k^3n^2l^{6}p^{12}\delta^6\big/\epsilon^{16}\right)$ with at least $N_J=\Omega(k^4l^{6} p^{12}\delta^6/\epsilon^{16})$ training samples,  after performing minibatch SGD with the batch size $B=\Omega\left(k^2l^4p^6\delta^3\big/\epsilon^{8}\right)$ and the learning rate $\eta=O\big(1\big/(m\text{poly}(l,p,\delta,1/\epsilon, \log m))\big)$ for $T=O\left(k^2l^2p^6\delta^3\big/\epsilon^{8}\right)$ iterations, it holds w.h.p. that
\begin{center}
   $\underset{(x,y)\sim\mathcal{D}}{\mathbb{P}}\left[yf_{M}(\theta^{(T)}, x)>0\right]\ge1-\epsilon$
\end{center}
\end{theorem}
Theorem \ref{Thm_mcnn_j} indicates that, 
with proper routers, joint-training pMoE 
needs $N_J=\Omega(k^4l^{6} p^{12}\delta^6/\epsilon^{16})$ training samples and $M_J=\Omega\left(k^3n^2l^{6}p^{12}\delta^6\big/\epsilon^{16}\right)$ neurons to achieve $\epsilon$ generalization error. Compared with CNN in Theorem \ref{Thm_single_cnn}, joint-training pMoE   reduces the sample  complexity and model size by a factor of $\Theta(n^8/k^4l^6)$ and $\Theta(n^{10}/k^3l^6)$,  respectively. With more experts (a larger $k$), it is easier to satisfy Assumption \ref{router_asmptn}  to learn proper routers but requires larger sample and model complexities. When the number of samples is fixed,  the expression of $N_J$ also indicates that $\epsilon$ sales as $k^{1/4}l^{3/8}$,   corresponding to an improved generalization when $k$ and $l$ decrease.

We provide the end-to-end computational complexity comparison between the analyzed pMoE models and general CNN model in Table \ref{table1} (see section \ref{table1_details} in Appendix for details). The results in Table \ref{table1} indicates that the computational complexity in joint-training pMoE is reduced by a factor of $O(n^5/k^2l^3)$ compared with CNN. Similarly, the reduction of computational complexity of separate-training pMoE is $O(n^5/l^5)$.

\section{Experimental Results}

\subsection{pMoE of Two-layer CNN}\label{experiment_mnist}
\textbf{Dataset}: We verify our theoretical findings about the model in (\ref{mcnn_bin}) on synthetic data prepared from MNIST \citep{lecun2010mnist} data set. 
Each sample contains $n=16$ patches with patch size $d=28\times28$. Each patch is drawn from the   MNIST dataset.  See Figure \ref{example_img} as an example. We treat the digits ``\textbf{1}'' and ``\textbf{0}'' as the class-discriminative patterns $o_1$ and   $o_2$, respectively. Each of the digits from ``\textbf{2}'' to ``\textbf{9}'' represents a class-irrelevant pattern set.  

\begin{figure}[ht]
\vskip 0.1in
    \begin{minipage}{0.25\linewidth}
        \centering
        \includegraphics[width=0.8\linewidth]{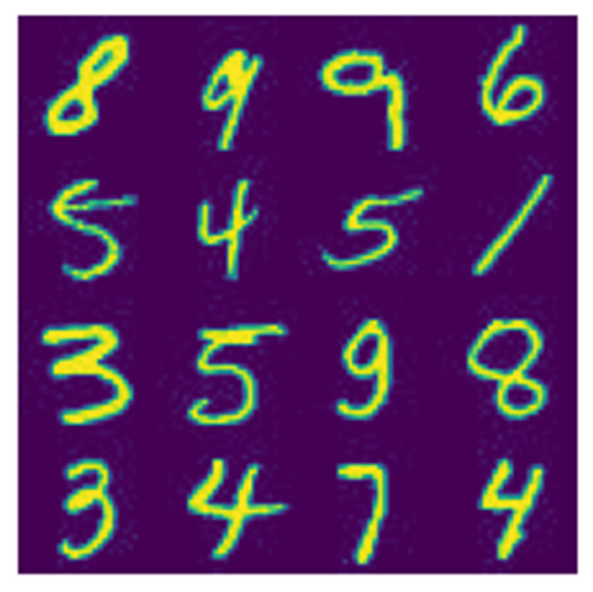}
        \caption{Sample image of the synthetic data from MNIST. Class label is ``1''. }
        \label{example_img}
    \end{minipage}
    ~
    \begin{minipage}{0.75\linewidth}
        \centering
        \includegraphics[width=0.85\linewidth]{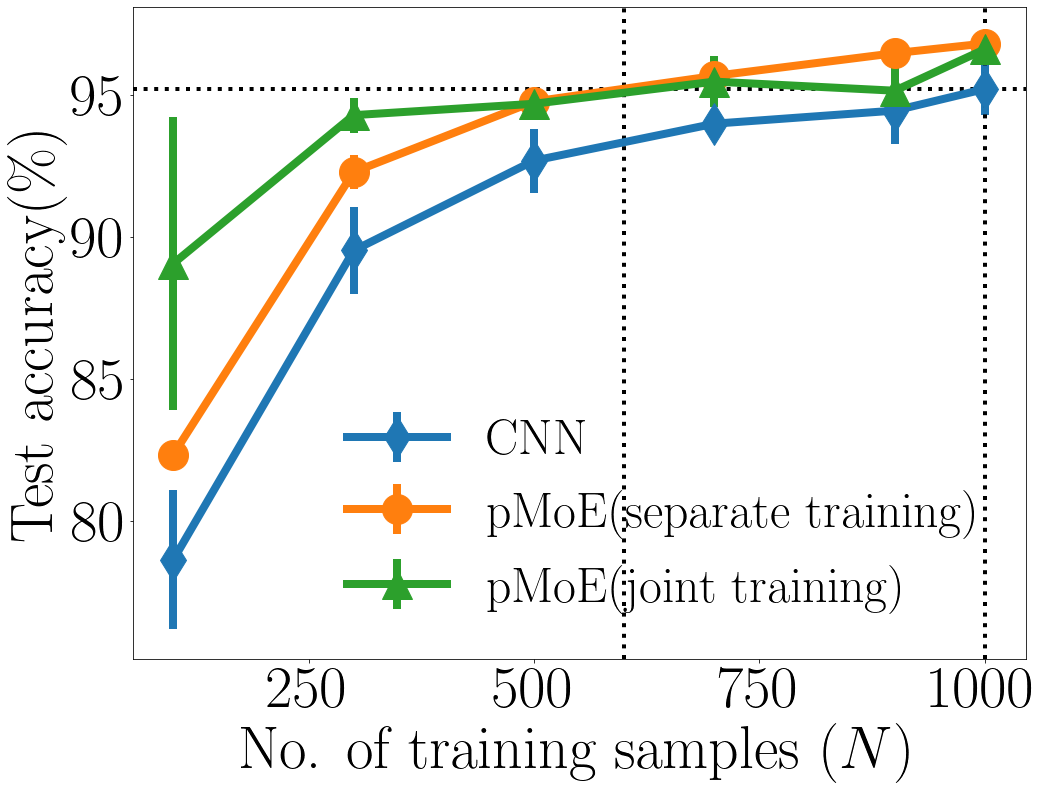}
        \caption{Generalization performance of pMoE and CNN with a similar model size}
        \label{sm_com_mnist}
    \end{minipage}
    \vskip -0.1in
\end{figure}

\begin{figure}[ht]
\vskip 0.1in
    \centering
    \includegraphics[width=0.55\linewidth]{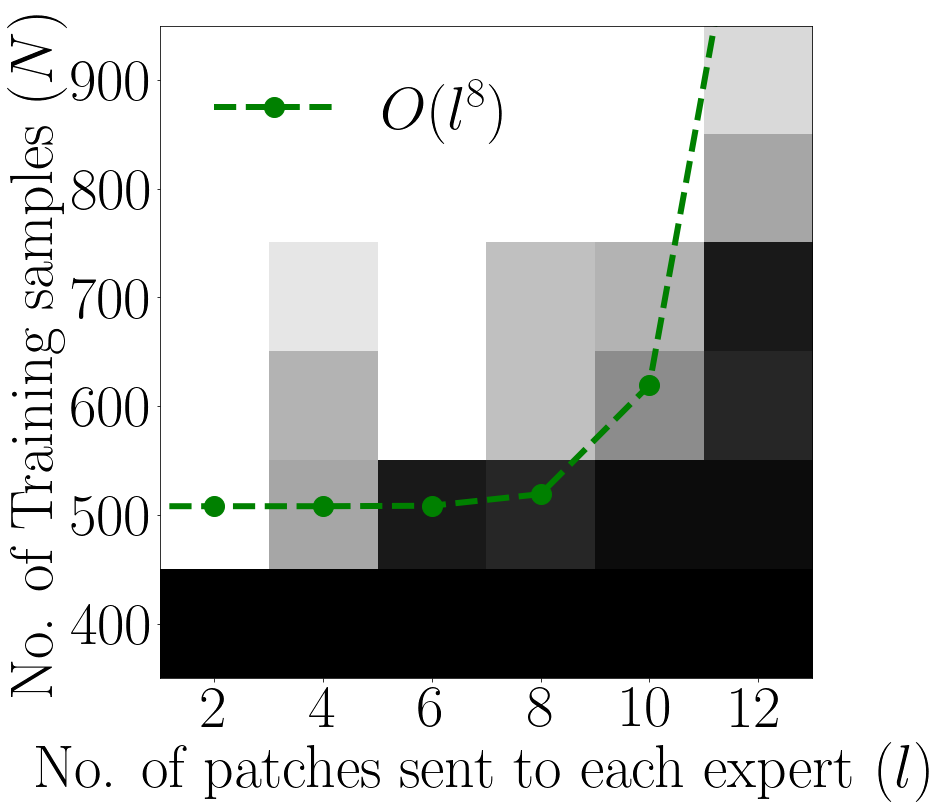}
    \caption{Phase transition of sample complexity with $l$ in separate-training pMoE}
    \label{sm_com_ph_tr_l}
    \vskip -0.1in
\end{figure}

\begin{figure}[ht]
\vskip 0.1in
    \centering
    \includegraphics[width=0.65\linewidth]{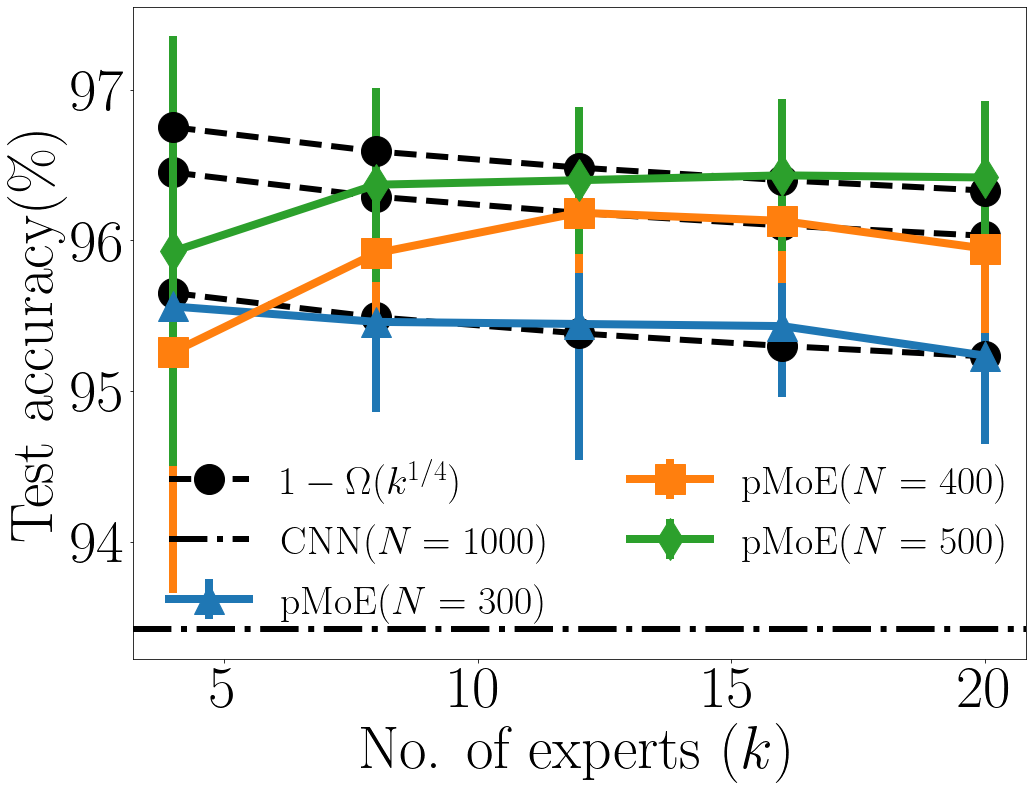}
    \caption{Change of test accuracy in joint-training pMoE with $k$ for fixed sample sizes}
    \label{jpmoe_k_vs_acc}
    \vskip -0.1in
\end{figure}

\begin{figure}[ht]
\vskip 0.1in
    \centering
    \includegraphics[width=0.65\linewidth]{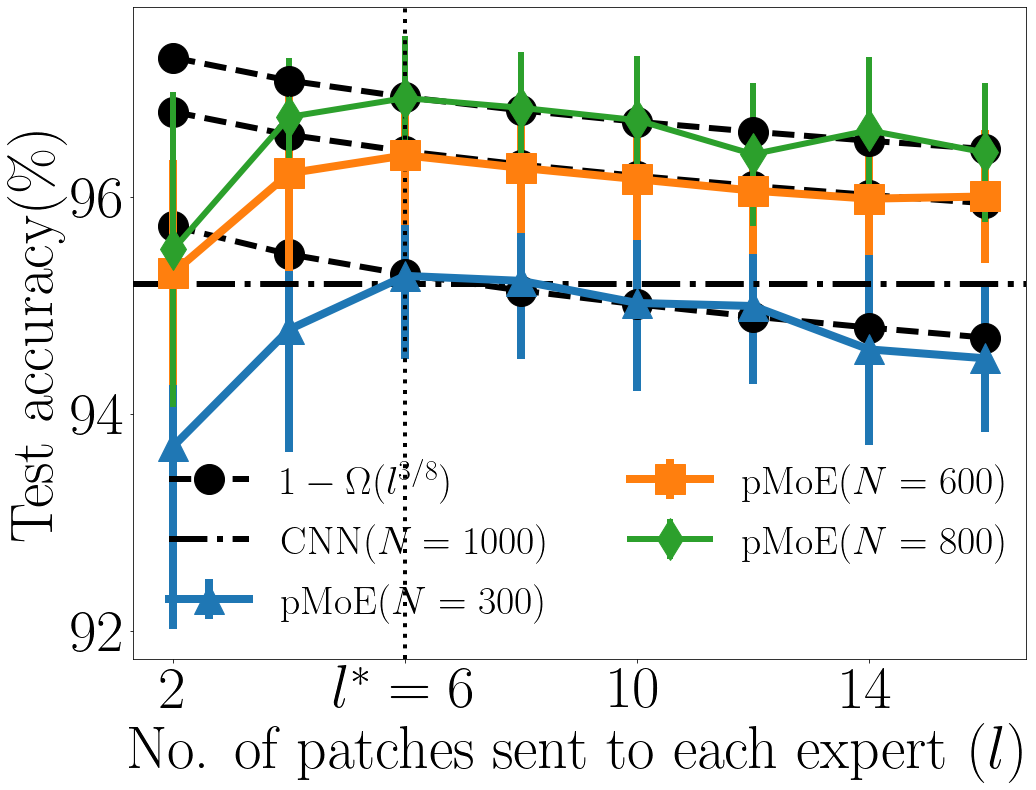}
    \caption{Change of test accuracy in joint-training pMoE with $l$ for fixed sample sizes}
    \label{jpmoe_l_vs_acc}
    \vskip -0.1in
\end{figure}

\textbf{Setup}: We compare separate-training  pMoE,  joint-training pMoE, and CNN with similar model sizes. The separate-training pMoE contains \textit{two} experts with $20$ hidden nodes in each expert. The joint-training pMoE has eight experts with five hidden nodes per expert. The CNN has $40$ hidden nodes. All are trained using SGD with $\eta=0.2$ until zero training error. pMoE converges much faster than CNN, which takes $150$ epochs. 
Before training the experts in the separate-training pMoE, we train the router for $100$ epochs. The models are evaluated on $1000$ test samples.


\textbf{Generalization performance}: Figure \ref{sm_com_mnist} compares the test accuracy of the three models, where   $l=2$ and $l=6$ for separate-training and joint-training pMoE, respectively. The error bars show the mean plus/minus one standard deviation of five independent experiments. 
pMoE outperforms CNN with the same number of training samples.  pMoE only requires 60\% of the training samples needed by CNN to achieve  $95\%$ test accuracy. 

Figure \ref{sm_com_ph_tr_l} shows the sample complexity of separate-training pMoE with respect to $l$. Each block represents 20 independent trials. A white block indicates all success,  and a black block indicates all failure. The sample complexity is polynomial in $l$, verifying Theorem \ref{Thm_mcnn_bin}.  Figure \ref{jpmoe_l_vs_acc} and \ref{jpmoe_k_vs_acc} show  the test accuracy of joint-training pMoE with a fixed sample size when $l$ and $k$ change, respectively. When $l$ is greater than $l^*$, which is $6$ in Figure \ref{jpmoe_l_vs_acc}, the test accuracy matches our predicted order. Similarly,  the dependence on $k$ also matches our prediction, when $k$ is large enough to make Assumption \ref{router_asmptn} hold.

\textbf{Router performance}: 
Figure \ref{fxd_rtr_per} verifies the discriminative property of separately trained routers (Lemma \ref{router_lemma}) by showing the percentage of testing data that have class-discriminative  patterns ($o_1$ and $o_2$) in top $l$ patches of the separately trained router.  With very few training samples (such as $300$), one can already learn a proper router that has discriminative patterns in top-$4$ patches for  95\% of data. 
Figure \ref{jnt_rtr_per} 
 verifies the discriminative property of jointly trained 
  routers (Assumption \ref{router_asmptn}). With only $300$ training samples, the jointly trained router dispatches  $o_1$ with the largest gating value to a particular expert for 95\% of class-1 data and similarly for $o_2$ in   92\%  of class-2 data.  




\begin{figure}[ht]
\vskip 0.1in
    \centering
    \includegraphics[width=0.63\linewidth]{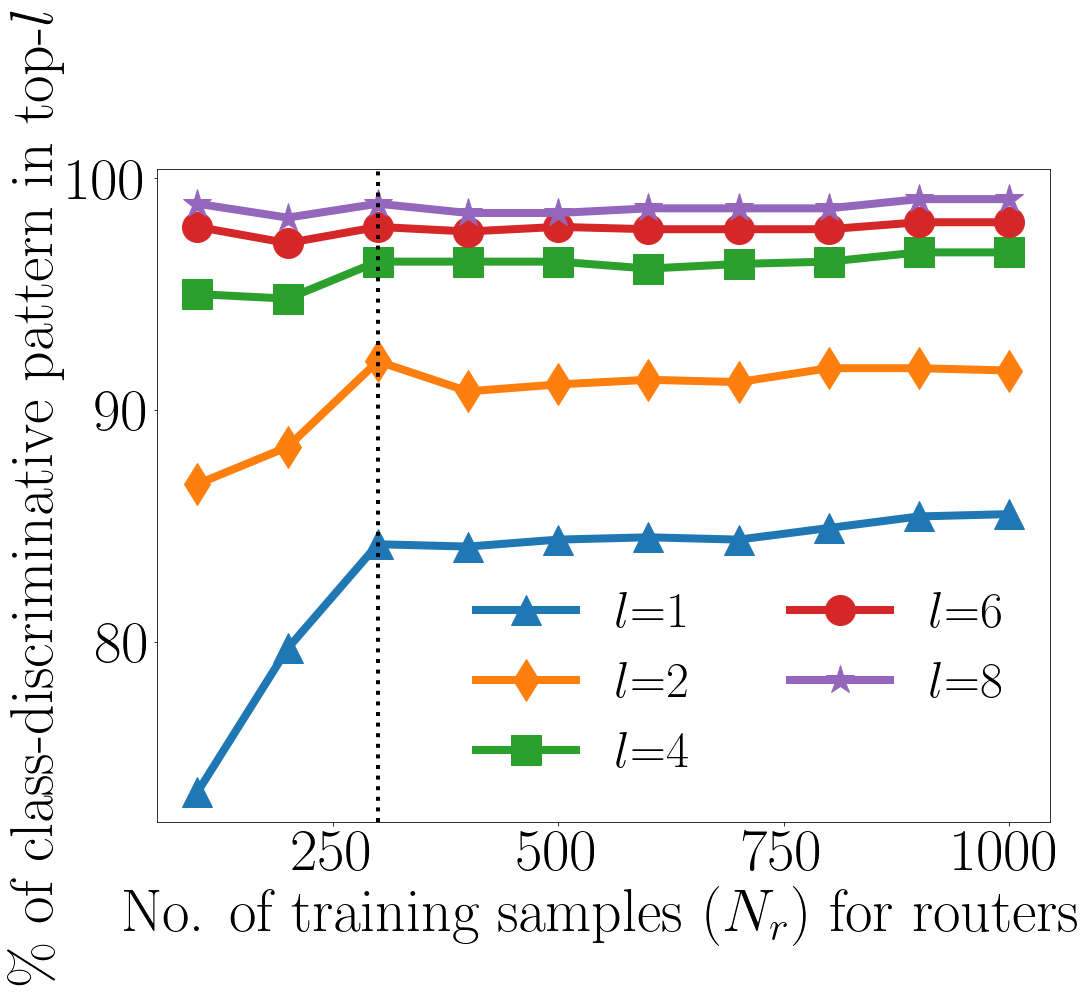}
    \caption{Percentage of properly routed discriminative patterns by  a separately trained router.}
    \label{fxd_rtr_per}
    \vskip -0.1in
\end{figure}

\begin{figure}[ht]
\vskip 0.1in
    \centering
    \includegraphics[width=0.70\linewidth]{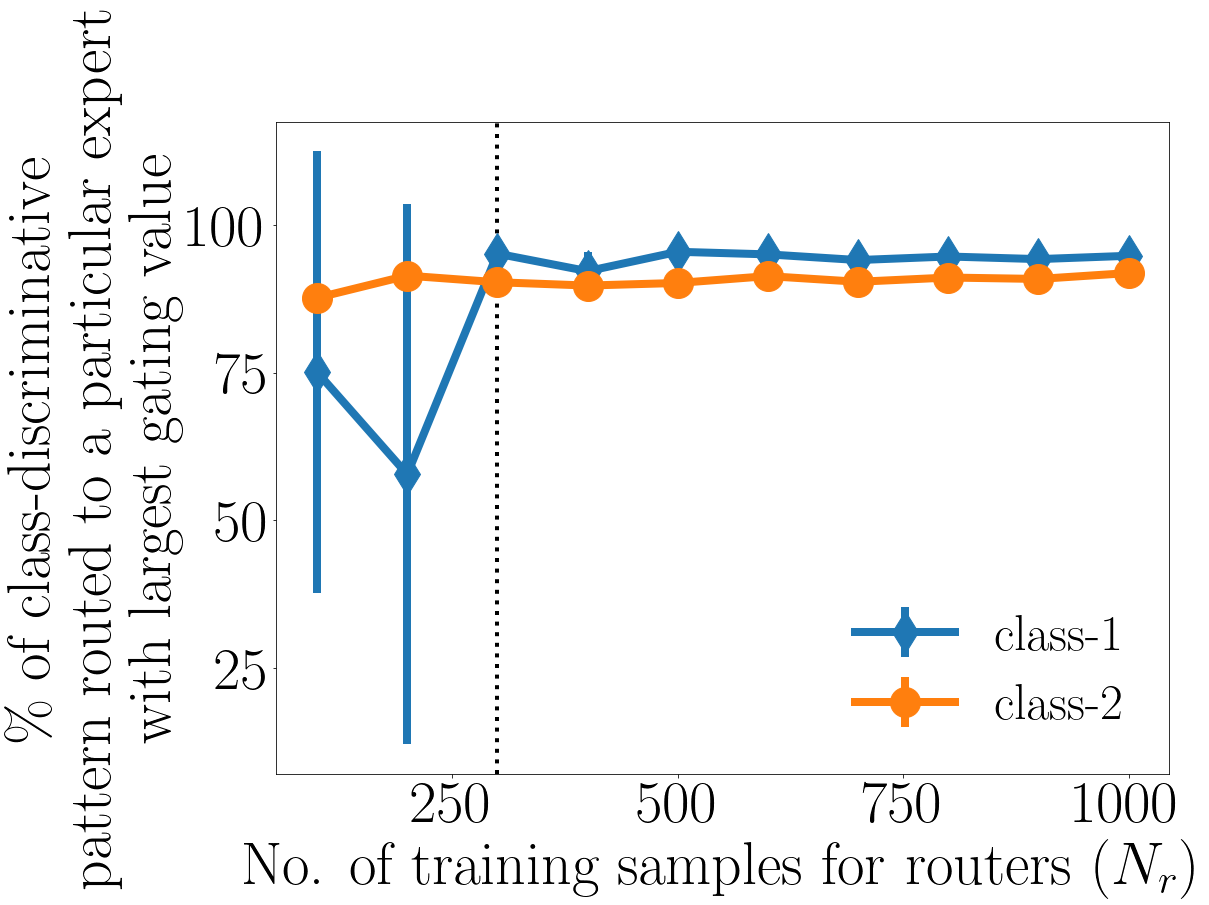}
    \caption{Percentage of properly routed discriminative patterns by a jointly trained router. $l=6$.}
    \label{jnt_rtr_per}
    \vskip -0.1in
\end{figure}

\subsection{pMoE of Wide Residual Networks (WRNs)}

\textbf{Neural network model}: We employ the 10-layer WRN \citep{zagoruyko2016wide} with a widening factor of 10 as the expert. We construct a patch-level MoE counterpart of WRN, referred to as WRN-pMoE,  by replacing the last convolutional layer of WRN with an pMoE layer  of an equal number of trainable parameters  (see Figure \ref{wrn_pmoe} in Appendix for an illustration).  WRN-pMoE is trained with the joint-training method\footnote{Code is available at \url{https://github.com/nowazrabbani/pMoE_CNN}}.  All the results are averaged over five independent experiments.

\textbf{Datasets}: We consider both CelebA \citep{liu2015deep} and CIFAR-10 datasets. The experiments on CIFAR-10 are deferred to the Appendix (see section \ref{cifar_10_exp}). We down-sample the images of CelebA to $64\times64$.   The last convolutional layer of WRN receives a ($16\times16\times640$) dimensional feature map. The feature map is divided into $16$ patches with size $4\times4\times640$ in WRN-pMoE. $k=8$ and $l=2$ for the pMoE layer.

\begin{figure*}[ht]
\vskip 0.1in
\centering
    \begin{minipage}{0.32\linewidth}
        \centering
        \includegraphics[width=0.99\linewidth]{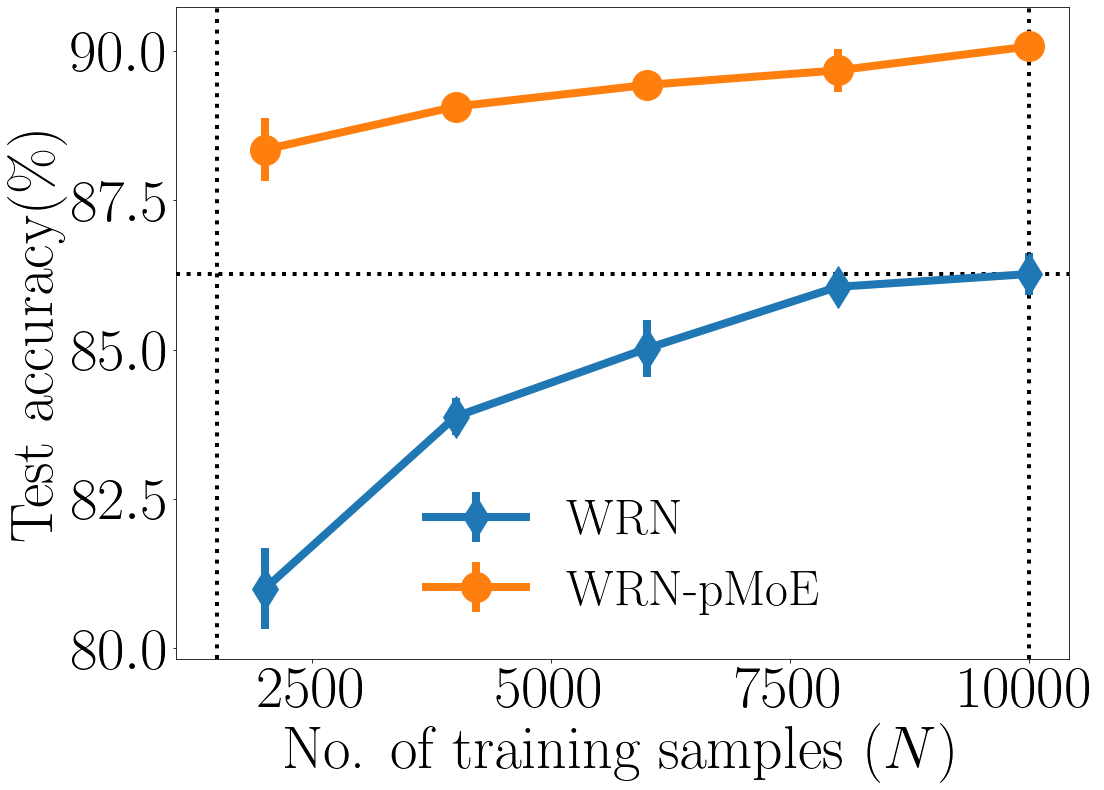}
        \caption{Classification accuracy of  WRN-pMoE and WRN on ``smiling''   in CelebA}
        \label{celeba_smiling}
    \end{minipage}
    ~
    \begin{minipage}{0.32\linewidth}
        \centering
        \includegraphics[width=0.99\linewidth]{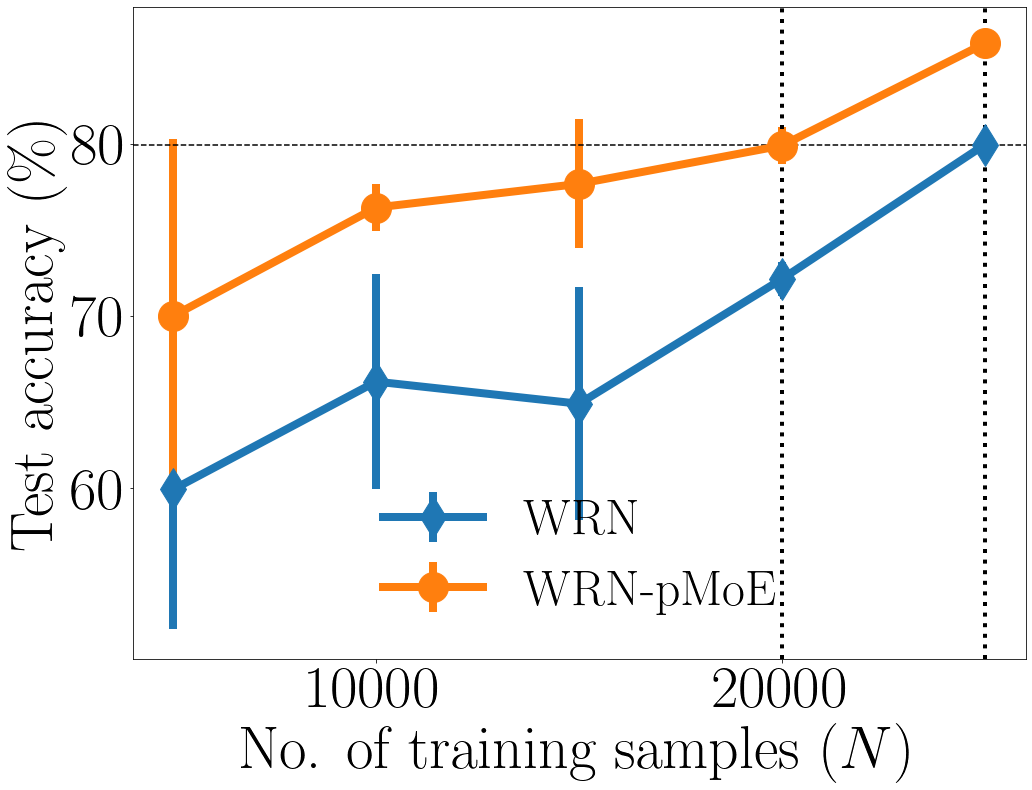}
        \caption{Classification accuracy of  WRN-pMoE and WRN on ``smiling'' when spuriously correlated with ``black hair'' in CelebA}
        \label{celeba_smiling_spurious}
    \end{minipage}
    ~
    \begin{minipage}{0.32\linewidth}
        \centering
        \includegraphics[width=0.99\linewidth]{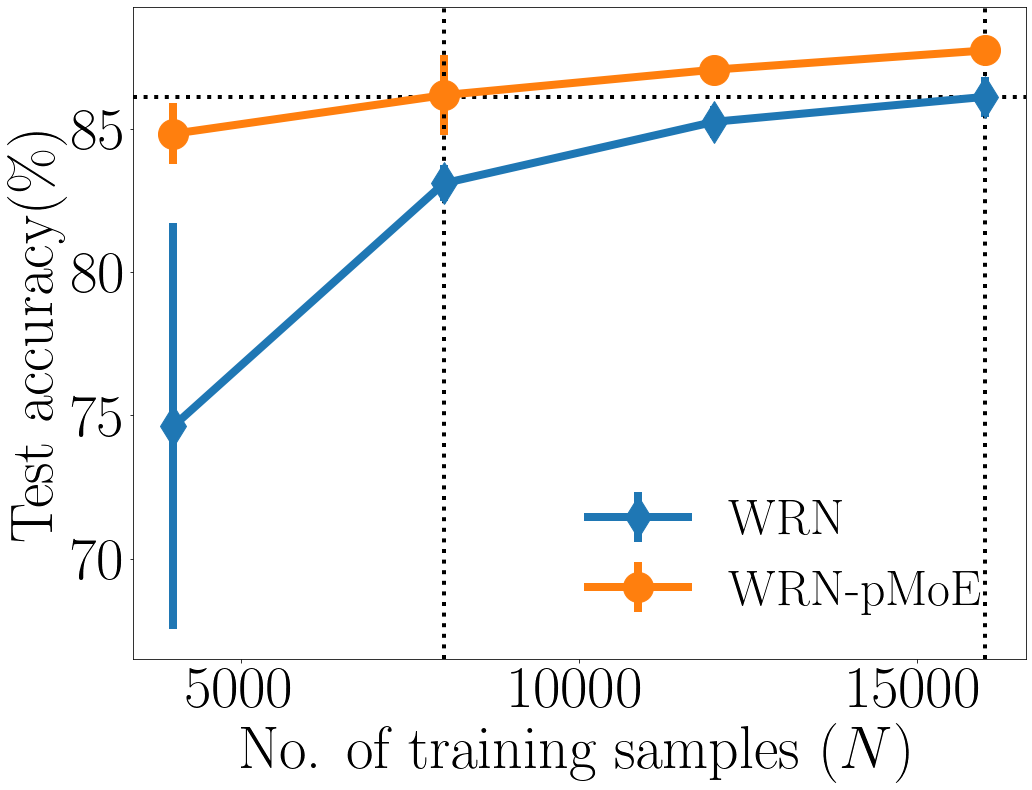}
        \caption{Classification accuracy of  WRN-pMoE and WRN on multiclass classification in CelebA}
        \label{celeba_multi_class}
    \end{minipage}
    \vskip -0.1in
\end{figure*}

\begin{table*}[t]
\caption{Comparison of training compute of WRN and WRN-pMoE.}
\label{table2}
\vskip 0.1in
\begin{center}
\begin{small}
\begin{sc}
\renewcommand{\arraystretch}{1.4}
\begin{tabular}{c|c|c|c|c}
\toprule
\multirow{2}{*}{No. of training samples} & \multicolumn{2}{c|}{Convergence time (sec)} & \multicolumn{2}{c}{Training FLOPs ($\times 10^{15}$)} \\\cline{2-5}
& WRN & WRN-pMoE & WRN & WRN-pMoE \\
\hline
$4000$    & $260$ & $\mathbf{156}$ & $6$ & $\mathbf{3.5}$\\
$8000$    & $324$ & $\mathbf{192}$ & $7.5$ & $\mathbf{4.4}$\\
$12000$    & $468$ & $\mathbf{280}$ & $11$ & $\mathbf{6.4}$\\
$16000$    & $630$ & $\mathbf{368}$ & $15$ & $\mathbf{8.5}$\\
\bottomrule
\end{tabular}
\renewcommand{\arraystretch}{1}
\end{sc}
\end{small}
\end{center}
\vskip -0.1in
\end{table*}

\textbf{Performance Comparison}: Figure \ref{celeba_smiling} shows the test accuracy of the binary classification problem on the attribute ``smiling.''   WRN-pMoE requires less than \textit{one-fifth} of the training samples needed by WRN to achieve 86\% accuracy. Figure \ref{celeba_smiling_spurious} shows the performance when the training data contain spurious correlations with the hair color as a spurious attribute. Specifically,  95\% of the training images with the attribute ``smiling'' also have the attribute ``black hair,'' while
95\% of the training images with the attribute ``not-smiling''   have the attribute ``blond hair.'' The models may learn the hair-color attribute rather than ``smiling'' due to spurious correlation and, thus, the test accuracies are lower in Figure \ref{celeba_smiling_spurious} than those in Figure \ref{celeba_smiling}. Nevertheless, WRN-pMoE outperforms WRN and reduces the sample complexity to achieve the same accuracy. 

Figure \ref{celeba_multi_class} shows the test accuracy of multiclass classification (four classes with class attributes: ``Not smiling, Eyeglass,'' ``Smiling, Eyeglass,'' ``Smiling, No eyeglass,'' and ``Not smiling, No eyeglass'') in CelebA. The results are consistent with    the binary classification results. Furthermore, Table \ref{table2} empirically verifies the computational efficiency of WRN-pMoE over WRN on multiclass classification in CelebA\footnote{An NVIDIA RTX 4500 GPU was used to run the experiments, training FLOPs are calculated as $\text{Training FLOPs}= \text{Training time (second)}\times \text{Number of GPUs}\times \text{peak FLOP/second}\times \text{GPU utilization rate}$}. 
Even with same number of training samples, WRN-pMoE is still more computationally efficient than WRN, because WRN-pMoE requires fewer   iterations to converge and has a lower per-iteration cost.

\section{Conclusion}

 MoE reduces computational costs significantly without hurting the generalization performance in various empirical studies, but the theoretical explanation is mostly elusive. This paper provides the first theoretical analysis of patch-level MoE  and proves its savings in sample complexity and model size quantitatively compared with the single-expert counterpart. Although   centered on a classification task using a mixture of two-layer CNNs, our theoretical insights are verified empirically on deep architectures and multiple datasets. Future works include analyzing other MoE architectures such as MoE in Vision Transformer (ViT) and connecting MoE with other sparsification methods to further reduce the computation.

\section*{Acknowledgements}

This work was supported by AFOSR FA9550-20-1-0122, NSF 1932196 and the Rensselaer-IBM AI Research Collaboration (http://airc.rpi.edu), part of the IBM AI Horizons Network (http://ibm.biz/AIHorizons). We thank Yihua Zhang at Michigan State University for the help in experiments with CelebA dataset. We thank all anonymous reviewers.

\bibliography{Reference/reference}

\begin{thebibliography}{61}
\providecommand{\natexlab}[1]{#1}
\providecommand{\url}[1]{\texttt{#1}}
\expandafter\ifx\csname urlstyle\endcsname\relax
  \providecommand{\doi}[1]{doi: #1}\else
  \providecommand{\doi}{doi: \begingroup \urlstyle{rm}\Url}\fi

\bibitem[Ahmed et~al.(2016)Ahmed, Baig, and Torresani]{ahmed2016network}
Ahmed, K., Baig, M.~H., and Torresani, L.
\newblock Network of experts for large-scale image categorization.
\newblock In \emph{European Conference on Computer Vision}, pp.\  516--532.
  Springer, 2016.

\bibitem[Allen-Zhu \& Li(2019)Allen-Zhu and Li]{allen2019can}
Allen-Zhu, Z. and Li, Y.
\newblock What can resnet learn efficiently, going beyond kernels?
\newblock \emph{Advances in Neural Information Processing Systems}, 32, 2019.

\bibitem[Allen-Zhu \& Li(2020{\natexlab{a}})Allen-Zhu and
  Li]{allen2020backward}
Allen-Zhu, Z. and Li, Y.
\newblock Backward feature correction: How deep learning performs deep
  learning.
\newblock \emph{arXiv preprint arXiv:2001.04413}, 2020{\natexlab{a}}.

\bibitem[Allen-Zhu \& Li(2020{\natexlab{b}})Allen-Zhu and Li]{allen2020towards}
Allen-Zhu, Z. and Li, Y.
\newblock Towards understanding ensemble, knowledge distillation and
  self-distillation in deep learning.
\newblock \emph{arXiv preprint arXiv:2012.09816}, 2020{\natexlab{b}}.

\bibitem[Allen-Zhu \& Li(2022)Allen-Zhu and Li]{allen2022feature}
Allen-Zhu, Z. and Li, Y.
\newblock Feature purification: How adversarial training performs robust deep
  learning.
\newblock In \emph{2021 IEEE 62nd Annual Symposium on Foundations of Computer
  Science (FOCS)}, pp.\  977--988. IEEE, 2022.

\bibitem[Allen-Zhu et~al.(2019{\natexlab{a}})Allen-Zhu, Li, and
  Liang]{allen2019learning}
Allen-Zhu, Z., Li, Y., and Liang, Y.
\newblock Learning and generalization in overparameterized neural networks,
  going beyond two layers.
\newblock \emph{Advances in neural information processing systems}, 32,
  2019{\natexlab{a}}.

\bibitem[Allen-Zhu et~al.(2019{\natexlab{b}})Allen-Zhu, Li, and
  Song]{allen2019convergence}
Allen-Zhu, Z., Li, Y., and Song, Z.
\newblock A convergence theory for deep learning via over-parameterization.
\newblock In \emph{International Conference on Machine Learning}, pp.\
  242--252. PMLR, 2019{\natexlab{b}}.

\bibitem[Arora et~al.(2019)Arora, Du, Hu, Li, and Wang]{arora2019fine}
Arora, S., Du, S., Hu, W., Li, Z., and Wang, R.
\newblock Fine-grained analysis of optimization and generalization for
  overparameterized two-layer neural networks.
\newblock In \emph{International Conference on Machine Learning}, pp.\
  322--332. PMLR, 2019.

\bibitem[Bai \& Lee(2019)Bai and Lee]{bai2019beyond}
Bai, Y. and Lee, J.~D.
\newblock Beyond linearization: On quadratic and higher-order approximation of
  wide neural networks.
\newblock In \emph{International Conference on Learning Representations}, 2019.

\bibitem[Bengio et~al.(2013)Bengio, L{\'e}onard, and
  Courville]{bengio2013estimating}
Bengio, Y., L{\'e}onard, N., and Courville, A.
\newblock Estimating or propagating gradients through stochastic neurons for
  conditional computation.
\newblock \emph{arXiv preprint arXiv:1308.3432}, 2013.

\bibitem[Brutzkus \& Globerson(2021)Brutzkus and
  Globerson]{brutzkus2021optimization}
Brutzkus, A. and Globerson, A.
\newblock An optimization and generalization analysis for max-pooling networks.
\newblock In \emph{Uncertainty in Artificial Intelligence}, pp.\  1650--1660.
  PMLR, 2021.

\bibitem[Brutzkus et~al.(2018)Brutzkus, Globerson, Malach, and
  Shalev-Shwartz]{brutzkus2018sgd}
Brutzkus, A., Globerson, A., Malach, E., and Shalev-Shwartz, S.
\newblock {SGD} learns over-parameterized networks that provably generalize on
  linearly separable data.
\newblock In \emph{International Conference on Learning Representations}, 2018.

\bibitem[Chen et~al.(1999)Chen, Xu, and Chi]{chen1999improved}
Chen, K., Xu, L., and Chi, H.
\newblock Improved learning algorithms for mixture of experts in multiclass
  classification.
\newblock \emph{Neural networks}, 12\penalty0 (9):\penalty0 1229--1252, 1999.

\bibitem[Chen et~al.(2022)Chen, Deng, Wu, Gu, and Li]{chen2022towards}
Chen, Z., Deng, Y., Wu, Y., Gu, Q., and Li, Y.
\newblock Towards understanding mixture of experts in deep learning.
\newblock \emph{arXiv preprint arXiv:2208.02813}, 2022.

\bibitem[Chizat et~al.(2019)Chizat, Oyallon, and Bach]{chizat2019lazy}
Chizat, L., Oyallon, E., and Bach, F.
\newblock On lazy training in differentiable programming.
\newblock \emph{Advances in Neural Information Processing Systems}, 32, 2019.

\bibitem[Collobert et~al.(2001)Collobert, Bengio, and
  Bengio]{collobert2001parallel}
Collobert, R., Bengio, S., and Bengio, Y.
\newblock A parallel mixture of {SVM}s for very large scale problems.
\newblock \emph{Advances in Neural Information Processing Systems}, 14, 2001.

\bibitem[Collobert et~al.(2003)Collobert, Bengio, and
  Bengio]{collobert2003scaling}
Collobert, R., Bengio, Y., and Bengio, S.
\newblock Scaling large learning problems with hard parallel mixtures.
\newblock \emph{International Journal of pattern recognition and artificial
  intelligence}, 17\penalty0 (03):\penalty0 349--365, 2003.

\bibitem[Daniely \& Malach(2020)Daniely and Malach]{daniely2020learning}
Daniely, A. and Malach, E.
\newblock Learning parities with neural networks.
\newblock \emph{Advances in Neural Information Processing Systems},
  33:\penalty0 20356--20365, 2020.

\bibitem[Du et~al.(2019)Du, Lee, Li, Wang, and Zhai]{du2019gradient}
Du, S., Lee, J., Li, H., Wang, L., and Zhai, X.
\newblock Gradient descent finds global minima of deep neural networks.
\newblock In \emph{International conference on machine learning}, pp.\
  1675--1685. PMLR, 2019.

\bibitem[Eigen et~al.(2013)Eigen, Ranzato, and Sutskever]{eigen2013learning}
Eigen, D., Ranzato, M., and Sutskever, I.
\newblock Learning factored representations in a deep mixture of experts.
\newblock \emph{arXiv preprint arXiv:1312.4314}, 2013.

\bibitem[Fedus et~al.(2022)Fedus, Zoph, and Shazeer]{fedus2022switch}
Fedus, W., Zoph, B., and Shazeer, N.
\newblock Switch transformers: Scaling to trillion parameter models with simple
  and efficient sparsity.
\newblock \emph{Journal of Machine Learning Research}, 23\penalty0
  (120):\penalty0 1--39, 2022.

\bibitem[Fu et~al.(2020)Fu, Chi, and Liang]{fu2020guaranteed}
Fu, H., Chi, Y., and Liang, Y.
\newblock Guaranteed recovery of one-hidden-layer neural networks via cross
  entropy.
\newblock \emph{IEEE transactions on signal processing}, 68:\penalty0
  3225--3235, 2020.

\bibitem[Ghorbani et~al.(2019)Ghorbani, Mei, Misiakiewicz, and
  Montanari]{ghorbani2019limitations}
Ghorbani, B., Mei, S., Misiakiewicz, T., and Montanari, A.
\newblock Limitations of lazy training of two-layers neural network.
\newblock \emph{Advances in Neural Information Processing Systems}, 32, 2019.

\bibitem[Ghorbani et~al.(2020)Ghorbani, Mei, Misiakiewicz, and
  Montanari]{ghorbani2020neural}
Ghorbani, B., Mei, S., Misiakiewicz, T., and Montanari, A.
\newblock When do neural networks outperform kernel methods?
\newblock \emph{Advances in Neural Information Processing Systems},
  33:\penalty0 14820--14830, 2020.

\bibitem[Ghorbani et~al.(2021)Ghorbani, Mei, Misiakiewicz, and
  Montanari]{ghorbani2021linearized}
Ghorbani, B., Mei, S., Misiakiewicz, T., and Montanari, A.
\newblock Linearized two-layers neural networks in high dimension.
\newblock \emph{The Annals of Statistics}, 49\penalty0 (2):\penalty0
  1029--1054, 2021.

\bibitem[Gross et~al.(2017)Gross, Ranzato, and Szlam]{gross2017hard}
Gross, S., Ranzato, M., and Szlam, A.
\newblock Hard mixtures of experts for large scale weakly supervised vision.
\newblock In \emph{Proceedings of the IEEE Conference on Computer Vision and
  Pattern Recognition}, pp.\  6865--6873, 2017.

\bibitem[Jacobs et~al.(1991)Jacobs, Jordan, Nowlan, and
  Hinton]{jacobs1991adaptive}
Jacobs, R.~A., Jordan, M.~I., Nowlan, S.~J., and Hinton, G.~E.
\newblock Adaptive mixtures of local experts.
\newblock \emph{Neural computation}, 3\penalty0 (1):\penalty0 79--87, 1991.

\bibitem[Jacot et~al.(2018)Jacot, Gabriel, and Hongler]{jacot2018neural}
Jacot, A., Gabriel, F., and Hongler, C.
\newblock Neural tangent kernel: Convergence and generalization in neural
  networks.
\newblock \emph{Advances in neural information processing systems}, 31, 2018.

\bibitem[Ji \& Telgarsky(2019)Ji and Telgarsky]{ji2019polylogarithmic}
Ji, Z. and Telgarsky, M.
\newblock Polylogarithmic width suffices for gradient descent to achieve
  arbitrarily small test error with shallow relu networks.
\newblock In \emph{International Conference on Learning Representations}, 2019.

\bibitem[Jordan \& Jacobs(1994)Jordan and Jacobs]{jordan1994hierarchical}
Jordan, M.~I. and Jacobs, R.~A.
\newblock Hierarchical mixtures of experts and the em algorithm.
\newblock \emph{Neural computation}, 6\penalty0 (2):\penalty0 181--214, 1994.

\bibitem[Karp et~al.(2021)Karp, Winston, Li, and Singh]{karp2021local}
Karp, S., Winston, E., Li, Y., and Singh, A.
\newblock Local signal adaptivity: Provable feature learning in neural networks
  beyond kernels.
\newblock \emph{Advances in Neural Information Processing Systems},
  34:\penalty0 24883--24897, 2021.

\bibitem[Krizhevsky(2009)]{krizhevsky2009learning}
Krizhevsky, A.
\newblock Learning multiple layers of features from tiny images.
\newblock Technical report, Canadian Institute For Advanced Research, 2009.

\bibitem[LeCun et~al.(2010)LeCun, Cortes, and Burges]{lecun2010mnist}
LeCun, Y., Cortes, C., and Burges, C.
\newblock {MNIST} handwritten digit database. {AT}\&{T} labs [online].
  available http.
\newblock \emph{yann. lecun. com/exdb/mnist}, 2010.

\bibitem[Lee et~al.(2019)Lee, Xiao, Schoenholz, Bahri, Novak, Sohl-Dickstein,
  and Pennington]{lee2019wide}
Lee, J., Xiao, L., Schoenholz, S., Bahri, Y., Novak, R., Sohl-Dickstein, J.,
  and Pennington, J.
\newblock Wide neural networks of any depth evolve as linear models under
  gradient descent.
\newblock \emph{Advances in neural information processing systems}, 32, 2019.

\bibitem[Lepikhin et~al.(2020)Lepikhin, Lee, Xu, Chen, Firat, Huang, Krikun,
  Shazeer, and Chen]{lepikhin2020gshard}
Lepikhin, D., Lee, H., Xu, Y., Chen, D., Firat, O., Huang, Y., Krikun, M.,
  Shazeer, N., and Chen, Z.
\newblock Gshard: Scaling giant models with conditional computation and
  automatic sharding.
\newblock In \emph{International Conference on Learning Representations}, 2020.

\bibitem[Lewis et~al.(2021)Lewis, Bhosale, Dettmers, Goyal, and
  Zettlemoyer]{lewis2021base}
Lewis, M., Bhosale, S., Dettmers, T., Goyal, N., and Zettlemoyer, L.
\newblock Base layers: Simplifying training of large, sparse models.
\newblock In \emph{International Conference on Machine Learning}, pp.\
  6265--6274. PMLR, 2021.

\bibitem[Li et~al.(2022{\natexlab{a}})Li, Wang, Liu, Chen, and
  Xiong]{li2022generalization}
Li, H., Wang, M., Liu, S., Chen, P.-Y., and Xiong, J.
\newblock Generalization guarantee of training graph convolutional networks
  with graph topology sampling.
\newblock In \emph{International Conference on Machine Learning}, pp.\
  13014--13051. PMLR, 2022{\natexlab{a}}.

\bibitem[Li et~al.(2022{\natexlab{b}})Li, Zhang, and Wang]{li2022learning}
Li, H., Zhang, S., and Wang, M.
\newblock Learning and generalization of one-hidden-layer neural networks,
  going beyond standard gaussian data.
\newblock In \emph{2022 56th Annual Conference on Information Sciences and
  Systems (CISS)}, pp.\  37--42. IEEE, 2022{\natexlab{b}}.

\bibitem[Li et~al.(2023)Li, Wang, Liu, and Chen]{li2023a}
Li, H., Wang, M., Liu, S., and Chen, P.-Y.
\newblock A theoretical understanding of shallow vision transformers: Learning,
  generalization, and sample complexity.
\newblock In \emph{The Eleventh International Conference on Learning
  Representations}, 2023.
\newblock URL \url{https://openreview.net/forum?id=jClGv3Qjhb}.

\bibitem[Li \& Liang(2018)Li and Liang]{li2018learning}
Li, Y. and Liang, Y.
\newblock Learning overparameterized neural networks via stochastic gradient
  descent on structured data.
\newblock \emph{Advances in neural information processing systems}, 31, 2018.

\bibitem[Li et~al.(2020)Li, Ma, and Zhang]{li2020learning}
Li, Y., Ma, T., and Zhang, H.~R.
\newblock Learning over-parametrized two-layer neural networks beyond {NTK}.
\newblock In \emph{Conference on learning theory}, pp.\  2613--2682. PMLR,
  2020.

\bibitem[Liu et~al.(2015)Liu, Luo, Wang, and Tang]{liu2015deep}
Liu, Z., Luo, P., Wang, X., and Tang, X.
\newblock Deep learning face attributes in the wild.
\newblock In \emph{Proceedings of the IEEE international conference on computer
  vision}, pp.\  3730--3738, 2015.

\bibitem[Malach et~al.(2021)Malach, Kamath, Abbe, and
  Srebro]{malach2021quantifying}
Malach, E., Kamath, P., Abbe, E., and Srebro, N.
\newblock Quantifying the benefit of using differentiable learning over tangent
  kernels.
\newblock In \emph{International Conference on Machine Learning}, pp.\
  7379--7389. PMLR, 2021.

\bibitem[Ramachandran \& Le(2018)Ramachandran and
  Le]{ramachandran2018diversity}
Ramachandran, P. and Le, Q.~V.
\newblock Diversity and depth in per-example routing models.
\newblock In \emph{International Conference on Learning Representations}, 2018.

\bibitem[Rasmussen \& Ghahramani(2001)Rasmussen and
  Ghahramani]{rasmussen2001infinite}
Rasmussen, C. and Ghahramani, Z.
\newblock Infinite mixtures of gaussian process experts.
\newblock \emph{Advances in neural information processing systems}, 14, 2001.

\bibitem[Riquelme et~al.(2021)Riquelme, Puigcerver, Mustafa, Neumann, Jenatton,
  Susano~Pinto, Keysers, and Houlsby]{riquelme2021scaling}
Riquelme, C., Puigcerver, J., Mustafa, B., Neumann, M., Jenatton, R.,
  Susano~Pinto, A., Keysers, D., and Houlsby, N.
\newblock Scaling vision with sparse mixture of experts.
\newblock \emph{Advances in Neural Information Processing Systems},
  34:\penalty0 8583--8595, 2021.

\bibitem[Shalev-Shwartz et~al.(2020)]{shalev2020computational}
Shalev-Shwartz, S. et~al.
\newblock Computational separation between convolutional and fully-connected
  networks.
\newblock In \emph{International Conference on Learning Representations}, 2020.

\bibitem[Shazeer et~al.(2017)Shazeer, Mirhoseini, Maziarz, Davis, Le, Hinton,
  and Dean]{shazeer2017outrageously}
Shazeer, N., Mirhoseini, A., Maziarz, K., Davis, A., Le, Q.~V., Hinton, G.~E.,
  and Dean, J.
\newblock Outrageously large neural networks: The sparsely-gated
  mixture-of-experts layer.
\newblock In \emph{International Conference on Learning Representations}, 2017.

\bibitem[Shi et~al.(2021)Shi, Wei, and Liang]{shi2021theoretical}
Shi, Z., Wei, J., and Liang, Y.
\newblock A theoretical analysis on feature learning in neural networks:
  Emergence from inputs and advantage over fixed features.
\newblock In \emph{International Conference on Learning Representations}, 2021.

\bibitem[Tresp(2000)]{NIPS2000_9fdb62f9}
Tresp, V.
\newblock Mixtures of gaussian processes.
\newblock In Leen, T., Dietterich, T., and Tresp, V. (eds.), \emph{Advances in
  Neural Information Processing Systems}, volume~13. MIT Press, 2000.
\newblock URL
  \url{https://proceedings.neurips.cc/paper/2000/file/9fdb62f932adf55af2c0e09e55861964-Paper.pdf}.

\bibitem[Vaswani et~al.(2017)Vaswani, Shazeer, Parmar, Uszkoreit, Jones, Gomez,
  Kaiser, and Polosukhin]{vaswani2017attention}
Vaswani, A., Shazeer, N., Parmar, N., Uszkoreit, J., Jones, L., Gomez, A.~N.,
  Kaiser, {\L}., and Polosukhin, I.
\newblock Attention is all you need.
\newblock \emph{Advances in neural information processing systems}, 30, 2017.

\bibitem[Yang et~al.(2019)Yang, Bender, Le, and Ngiam]{yang2019condconv}
Yang, B., Bender, G., Le, Q.~V., and Ngiam, J.
\newblock Condconv: Conditionally parameterized convolutions for efficient
  inference.
\newblock \emph{Advances in Neural Information Processing Systems}, 32, 2019.

\bibitem[Yehudai \& Shamir(2019)Yehudai and Shamir]{yehudai2019power}
Yehudai, G. and Shamir, O.
\newblock On the power and limitations of random features for understanding
  neural networks.
\newblock \emph{Advances in Neural Information Processing Systems}, 32, 2019.

\bibitem[Yu et~al.(2019)Yu, Zhang, and Zhu]{yu2019learning}
Yu, B., Zhang, J., and Zhu, Z.
\newblock On the learning dynamics of two-layer nonlinear convolutional neural
  networks.
\newblock \emph{arXiv preprint arXiv:1905.10157}, 2019.

\bibitem[Zagoruyko \& Komodakis(2016)Zagoruyko and
  Komodakis]{zagoruyko2016wide}
Zagoruyko, S. and Komodakis, N.
\newblock Wide residual networks.
\newblock \emph{arXiv preprint arXiv:1605.07146}, 2016.

\bibitem[Zhang et~al.(2020{\natexlab{a}})Zhang, Wang, Liu, Chen, and
  Xiong]{zhang2020fast}
Zhang, S., Wang, M., Liu, S., Chen, P.-Y., and Xiong, J.
\newblock Fast learning of graph neural networks with guaranteed
  generalizability: one-hidden-layer case.
\newblock In \emph{International Conference on Machine Learning}, pp.\
  11268--11277. PMLR, 2020{\natexlab{a}}.

\bibitem[Zhang et~al.(2020{\natexlab{b}})Zhang, Wang, Xiong, Liu, and
  Chen]{zhang2020improved}
Zhang, S., Wang, M., Xiong, J., Liu, S., and Chen, P.-Y.
\newblock Improved linear convergence of training {CNN}s with generalizability
  guarantees: A one-hidden-layer case.
\newblock \emph{IEEE Transactions on Neural Networks and Learning Systems},
  32\penalty0 (6):\penalty0 2622--2635, 2020{\natexlab{b}}.

\bibitem[Zhong et~al.(2017{\natexlab{a}})Zhong, Song, and
  Dhillon]{zhong2017learning}
Zhong, K., Song, Z., and Dhillon, I.~S.
\newblock Learning non-overlapping convolutional neural networks with multiple
  kernels.
\newblock \emph{arXiv preprint arXiv:1711.03440}, 2017{\natexlab{a}}.

\bibitem[Zhong et~al.(2017{\natexlab{b}})Zhong, Song, Jain, Bartlett, and
  Dhillon]{zhong2017recovery}
Zhong, K., Song, Z., Jain, P., Bartlett, P.~L., and Dhillon, I.~S.
\newblock Recovery guarantees for one-hidden-layer neural networks.
\newblock In \emph{International conference on machine learning}, pp.\
  4140--4149. PMLR, 2017{\natexlab{b}}.

\bibitem[Zhou et~al.(2022)Zhou, Lei, Liu, Du, Huang, Zhao, Dai, Chen, Le, and
  Laudon]{zhou2022mixtureofexperts}
Zhou, Y., Lei, T., Liu, H., Du, N., Huang, Y., Zhao, V.~Y., Dai, A.~M., Chen,
  Z., Le, Q.~V., and Laudon, J.
\newblock Mixture-of-experts with expert choice routing.
\newblock In Oh, A.~H., Agarwal, A., Belgrave, D., and Cho, K. (eds.),
  \emph{Advances in Neural Information Processing Systems}, 2022.
\newblock URL \url{https://openreview.net/forum?id=jdJo1HIVinI}.

\bibitem[Zou et~al.(2020)Zou, Cao, Zhou, and Gu]{zou2020gradient}
Zou, D., Cao, Y., Zhou, D., and Gu, Q.
\newblock Gradient descent optimizes over-parameterized deep relu networks.
\newblock \emph{Machine learning}, 109\penalty0 (3):\penalty0 467--492, 2020.

\end{thebibliography}
\bibliographystyle{icml2023}

\newpage
\appendix
\onecolumn
\section{Experiments on CIFAR-10 Datasets}\label{cifar_10_exp}

We also compare WRN and WRN-pMoE on CIFAR-10-based datasets. To better reflect local features, in addition to the original CIFAR-10, we adopt techniques of \citet{karp2021local} to generate two datasets based on CIFAR-10: 

\textbf{1. CIFAR-10 with \textsc{Image}\textsc{Net} noise.} Each CIFAR-10 image is down-sampled to size $16\times16$  and placed at a random location of a background image chosen from ImageNet Plants synset. Figure \ref{cifr_10_example_images}(c) shows an example image of this dataset.

\textbf{2. CIFAR-\textsc{Vehicles}.}  Each vehicle image of CIFAR-10 is down-sampled to size $16\times16$   and placed in one quadrant of an image randomly where the other quadrants are randomly filled with down-sampled animal images in CIFAR-10. See Figure \ref{cifr_10_example_images}(b) for a sample image. 

The last convolutional layer of WRN receives a $(8\times 8\times 640)$ dimensional feature map. 
In WRN-pMoE we divide this feature map into $64$ patches with size $(1\times 1\times640)$. The MoE layer of WRN-pMoE contains $k=4$ experts with each expert receiving $l=16$ patches. 
\begin{figure}[ht]
\vskip 0.1in
\hspace{0.5cm}
    \begin{minipage}{0.40\linewidth}
            \centering
            \includegraphics[width=1\linewidth]{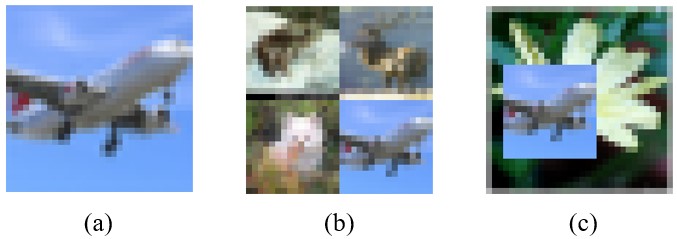}
            \caption{Example images from (a) CIFAR-10, (b) CIFAR-\textsc{Vehicles}, and  (c) CIFAR-10, \textsc{Image}\textsc{Net} noise datasets}
            \label{cifr_10_example_images}
    \end{minipage}
    ~
    \hspace{0.4cm}
    \begin{minipage}{0.50\linewidth}
        \centering
        \includegraphics[width=0.63\linewidth]{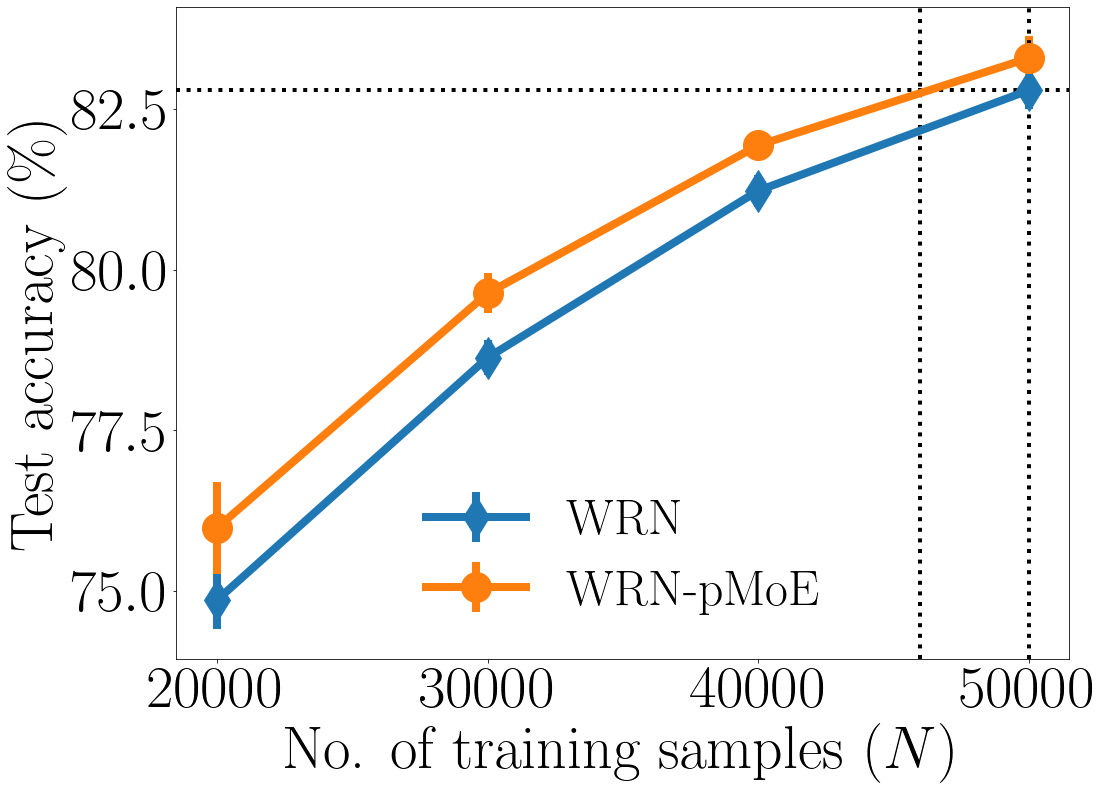}
        \caption{Ten-classification accuracy of WRN and WRN-pMoE on CIFAR-10}
        \label{cifr_10}
    \end{minipage}
    \vskip -0.1in
\end{figure}
\begin{figure}[h]
\vskip 0.1in
    \begin{minipage}{0.48\linewidth}
        \centering
        \includegraphics[width=0.63\linewidth]{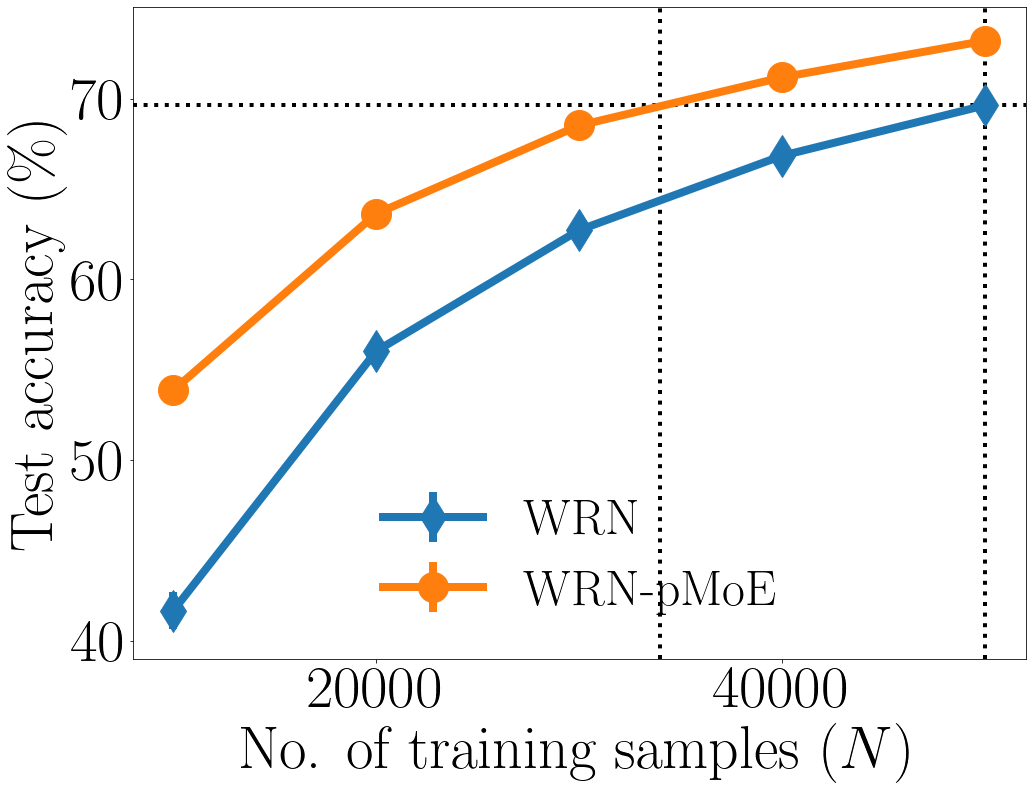}
        \caption{Ten-classification accuracy of  WRN and WRN-pMoE on CIFAR-10, \textsc{Image}\textsc{Net} noise}
        \label{cifr_10_imgnt}
    \end{minipage}
    ~
    \begin{minipage}{0.48\linewidth}
        \centering
        \includegraphics[width=0.63\linewidth]{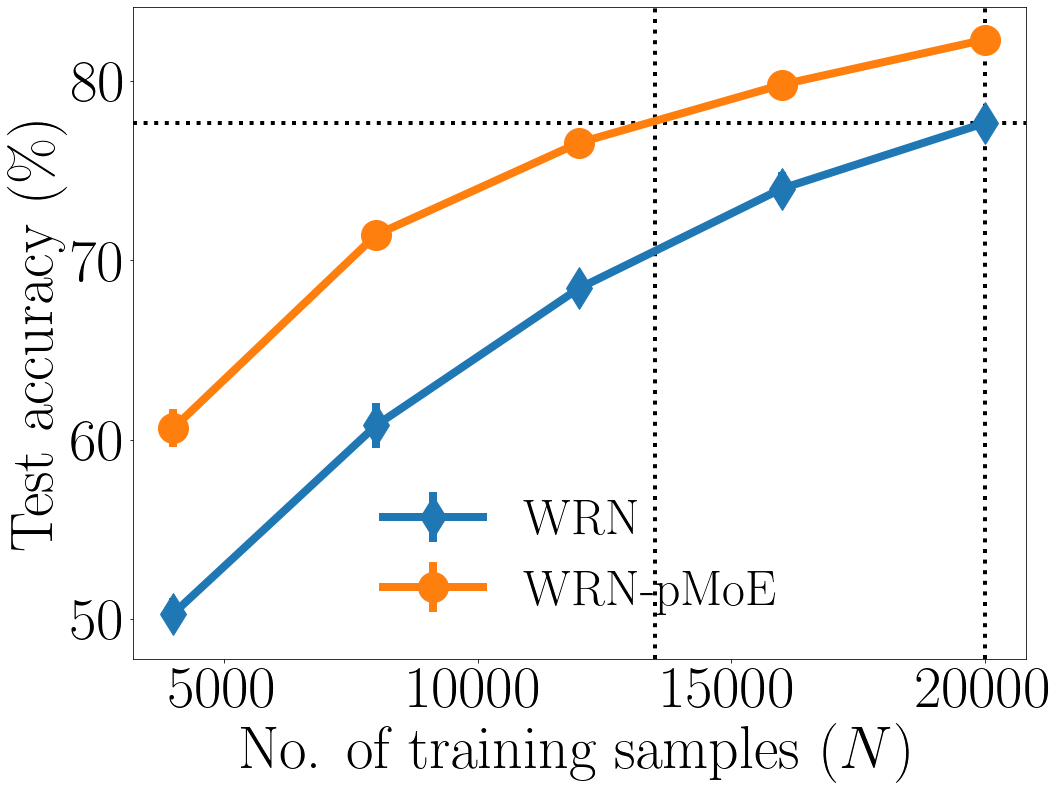}
        \caption{Four-classification accuracy of  WRN and WRN-pMoE on CIFAR-\textsc{Vehicles}}
        \label{cifr_10_block}
    \end{minipage}
    \vskip -0.1in
\end{figure}

Figures \ref{cifr_10}, \ref{cifr_10_imgnt}, and  \ref{cifr_10_block} compare the test accuracy of WRN and WRN-pMoE for the ten-classification problem on CIFAR10 and  CIFAR-10 with \textsc{Image}\textsc{Net} noise, and the four-classification problem in  CIFAR-\textsc{Vehicles}, respectively. WRN-pMoE outperforms WRN in all these datasets, indicating reduced sample complexity using the pMoE layer. The performance gap is more significant in the other two datasets than the original CIFAR-10 dataset. That is because these constructed datasets contain local features, and the pMoE layer has a clear advantage in learning local features effectively. 

\section{Preliminaries}
The loss function for SGD at iteration $t$ with minibatch $\mathcal{B}_t$:
\begin{equation}\label{loss_sgd}
    \mathcal{L}(\theta^{(t)}):=\cfrac{1}{B}\sum_{(x,y)\in\mathcal{B}_t}\log{(1+e^{-yf_M(\theta^{(t)}, x)})}
\end{equation}
For the router-training in separate-training pMoE, the loss function of SGD at iteration $t$ with minibatch $\mathcal{B}_t^r$:
\begin{equation}
    \ell_r(w_1^{(t)},w_2^{(t)}):=-\cfrac{1}{B_r}\hspace{0.1cm}\sum_{(x,y)\in\mathcal{B}_t^r} y  \langle w^{(t)}_1-w_2^{(t)}, \sum_{j=1}^n x^{(j)}\rangle
\end{equation}

\textbf{Notations:}
\begin{enumerate}
    \item Generally $\Tilde{O}(.)$ and $\Tilde{\Omega}(.)$ hides factor $\log(\text{poly}(m,n,p,\delta,\frac{1}{\epsilon}))$. At Lemma \ref{a_lemma_3} and \ref{a_corollary_1}, $\Tilde{\Omega}(.)$ hides factor $\log(\text{poly}(n))$.
    \item Generally with high probability (abbreviated as w.h.p.) implies with probability $1-\cfrac{1}{\text{poly}(m,n,p,\delta,\frac{1}{\epsilon})}$, where $\text{poly}(.)$ implies a sufficiently large polynomial. At Lemma \ref{a_lemma_2}, \ref{a_lemma_3} and \ref{a_corollary_1} ``w.h.p.'' implies $1-\cfrac{1}{\text{poly}(n)}$.
    \item We denote, $\sigma=\frac{1}{\sqrt{m}}$ such that the expert initialization, $w_{r,s}^{(0)}\sim \mathcal{N}(0,\sigma^2\mathbb{I}_{d\times d}), 
    \forall s \in [k], \forall r\in[m/k]$. 
\end{enumerate}
The training algorithms for separate-training and joint-training pMoE are given in Algorithm \ref{alg:1} and Algorithm \ref{alg:2}, respectively:
\begin{algorithm}\caption{Two-phase SGD for separate-training pMoE}\label{alg:1}
\textbf{Input} : Training data $\{(x_i,y_i)\}_{i=1}^N$, learning rates $\eta_r$ and $\eta$, number of iterations $T_r$ and $T$, batch- \\
\indent\hspace{1.2cm}sizes $B_r$ and $B$\\
\textbf{Step-1}: Initialize $w_s^{(0)}, w_{r,s}^{(0)}, a_{r,s}, \forall s\in\{1,2\}, r\in[m/k]$ according to (\ref{eqn:router_ini}) and (\ref{eqn:expert_ini})\\
\textbf{Step-2}: for $t=0,1, ... ,T_r-1$ do:\\
\centerline{$w_{s}^{(t+1)}=w_{s}^{(t)}-\eta_r\cfrac{\partial \ell_r(w_1^{(t)},w_2^{(t)})}{\partial w_s^{(t)}}, \forall s\in\{1,2\}$}\\
\textbf{Step-3}: for $t=0, 1, ..., T-1$ do:\\
\centerline{$w_{r,s}^{(t+1)}=w_{r,s}^{(t)}-\eta\cfrac{\partial \mathcal{L}(\theta^{(t)})}{\partial w_{r,s}^{(t)}}, \hspace{0.2cm}\forall r\in[m/k], s\in\{1,2\}$}
\end{algorithm}
\begin{algorithm}\caption{SGD for joint-training pMoE}\label{alg:2}
\textbf{Input} : Training data $\{(x_i,y_i)\}_{i=1}^N$, learning rate $\eta$, number of iteration $T$, batch-size $B$\\
\textbf{Step-1}: Initialize $w_s^{(0)}, w_{r,s}^{(0)}, a_{r,s}, \forall s\in[k], r\in[m/k]$ according to (\ref{eqn:router_ini}) and (\ref{eqn:expert_ini})\\
\textbf{Step-2}: for $t=0,1, ... ,T-1$ do:\\
\centerline{$w_{s}^{(t+1)}=w_{s}^{(t)}-\eta\cfrac{\partial \mathcal{L}(\theta^{(t)})}{\partial w_{s}^{(t)}}, \forall s\in[k]$}\\
\centerline{$w_{r,s}^{(t+1)}=w_{r,s}^{(t)}-\eta\cfrac{\partial \mathcal{L}(\theta^{(t)})}{\partial w_{r,s}^{(t)}}, \hspace{0.2cm}\forall r\in[m/k], s\in[k]$}
\end{algorithm}

\section{Proof Sketch}
The proof of generalization guarantee for pMoE (i.e., Theorem \ref{Thm_mcnn_bin} and \ref{Thm_mcnn_j}) can be outlined as follows (the proof for single CNN follows a simpler version of the outline provided below):

\textbf{Step 1.} (\textbf{Feature learning in the router}) For separate-training pMoE, we first show that the batch-gradient of the router loss (i.e., $\ell_r(w_1^{(t)},w_2^{(t)})$) w.r.t. the gating kernels (i.e., $w_1^{(t)}$ and $w_2^{(t)}$) has large component (of size $\frac{1-\delta_d}{2}-\Omega\left(\frac{n}{\sqrt{B_r}}\right)$) along the class-discriminative pattern $o_1$ and $o_2$ respectively. Then, by selecting $B_r=\Omega\left(\frac{n^2}{(1-\delta_d)^2}\right)$ (which provides us $\Omega(1)$ loss reduction per step) and training for $\Omega\left(\frac{1}{1-\delta_d}\right)$ iterations, we can show that $w_1$ and $w_2$ is sufficiently aligned with $o_1$ and $o_2$ respectively to guarantee the selection of these class-discriminative patterns in TOP-$l$ patches when $l\ge l^*$ (see Lemma \ref{a_corollary_1} for exact statement).

\textbf{Step 2.} (\textbf{Coupling the experts to pseudo experts}) When the experts of pMoE are sufficiently overparameterized, w.h.p. the experts can be coupled to a smooth pseudo network\footnote{The pseudo network is defined as the network which activation pattern does not change from the initialization i.e., the sign of the pre-activation output of hidden nodes does not change from the sign at initialization; see \citep{li2018learning} for details.} of experts as for every sample drawn from the distribution $\mathcal{D}$ and every $\tau>0$, the activation pattern for $1-\Omega\left(\frac{\tau l}{\sigma}\right)$ (for separate-training pMoE) or $1-\Omega\left(\frac{\tau n}{\sigma}\right)$ (for joint-training pMoE) fraction of hidden nodes in each expert does not change from the initialization for $O(\frac{\tau}{\eta})$ iterations (see Lemma \ref{a_lemma_4} or \ref{a_lemma_8} for exact statement). This indicates that with $\tau=O\left(\frac{\sigma}{l}\right)$ (for separate-training pMoE) or $\tau=O\left(\frac{\sigma}{n}\right)$ (for joint-training pMoE), $\eta=\Omega\left(\frac{1}{ml}\right)$ (for separate-training pMoE) or $\eta=\Omega\left(\frac{1}{mn}\right)$ (for joint-training pMoE) and $\sigma=O\left(\frac{1}{\sqrt{m}}\right)$ we can couple $\Omega(1)$ fraction of hidden nodes of each expert to the corresponding pseudo experts for $O(\sqrt{m})$ iterations.

\textbf{Step 3.}(\textbf{Large error implies large gradient}) We can now analyze the pseudo network of experts corresponding to the separate-training pMoE to show that, at any iteration $t$,  the magnitude of the expected gradient for any expert $s\in\{1,2\}$ of the pseudo network is $\Omega\left(\frac{v_s^{(t)}}{l}\right)$ where $v_s^{(t)}$ characterizes the class-conditional expected error over samples with $y=+1$ and $y=-1$ for $s=1$ and $s=2$, respectively (see Lemma \ref{a_lemma_6} for exact statement). Similarly, for joint-training pMoE we show that the magnitude of the expected gradient is $\Omega\left(\frac{v_s^{(t)}}{l}\right)$, but this time $v_s^{(t)}$ characterizes the maximum of the class-conditional expected-errors over the samples for which the expert ``$s$'' receiving class-discriminative patterns from the router (see Lemma \ref{a_lemma_10} for exact statement). 

\textbf{Step 4.} (\textbf{Convergence}) Now let us define $v^{(t)}=\sqrt{\sum_{s\in[k]}v_s^2(t)}$. For separate-training pMoE, by selecting the batch size $B_t=\Omega(\frac{l^4}{(v^{(t)})^4})$ at iteration $t$, $\eta=\Omega(\frac{(v^{(t)})^2}{ml^2})$ and $\tau=O(\frac{\sigma (v^{(t)})^2}{l^3})$, we can couple the empirical batch gradient of each expert of the true network for that batch to the expected gradient of the corresponding expert of the pseudo network. Because the pseudo network is smooth, we can show that SGD minimizes the expected loss of the true network by $\Omega(\frac{\eta m (v^{(t)})^2}{l^2})$ at each iteration for $t=O(\frac{\sigma (v^{(t)})^2}{\eta l^3})$ iterations (see Lemma \ref{a_lemma_7} for the  exact statement). Similarly, for joint-training pMoE, by selecting $B_t=\Omega(\frac{k^2}{(v^{(t)})^4})$ and $\eta=\Omega(\frac{(v^{(t)})^2l^3}{mk^2})$ we can show that SGD minimizes the expected loss of the true network by $\Omega(\frac{\eta m (v^{(t)})^2}{l^2})$ for $t=O(\frac{\sigma (v^{(t)})^2l^2}{\eta nk})$ (see Lemma \ref{a_lemma_11} for exact statement). As the loss of the true network is $O(1)$ at initialization, eventually the network will converge.

\textbf{Step 5.} (\textbf{Generalization}) We show that to ensure at most $\epsilon$ generalization error after any iteration $t$, we need $\max\{v_1^{(t)},v_2^{(t)}\}<\epsilon^2$ where $v_1^{(t)}$ and $v_2^{(t)}$ correspond  to the class-conditional expected error of the class with $y=+1$ and $y=-1$, respectively. Now as we show that the router in the separate-training pMoE dispatch class-discriminative patches of all the samples labeled as $y=+1$ to the expert indexed by $s=1$ and class-discriminative patches of all the samples labeled as $y=-1$ to the expert indexed by $s=2$ from the beginning of expert-training, $v^{(t)}<\epsilon^2$ ensures $\max\{v_1^{(t)},v_2^{(t)}\}<\epsilon^2$. On the other hand, for the joint-training pMoE,  as we assume that the router ensures the dispatchment of all the class-discriminative patches of a class to a particular expert before the convergence of the model and the gating value of the patch is the largest among all the patches sent to that particular expert, $v^{(t)}<\frac{\epsilon^2}{l}$ implies $\max\{v_1^{(t)},v_2^{(t)}\}<\epsilon^2$. Hence for separate-training pMoE, by setting $v^{(t)}\ge\epsilon^2$ we show that with $B=\Omega(l^4/\epsilon^8)$ and $\eta=\Omega(1/m\text{poly}(l,1/\epsilon))$ for $T=O(l^4/\epsilon^8)$ iterations, we can guarantee that the generalization error is less than $\epsilon$ (see Theorem \ref{a_theorem_1} for exact statement). Similarly, for joint-training pMoE, by setting $v^{(t)}\ge\frac{\epsilon^2}{l}$ and setting $B=\Omega(k^2l^4/\epsilon^8)$ and $\eta=\Omega(1/(m\text{poly}(l,1/\epsilon)$ for $T=O(k^2l^2/\epsilon^8)$ iterations, we can guarantee that the generalization error is less than $\epsilon$ (see Theorem \ref{a_theorem_2} for exact statement).

\section{Proof of the Lemma \ref{router_lemma}}\label{router_lemma_proof}
\begin{definition}{($\delta^\prime$-closer class-irrelevant patterns)}
    For any $\delta^\prime>0$, a class-irrelevant pattern $q$ is $\delta^\prime$-closer to $o_1$ than $o_2$, if $\langle o_1, q\rangle-\langle o_2,q\rangle>\delta^\prime$ for any $\delta^\prime>0$. Similarly, a class-irrelevant pattern $q$ is $\delta^\prime$-closer to $o_2$ than $o_1$ if $\langle o_2,q\rangle-\langle o_1,q\rangle>\delta^\prime$.
\end{definition}
\begin{definition}{(Set of $\delta^\prime$-closer class-irrelevant patterns, $\mathcal{S}_c(\delta^\prime)$)}
    For any $\delta^\prime>0$, define the set of $\delta^\prime$-closer class-irrelevant patterns, denoted as $\mathcal{S}_c(\delta^\prime)\subset\bigcup_{i=1}^pS_j$ such that: $\forall q\in \mathcal{S}_c(\delta^\prime), |\langle o_1-o_2,q\rangle|>\delta^\prime$.
\end{definition}
\begin{definition}\label{thershold_l}{(Threshold, $l^*$)}
     Define the threshold $l^*$ such that:\\\indent\hspace{2.3cm} $\forall (x,y)\sim\mathcal{D}, \left|\{j\in[n]: x^{(j)}\neq o_1 \text{ and } x^{(j)}\in S_c\left(\frac{1-\delta_d}{2}\right)\}\right|\le l^*-1$
\end{definition}
\begin{lemma}\label{a_corollary_1}{(Full version of \textbf{Lemma \ref{router_lemma}})}
    For every $l\geq l^*$, w.h.p. over the random initialization defined in (\ref{eqn:router_ini}), after completing the Step-2 of Algorithm-1 with batch-size $B_r=\Tilde{\Omega}\left(\cfrac{n^2}{(1-\delta_d)^2}\right)$ and learning rate $\eta_r=\Theta\left(\frac{1}{n}\right)$ for $T_r=\Omega\left(\cfrac{1}{1-\delta_d}\right)$ iterations, the returned $w_1^{(T_r)}$ and $w_2^{(T_r)}$ satisfy 
    \begin{equation*}
        \underset{j\in[n]}{\text{arg}}(x^{(j)}=o_1) \in J_1(w_1^{(T_r)}, x), \quad \forall  (x,y=+1) \sim \mathcal{D}
    \end{equation*}
    \begin{equation*}
        \underset{j\in[n]}{\text{arg}}(x^{(j)}=o_2) \in J_2(w_2^{(T_r)}, x), \quad \forall  (x,y=-1) \sim \mathcal{D}
    \end{equation*}
\end{lemma}
\begin{proof}
    The proof follows directly from the Definition \ref{thershold_l} and the Lemma \ref{a_lemma_3}.
\end{proof}

\section{Lemmas Used to Prove the Lemma \ref{router_lemma}}\label{aux_router_lemma_proof}
We denote,\\
$\nabla_{w_s^{(t)}}\mathbb{E}[\ell_r(w_1,w_2)]:=\mathbb{E}_{\mathcal{D}}\left[\cfrac{\partial \ell_r(w^{(t)}_1,w_2^{(t)})}{\partial w_s^{(t)}}\right]$ where $w_s^{(t)}\in\{w_1^{(t)},w_2^{(t)}\}$ for all $t\in[T_r]$.
\begin{lemma}\label{lemma_a_1}
At any iteration $t\le T_r$ of the Step-2 of Algorithm \ref{alg:1},
\begin{center}
    $\nabla_{w_1^{(t)}}\mathbb{E}[l(f(x),y)]=-\cfrac{1}{2}\left(o_1-o_2\right)$, and $\nabla_{w_2^{(t)}}\mathbb{E}[l(f(x),y)]=-\cfrac{1}{2}\left(o_2-o_1\right)$
\end{center}
\end{lemma}
\begin{proof}
    As, $\ell_r(w_1^{(t)},w_2^{(t)})=-\cfrac{1}{B_r}\hspace{0.1cm}\sum_{(x,y)\in\mathcal{B}_t^r} y  \langle w^{(t)}_1-w_2^{(t)}, \sum_{j=1}^n x^{(j)}\rangle$,
    
    $\nabla_{w_1^{(t)}}\mathbb{E}[l_r(w_1,w_2)]=-\mathbb{E}_{\mathcal{D}}\left[y\sum_{j=1}^nx^{(j)}\right]$ and $\nabla_{w_2^{(t)}}\mathbb{E}[l_r(w_1,w_2)]=\mathbb{E}_{\mathcal{D}}\left[y\sum_{j=1}^nx^{(j)}\right]$

    Therefore,
    \begin{align*}
        &\nabla_{w_1^{(t)}}\mathbb{E}[l_r(w_1,w_2)]=-\frac{1}{2}\mathbb{E}_{\mathcal{D}|y=+1}\left[\sum_{j=1}^nx^{(j)}|y=+1\right]+\frac{1}{2}\mathbb{E}_{\mathcal{D}|y=-1}\left[\sum_{j=1}^nx^{(j)}|y=-1\right]\\
        &=-\frac{1}{2}\mathbb{E}_{\mathcal{D}|y=+1}\left[\sum_{j\in[n]/\underset{j}{\text{arg }}x^{(j)}=o_1}x^{(j)}|y=+1\right]+\frac{1}{2}\mathbb{E}_{\mathcal{D}|y=-1}\left[\sum_{j\in[n]/\underset{j}{\text{arg }}x^{(j)}=o_2}x^{(j)}|y=-1\right]\\
        &\hspace{0.35cm}-\frac{1}{2}\left(o_1-o_2\right)\\
        &=-\frac{1}{2}\left(o_1-o_2\right)
    \end{align*}
    where the last equality comes from the fact that class-irrelevant patterns are distributed identically in both classes. Using similar line of arguments we can show that, $\nabla_{w_2^{(t)}}\mathbb{E}[l(f(x),y)]=-\cfrac{1}{2}\left(o_2-o_1\right)$.
\end{proof}
\begin{lemma}\label{a_lemma_2}
    With probability $1-\frac{1}{poly(n)}$ (i.e., w.h.p.) over the random initialization of the gating kernels defined in (\ref{eqn:router_ini}), $\left|\left|w_s^{(0)}\right|\right|\le\frac{1}{n^2}$; $\forall s\in\{1,2\}$
\end{lemma}
\begin{proof}
    Let us denote the $i$-th element of the vector $w_s^{(0)}$ as $w_{s_i}^{(0)}$ where $i\in[d]$.\\
    Then according to the random initialization of $w_s^{(0)}$ and using a Gaussian tail-bound (i.e., for $X\sim\mathcal{N}(0,\sigma^2):Pr[|X|\ge t]\le 2e^{-t^2/2\sigma^2}$): $\mathbb{P}\left[\left|w_{s_i}^{(0)}\right|\ge\frac{1}{n^2\sqrt{d}}\right]\le\frac{1}{poly(n)}$.\\
    Let us denote the event $\mathcal{E}:\forall i\in[d], \left|w_{s_i}^{(0)}\right|\le\frac{1}{n^2\sqrt{d}}$.
    Therefore, $\mathbb{P}\left[\mathcal{E}\right]\ge1-\frac{1}{poly(n)}$.\\
    Now, conditioned on the event $\mathcal{E}, \left|\left|w_s^{(0)}\right|\right|\le\frac{1}{n^2}$.\\
    Therefore, $\mathbb{P}\left[\left|\left|w_s^{(0)}\right|\right|\le\frac{1}{n^2}\right]\le\mathbb{P}\left[\left|\left|w_s^{(0)}\right|\right|\ge\frac{1}{n^2}|\mathcal{E}\right]\mathbb{P}\left[\mathcal{E}\right]=1-\frac{1}{poly(n)}$
\end{proof}

\begin{lemma}\label{a_lemma_3}
     W.h.p. over the random initialization of the gating-kernels defined in (\ref{eqn:router_ini}) and randomly selected batch of batch-size $B_r=\Tilde{\Omega}\left(\cfrac{n^2}{(1-\delta_d)^2}\right)$ at each iteration, after $T_r=\Omega\left(\cfrac{1}{1-\delta_d}\right)$ iterations of Step-2 of Algorithm \ref{alg:1} with learning rate $\eta_r=\Theta\left(\frac{1}{n}\right)$, $\forall (x,y)\sim\mathcal{D}, j\in[n]:x^{(j)}\not\in\mathcal{S}_c(\frac{1-\delta_d}{2})$, $\langle w_1^{(T_r)}, o_1\rangle>\langle w_1^{(T_r)},x^{(j)}\rangle$ and $\langle w_2^{(T_r)}, o_2\rangle>\langle w_2^{(T_r)},x^{(j)}\rangle$.
\end{lemma}
\begin{proof}
    Let, at $t$-th iteration of Step-2 of Algorithm \ref{alg:1}, $\Tilde{\nabla}_{w_s}^{(t)}=\cfrac{\partial \ell_r(w_1^{(t)},w_2^{(t)})}{\partial w_s^{(t)}}$ for all $s\in\{1,2\}$\\\\
Also let us denote, $\nabla_{w_s^{(t)}}\mathbb{E}\left[\ell_r(w_1^{(t)},w_2^{(t)})\right]=\nabla_{w_s}^{(t)}$ for all $s\in\{1,2\}$ \\\\
Therefore, after $T_r$-th iteration of SGD and using Lemma \ref{lemma_a_1},
\begin{align*}
    w_1^{(T_r)}&=w_1^{(0)}-\eta_r\overset{T_r-1}{\underset{t=0}{\sum}}\Tilde{\nabla}_{w_1}^{(t)}\\
    &=w_1^{(0)}+\cfrac{\eta_r T_r}{2}\left(o_1-o_2\right)-\eta_r\overset{T_r-1}{\underset{t=0}{\sum}}\left(\Tilde{\nabla}_{w_1}^{(t)}-\nabla_{w_1}^{(t)}\right)
\end{align*}
Similarly, $w_2^{(T_r)}=w_2^{(0)}+\cfrac{\eta_r T_r}{2}\left(o_2-o_1\right)-\eta_r\overset{T_r-1}{\underset{t=0}{\sum}}\left(\Tilde{\nabla}_{w_2}^{(t)}-\nabla_{w_2}^{(t)}\right)$.\\\\
Now, $||\Tilde{\nabla}_{w_s}^{(t)}||=O(n)$. Hence, w.h.p. over a randomly sampled batch of size $B_r$, using Hoeffding's concentration,\\\\
$||\Tilde{\nabla}_{w_s}^{(t)}-\nabla_{w_s}^{(t)}||=\Tilde{O}\left(\cfrac{n}{\sqrt{B_r}}\right); \forall s\in\{1,2\}$.\\\\
Now,
\begin{align*}
    \langle w_1^{(T_r)},o_1\rangle&=\langle w_1^{(0)}, o_1\rangle+\cfrac{\eta_r T_r}{2}\left(1-\langle o_1,o_2\rangle\right)-\eta_r\overset{T_r-1}{\underset{t=0}{\sum}}\langle \Tilde{\nabla}_{w_1}^{(t)}-\nabla_{w_1}^{(t)},o_1\rangle\\
    &\ge \cfrac{\eta_r T_r}{2}\left(1-\delta_d\right)-\eta_r T_r\Tilde{O}\left(\cfrac{n}{\sqrt{B_r}}\right)-\left|\left|w_1^{(0)}\right|\right|
\end{align*}
On the other hand, $\forall (x,y)\sim\mathcal{D}, \forall j\in[n]: x^{(j)}\not\in\mathcal{S}_c\left(\frac{1-\delta_d}{2}\right)$,
\begin{align*}
    \langle w_1^{(T_r)}, x^{(j)}\rangle&=\langle w_1^{(0)},x^{(j)}\rangle+\cfrac{\eta_r T_r}{2}\left(\langle o_1,x^{(j)}\rangle-\langle o_2,x^{(j)}\rangle\right)-\eta_r\overset{T_r-1}{\underset{t=0}{\sum}}\langle \Tilde{\nabla}_{w_1}^{(t)}-\nabla_{w_1},x^{(j)}\rangle\\
    &\le\cfrac{\eta_r T_r}{4}(1-\delta_d)+\eta_r T_r\Tilde{O}\left(\cfrac{n}{\sqrt{B_r}}\right)+\left|\left|w_1^{(0)}\right|\right|
\end{align*}
From Lemma \ref{a_lemma_2}, w.h.p. over the random initialization: $\left|\left|w_1^{(0)}\right|\right|\le\cfrac{1}{n^2}$.\\\\
Therefore, selecting $B_r=\Tilde{\Omega}\left(\cfrac{n^2}{(1-\delta_d)^2}\right)$ and $\eta_r=\Theta\left(\frac{1}{n}\right)$, we need $T_r=\Omega\left(\cfrac{1}{1-\delta_d}\right)$ iterations to achieve $\langle w_1^{(T_r)}, o_1\rangle>\langle w_1^{(T_r)},x^{(j)}\rangle$, $\forall j\in[n]:x^{(j)}\in\mathcal{S}_c\left(\frac{1-\delta_d}{2}\right)$\\\\
Similar line of arguments can be made to show with batch size $B_r=\Tilde{\Omega}\left(\cfrac{n^2}{(1-\delta_d)^2}\right)$ and learning rate $\eta_r=\Theta\left(\frac{1}{n}\right)$, after $T_r=\Omega\left(\cfrac{1}{1-\delta_d}\right)$ iterations, $\langle w_2^{(T_r)}, o_2\rangle\ge\langle w_2^{(T_r)},x^{(j)}\rangle$, $\forall j\in[n]:x^{(j)}\in\mathcal{S}_c\left(\frac{1-\delta_d}{2}\right)$.\\\\
\end{proof}

\section{Proofs of the Theorem \ref{Thm_mcnn_bin}, \ref{Thm_single_cnn} and \ref{Thm_mcnn_j}}\label{theorems_proof}
\begin{definition}\label{value_function} 
    At any iteration $t$ of the minibatch SGD,
    \begin{enumerate}
        \item Define the \textbf{value function}, $\displaystyle v^{(t)}(\theta^{(t)},x,y):=\frac{1}{1+e^{yf_M(\theta^{(t)},x)}}$. It is easy to show that for any $(x,y)\sim\mathcal{D}$, $0\le v^{(t)}(\theta^{(t)},x,y)\le1$. The function captures the prediction error, i.e., a larger $v^{(t)}$ indicates a larger prediction error.
        \item Define, the \textbf{class-conditional expected value function}, $v_1^{(t)}:=\mathbb{E}_{\mathcal{D}|y=+1}[v^{(t)}(\theta^{(t)},x,y)|y=+1]$ and $v_2^{(t)}:=\mathbb{E}_{\mathcal{D}|y=-1}[v^{(t)}(\theta^{(t)},x,y)|y=-1]$.
        Here, $v_1^{(t)}$ captures the expected error for the class with label $y=+1$ and $v_2^{(t)}$ captures the expected error for the class with label $y=-1$.  
    \end{enumerate}
\end{definition}
\begin{definition}\label{loss_reduction}
    At any iteration $t$ of the minibatch SGD,
    \begin{enumerate}
        \item For any sample $(x,y)\sim\mathcal{D}$, we define the \textbf{reduction of loss} at the $t$-th iteration of SGD as,
        \begin{align*}
            \Delta L(\theta^{(t)},\theta^{(t+1)},x,y):=\mathcal{L}(\theta^{(t)},x,y)-\mathcal{L}(\theta^{(t+1)},x,y)
        \end{align*}
        where, $\mathcal{L}(\theta^{(t)},x,y):=\log(1+e^{-yf_M(\theta^{(t)},x)})$ is the \textbf{single-sample loss function}.
        \item Define the \textbf{expected reduction of loss} at the $t$-th iteration of SGD as,
        \begin{align*}
            \Delta L(\theta^{(t)},\theta^{(t+1)}):=\mathbb{E}_\mathcal{D}\left[\mathcal{L}(\theta^{(t)},x,y)-\mathcal{L}(\theta^{(t+1)},x,y)\right]
        \end{align*}
    \end{enumerate}
\end{definition}
\begin{theorem}\label{a_theorem_1}{(Full version of \textbf{Theorem \ref{Thm_mcnn_bin}})}
    For every $\epsilon>0$ and $l\ge l^*$, for every $m\ge M_S=\Tilde{\Omega}\left(l^{10}p^{12}\delta^6\big/\epsilon^{16}\right)$ with at least   $N_S=\Tilde{\Omega}(l^8 p^{12}\delta^6/\epsilon^{16})$ training samples, after performing minibatch SGD with the batch size $B=\Tilde{\Omega}\left(l^4p^6\delta^3\big/\epsilon^{8}\right)$ and the learning rate $\eta=\Tilde{O}\big(1\big/m \textrm{poly}(l,p,\delta,1/\epsilon, \log m)\big)$ for $T=\Tilde{O}\left(l^4p^6\delta^3\big/\epsilon^{8}\right)$ iterations, it holds w.h.p. that
    \begin{center}
        $\underset{(x,y)\sim\mathcal{D}}{\mathbb{P}}\left[yf_{M}(\theta^{(T)},x)>0\right]\ge1-\epsilon$
    \end{center}
\end{theorem}
\begin{proof}
    First we will show that for any $\epsilon<\frac{1}{2}$, if $\mathbb{P}_{(x,y)\sim \mathcal{D}}\left[yf(\theta^{(t)}, x)>0\right]\le1-\epsilon$, then $\max\{v_1^{(t)},v_2^{(t)}\}\ge\epsilon^2$.\\\\
    Now for any $(x,y)\sim\mathcal{D}$ and $\epsilon<\frac{1}{2}$, if $v^{(t)}(\theta^{(t)},x,y)\le\epsilon$, $yf_M(\theta^{(t)},x,y)>0$ i.e., the prediction is correct.\\\\
    Now if $v_1^{(t)}=\mathbb{E}_{\mathcal{D}|y=+1}\left[v^{(t)}(\theta^{(t)},x,y)\big|y=+1\right]\le\epsilon^2$, then using Markov's inequality $\mathbb{P}_{\mathcal{D}|y=+1}\left[v^{(t)}(\theta^{(t)},x,y)\le\epsilon\right]\ge1-\epsilon$ which implies for any $\epsilon<\frac{1}{2}$, $\mathbb{P}_{\mathcal{D}|y=+1}\left[yf(\theta^{(t)}, x)>0\right]\ge1-\epsilon$.\\\\
    Similarly, if $v_2^{(t)}=\mathbb{E}_{\mathcal{D}|y=-1}\left[v^{(t)}(\theta^{(t)},x,y)\big|y=-1\right]\le\epsilon^2$, for any $\epsilon<\frac{1}{2}$, $\mathbb{P}_{\mathcal{D}|y=-1}\left[yf(\theta^{(t)}, x)>0\right]\ge1-\epsilon$.\\\\
    Therefore, for any $\epsilon<\frac{1}{2}$, if $\mathbb{P}_{(x,y)\sim \mathcal{D}}\left[yf(\theta^{(t)}, x)>0\right]\le1-\epsilon$, then $\max\{v_1^{(t)},v_2^{(t)}\}\ge\epsilon^2$.\\\\
    Now, if $v^{(t)}:=\sqrt{\underset{s\in\{1,2\}}{\sum}(v_s^{(t)})^2}\le\epsilon^2$ then $\max\{v_1^{(t)},v_2^{(t)}\}\le\epsilon^2$, which implies after a proper number of iterations if $v^{(t)}\le\epsilon^2$ then $\underset{(x,y)\sim\mathcal{D}}{\mathbb{P}}\left[yf_{M}(\theta^{(T)},x)>0\right]\ge1-\epsilon$.\\\\
    Let, $v^{(t)}\ge\epsilon^2$. Then by using Lemma \ref{a_lemma_7} for every $l\ge l^*$, with $\eta=\Tilde{O}\left(\cfrac{\epsilon^4}{ml^2p^3\delta^{3/2}}\right)$ and $B=\Tilde{\Omega}\left(\cfrac{l^4p^6\delta^3}{\epsilon^8}\right)$, at least for $t=\Tilde{O}\left(\cfrac{\sigma\epsilon^4}{\eta l^3p^3\delta^{3/2}}\right)$ we have,
    \begin{equation}\label{a_loss_reduction_eqn}
        \Delta L(\theta^{(t)},\theta^{(t+1)})=\Tilde{\Omega}\left(\cfrac{\eta m\epsilon^4}{l^2p^3\delta^{3/2}}\right)
    \end{equation}
    Now, as $w_{r,s}^{(0)}\sim\mathcal{N}(0,\sigma^2)$ with $\sigma=\frac{1}{\sqrt{m}}$, $\left\langle w_{r,s}^{(0)},x^{(j)}\right\rangle\sim\mathcal{N}(0,\sigma^2)$  $\forall j\in J_s(w_s^{(0)},x)$ and $\forall (x,y)\sim\mathcal{D}$. Therefore, w.h.p. $\left|f_M(\theta^{(0)},x)\right|=\Tilde{O}(1)$ which implies $\mathcal{L}(\theta^{(0)},x,y)=\Tilde{O}(1)$. Now as $\mathcal{L}(\theta^{(t)},x,y)>0$, (\ref{a_loss_reduction_eqn}) can happen at most $\Tilde{O}\left(\frac{l^2p^3\delta^{3/2}}{\eta m\epsilon^4}\right)$ iterations. Now as $\eta m=\Tilde{O}\left(\cfrac{\epsilon^4}{l^2p^3\delta^{3/2}}\right)$, we need $T=\Tilde{O}\left(\cfrac{l^4p^6\delta^3}{\epsilon^8}\right)$ iterations to ensure that $v^{(t)}\le\epsilon^2$.\\\\
    On the other hand, to ensure (\ref{a_loss_reduction_eqn}) hold for $T$ iterations, we need,
    \begin{align*}
        \cfrac{\sigma\epsilon^4}{\eta l^3p^3\delta^{3/2}}=\Tilde{\Omega}\left(\cfrac{l^2p^3\delta^{3/2}}{\eta m\epsilon^4}\right)
    \end{align*}
    which implies we need $m=\Tilde{\Omega}\left(\cfrac{l^{10}p^{12}\delta^6}{\epsilon^{16}}\right)$.
\end{proof}

Now, for any $(x,y=+1)\sim\mathcal{D}$ and $(x,y=-1)\sim\mathcal{D}$, let us denote the index of the class-discriminative patterns i.e., $o_1$ and $o_2$ as $j_{o_1}$ and $j_{o_2}$, respectively.
\begin{definition}\label{joint_training_value_function}
    At any iteration $t$ of \textbf{minibatch SGD of the joint-training pMoE} (i.e., Step-2 of Algorithm \ref{alg:2}),
    \begin{enumerate}
        \item For any $(x,y=+1)\sim\mathcal{D}$ and the expert $s\in[k]$, define the \textbf{event that $o_1$ in Top-$l$} as, $\mathcal{E}_{1,s}^{(t)}: j_{o_1}\in J_s(w_s^{(t)},x)$. Similarly, for any $(x,y=-1)\sim\mathcal{D}$ define the \textbf{event that $o_2$ in Top-$l$} as, $\mathcal{E}_{2,s}^{(t)}: j_{o_2}\in J_s(w_s^{(t)},x)$.
        \item For any expert $s\in[k]$, define the \textbf{probability of the event that $o_1$ in Top-$l$} as, $p_{1,s}^{(t)}:=\mathbb{P}_{\mathcal{D}|y=+1}\left[\mathcal{E}_{1,s}^{(t)}\big|y=+1\right]$ and the \textbf{probability of the event that $o_2$ in Top-$l$} as, $p_{2,s}^{(t)}:=\mathbb{P}_{\mathcal{D}|y=-1}\left[\mathcal{E}_{2,s}^{(t)}\big|y=-1\right]$
        \item For any expert $s\in[k]$ define, $v_{1,s}^{(t)}:=\mathbb{E}_{\mathcal{D}|y=+1,\mathcal{E}_{1,s}^{(t)}}\left[p_{1,s}^{(t)}G_{j_{o_1},s}^{(t)}(x)v^{(t)}(\theta^{(t)},x,y)\big|y=+1,\mathcal{E}_{1,s}^{(t)}\right]$ and $v_{2,s}^{(t)}:=\mathbb{E}_{\mathcal{D}|y=-1,\mathcal{E}_{2,s}^{(t)}}\left[p_{2,s}^{(t)}G_{j_{o_2},s}^{(t)}(x)v^{(t)}(\theta^{(t)},x,y)\big|y=-1,\mathcal{E}_{2,s}^{(t)}\right]$ where $G_{j_{o_1},s}^{(t)}(x)$ and $G_{j_{o_2},s}^{(t)}(x)$ denote the gating value for the class-discriminative patterns $o_1$ and $o_2$ conditioned on $\mathcal{E}_{1,s}^{(t)}$ and $\mathcal{E}_{2,s}^{(t)}$, respectively.
    \end{enumerate}
\end{definition}
\begin{theorem}\label{a_theorem_2}{(Full version of the \textbf{Theorem \ref{Thm_mcnn_j}})}
    Suppose Assumption \ref{router_asmptn} hold. Then for every $\epsilon>0$, for every  $m \ge M_J=\Tilde{\Omega}\left(k^3n^2l^{6}p^{12}\delta^6\big/\epsilon^{16}\right)$ with at least $N_J=\Tilde{\Omega}(k^4l^{6} p^{12}\delta^6/\epsilon^{16})$ training samples,  after performing minibatch SGD with the batch size $B=\Tilde{\Omega}\left(k^2l^4p^6\delta^3\big/\epsilon^{8}\right)$ and the learning rate $\eta=\Tilde{O}\big(1\big/m poly(l,p,\delta,1/\epsilon, \log m)\big)$ for $T=\Tilde{O}\left(k^2l^2p^6\delta^3\big/\epsilon^{8}\right)$ iterations, it holds w.h.p. that
    \begin{center}
       $\underset{(x,y)\sim\mathcal{D}}{\mathbb{P}}\left[yf_{M}(\theta^{(T)}, x)>0\right]\ge1-\epsilon$
    \end{center}
\end{theorem}
\begin{proof}
    From the argument of the proof of Theorem \ref{a_theorem_1}, we know that for any $\epsilon<\frac{1}{2}$, if $\mathbb{P}_{(x,y)\sim \mathcal{D}}\left[yf(\theta^{(t)}, x)>0\right]\le1-\epsilon$, then $\max\{v_1^{(t)},v_2^{(t)}\}\ge\epsilon^2$ where $v_1^{(t)}:=\mathbb{E}_{\mathcal{D}|y=+1}[v^{(t)}(\theta^{(t)},x,y)|y=+1]$ and $v_2^{(t)}:=\mathbb{E}_{\mathcal{D}|y=-1}[v^{(t)}(\theta^{(t)},x,y)|y=-1]$\\\\
    Now, we will consider the case when $t\ge T^\prime$ where $T^\prime$ is defined in Assumption \ref{router_asmptn}.\\\\
    Now, if the expert $s_1\in[k]$ satisfies Assumption \ref{router_asmptn} for $y=+1$, then $p_{1,s_1}^{(t)}=1$ and $G_{j_{o_1},s_1}^{(t)}(x)\ge\cfrac{1}{l}$ for any $(x,y=+1)\sim\mathcal{D}$. Therefore, $v_{1,s_1}^{(t)}\ge\cfrac{v_1^{(t)}}{l}$.\\\\ 
    Similarly, if the expert $s_2\in[k]$ satisfies Assumption \ref{router_asmptn} for $y=-1$, then $v_{2,s_2}^{(t)}\ge\cfrac{v_2^{(t)}}{l}$.\\\\
    Now for any expert $s\in[k]$, let us define $v_s^{(t)}:=\max\{v_{1,s}^{(t)},v_{2,s}^{(t)}\}$\\\\
    Now, if $v^{(t)}:=\sqrt{\underset{s\in[k]}{\sum}(v_s^{(t)})^2}\le\cfrac{\epsilon^2}{l}$, then $v_{s_1}^{(t)}\le\cfrac{\epsilon^2}{l}$ and $v_{s_2}^{(t)}\le\cfrac{\epsilon^2}{l}$.\\\\
    This implies, $\max\{v_{1,s_1}^{(t)},v_{2,s_1}^{(t)}\}\le\cfrac{\epsilon^2}{l}$ and $\max\{v_{1,s_2}^{(t)},v_{2,s_2}^{(t)}\}\le\cfrac{\epsilon^2}{l}$.\\\\
    Therefore, $v_{1,s_1}^{(t)}\le\cfrac{\epsilon^2}{l}$ and $v_{2,s_2}^{(t)}\le\cfrac{\epsilon^2}{l}$ which implies $v_1^{(t)}\le\epsilon^2$ and $v_2^{(t)}\le\epsilon^2$.\\\\
    In that case, $\max\{v_1^{(t)},v_2^{(t)}\}\le\epsilon^2$.\\\\
    Therefore, by taking $v^{(t)}\ge\cfrac{\epsilon^2}{l}$, using the results of Lemma \ref{a_lemma_11} and following same procedure as in Theorem \ref{a_theorem_1} we can complete the proof.
\end{proof}
\begin{theorem}\label{a_theorem_3}{(Full version of the \textbf{Theorem \ref{Thm_single_cnn}})}
    For every $\epsilon>0$, for every $m\ge  M_C=\Tilde{\Omega}\left(n^{10}p^{12}\delta^6\big/\epsilon^{16}\right)$ with at least  $N_C=\Tilde{\Omega}(n^8p^{12}\delta^6/\epsilon^{16})$ training samples,
    after performing minibatch SGD with  the batch size $B=\Tilde{\Omega}\left(n^4p^6\delta^3\big/\epsilon^{8}\right)$ and the learning rate $\eta=\Tilde{O}\big(1\big/m \textrm{poly}(n,p,\delta,1/\epsilon, \log m)\big)$ for $T=\Tilde{O}\left(n^4p^6\delta^3\big/\epsilon^{8}\right)$ iterations, it holds w.h.p. that 
    \begin{center}
        $\underset{(x,y)\sim\mathcal{D}}{\mathbb{P}}\left[yf_{C}(\theta^{(T)},x)>0\right]\ge1-\epsilon$
    \end{center} 
\end{theorem}
\begin{proof}
    From the argument of the proof of Theorem \ref{a_theorem_1}, we know that for any $\epsilon<\frac{1}{2}$, if $\mathbb{P}_{(x,y)\sim \mathcal{D}}\left[yf(\theta^{(t)}, x)>0\right]\le1-\epsilon$, then $v^{(t)}:=\max\{v_1^{(t)},v_2^{(t)}\}\ge\epsilon^2$ where $v_1^{(t)}:=\mathbb{E}_{\mathcal{D}|y=+1}[v^{(t)}(\theta^{(t)},x,y)|y=+1]$ and $v_2^{(t)}:=\mathbb{E}_{\mathcal{D}|y=-1}[v^{(t)}(\theta^{(t)},x,y)|y=-1]$.\\\\
    Therefore, taking $v^{(t)}\ge\epsilon^2$, using the results of Lemma \ref{a_lemma_14} and following similar procedure as in Theorem \ref{a_theorem_1} we can complete the proof.
\end{proof}

\section{Lemmas Used to Prove the Theorem \ref{Thm_mcnn_bin}}\label{aux_sep_pmoe_thm_proof}
For any iteration $t$ of the Step-3 of Algorithm \ref{alg:1}, recall the loss function for a single-sample generated by the distribution $\mathcal{D}$, $\mathcal{L}(\theta^{(t)},x,y):=\log(1+e^{-yf_M(\theta^{(t)},x)})$. The gradient of the loss for a  single sample with respect to the hidden nodes of the experts:
\begin{equation}\label{grad_sep_pmoe}
    \frac{\partial \mathcal{L}(\theta^{(t)},x,y)}{\partial w_{r,s}^{(t)}}=-ya_{r,s}v^{(t)}(\theta^{(t)},x,y)\left(\cfrac{1}{l}\underset{j\in J_s(w_s^{(t)},x)}{\sum}x^{(j)}1_{\langle w_{r,s}^{(t)},x^{(j)}\rangle\ge0}\right)
\end{equation}
We define the corresponding \textit{\textbf{pseudo-gradient}} as:
\begin{equation}\label{pseudo_grad_sep_moe}
    \frac{\overset{\sim}{\partial} \mathcal{L}(\theta^{(t)},x,y)}{\partial w_{r,s}^{(t)}}=-ya_{r,s}v^{(t)}(\theta^{(t)},x,y)\left(\cfrac{1}{l}\underset{j\in J_s(w_s^{(t)},x)}{\sum}x^{(j)}1_{\langle w_{r,s}^{(0)},x^{(j)}\rangle\ge0}\right)
\end{equation}
Therefore, the expected pseudo-gradient:
\begin{align*}
    \frac{\overset{\sim}{\partial}\Hat{\mathcal{L}}(\theta^{(t)})}{\partial w_{r,s}^{(t)}}&=\mathbb{E}_{\mathcal{D}}\left[\frac{\overset{\sim}{\partial} \mathcal{L}(\theta^{(t)},x,y)}{\partial w_{r,s}^{(t)}}\right]\\
    &=-\cfrac{a_{r,s}}{2}\left(\mathbb{E}_{\mathcal{D}|y=+1}\left[v^{(t)}(\theta^{(t)},x,y)\left(\cfrac{1}{l}\underset{j\in J_s(w_s^{(t)},x)}{\sum}x^{(j)}1_{\langle w_{r,s}^{(0)},x^{(j)}\rangle\ge0}\right)\Big\vert y=+1\right]\right.\\
    &\left.\hspace{1.5cm}-\mathbb{E}_{\mathcal{D}|y=-1}\left[v^{(t)}(\theta^{(t)},x,y)\left(\cfrac{1}{l}\underset{j\in J_s(w_s^{(t)},x)}{\sum}x^{(j)}1_{\langle w_{r,s}^{(0)},P_jx\rangle\ge0}\right)\Big\vert y=-1\right]\right)\\
    &=-\cfrac{a_{r,s}}{2}P_{r,s}^{(t)}
\end{align*}
Here,
\begin{align*}
    P_{r,s}^{(t)}&=\mathbb{E}_{\mathcal{D}|y=+1}\left[v^{(t)}(\theta^{(t)},x,y)\left(\cfrac{1}{l}\underset{j\in J_s(w_s^{(t)},x)}{\sum}x^{(j)}1_{\langle w_{r,s}^{(0)},x^{(j)}\rangle\ge0}\right)\Big\vert y=+1\right]\\
    &\hspace{1.5cm}-\mathbb{E}_{\mathcal{D}|y=-1}\left[v^{(t)}(\theta^{(t)},x,y)\left(\cfrac{1}{l}\underset{j\in J_s(w_s^{(t)},x)}{\sum}x^{(j)}1_{\langle w_{r,s}^{(0)},x^{(j)}\rangle\ge0}\right)\Big\vert y=-1\right]
\end{align*}

\begin{lemma}\label{a_lemma_4}
    W.h.p. over the random initialization of the hidden nodes of the experts defined in \ref{eqn:expert_ini}, for every $(x,y)\sim\mathcal{D}$ and for every $\tau>0$, for every $t=\Tilde{O}\left(\cfrac{\tau}{\eta}\right)$ of the Step-3 of Algorithm \ref{alg:1}, we have that for at least $\left(1-\cfrac{2e\tau l}{\sigma}\right)$ fraction of $r\in[m/2]$ of the expert $s\in\{1,2\}$:\\\\
    \centerline{$\cfrac{\partial \mathcal{L}(\theta^{(t)},x,y)}{\partial w_{r,s}^{(t)}}=\cfrac{\overset{\sim}{\partial} \mathcal{L}(\theta^{(t)},x,y)}{\partial w_{r,s}^{(t)}}$ and $|\langle w_{r,s}^{(t)}, x^{(j)}\rangle|\ge\tau, \forall j\in J_s(w_s^{(t)},x)$} \\\\
\end{lemma}
\begin{proof}
    Recall the gradient of the loss for single-sample $(x,y)\sim \mathcal{D}$ w.r.t. the hidden node $r\in[m/2]$ of the expert $s\in\{1,2\}$:
    \begin{align*}
    \frac{\partial \mathcal{L}(\theta^{(t)},x,y)}{\partial w_{r,s}^{(t)}}=-ya_{r,s}v^{(t)}(\theta^{(t)},x,y)\left(\cfrac{1}{l}\underset{j\in J_s(w_s^{(t)},x)}{\sum}x^{(j)}1_{\langle w_{r,s}^{(t)},x^{(j)}\rangle\ge0}\right)
    \end{align*}
    and the corresponding pseudo-gradient:
    \begin{align*}
    \frac{\overset{\sim}{\partial} \mathcal{L}(\theta^{(t)},x,y)}{\partial w_{r,s}^{(t)}}=-ya_{r,s}v^{(t)}(\theta^{(t)},x,y)\left(\cfrac{1}{l}\underset{j\in J_s(w_s^{(t)},x)}{\sum}x^{(j)}1_{\langle w_{r,s}^{(0)},x^{(j)}\rangle\ge0}\right)
    \end{align*}
    Now, $a_{r,s}\sim \mathcal{N}(0,1)$.
    Hence, using the concentration bound of Gaussian random variable (i.e., for $X\sim\mathcal{N}(0,\sigma^2):Pr[|X|\le t]\ge 1-2e^{-t^2/2\sigma^2}$) and as $\Tilde{O}(.)$ hides factor $\log \left(poly(m, n, p, \delta,\frac{1}{\epsilon})\right)$ we get:
    \begin{align*}
        \mathbb{P}\left[|a_{r,s}|=\Tilde{O}(1)\right]\ge 1-\frac{1}{poly(m,n,p,\delta,\frac{1}{\epsilon})} \text{ (i.e., w.h.p.)}
    \end{align*}
    Now as $v^{(t)}(\theta^{(t)},x,y)\le1$ and $||x^{(j)}||=1$, w.h.p. $\left|\left|\frac{\partial \mathcal{L}(\theta^{(t)},x,y)}{\partial w_{r,s}^{(t)}}\right|\right|=\Tilde{O}(1)$ so as the mini-batch gradient, $\left|\left|\frac{\partial \mathcal{L}(\theta^{(t)})}{\partial w_{r,s}^{(t)}}\right|\right|=\Tilde{O}(1)$.\\
    Now, from the update rule of the Step-3 of Algorithm \ref{alg:1}, $w_{r,s}^{(t)}-w_{r,s}^{(t+1)}=\eta\frac{\partial \mathcal{L}(\theta^{(t)})}{\partial w_{r,s}^{(t)}}$\\\\
    Therefore, using the property of Telescoping series, $\displaystyle w_{r,s}^{(0)}-w_{r,s}^{(t)}=\eta\overset{t}{\underset{i=1}{\sum}}\frac{\partial \mathcal{L}(\theta^{(i-1)})}{\partial w_{r,s}^{(t)}}$\\
    Therefore, $\left|\left|w_{r,s}^{(t)}-w_{r,s}^{(0)}\right|\right|=\Upsilon\eta t$ where we denote $\Tilde{O}(1)$ by $\Upsilon$\\\\
    Now, for every $\tau > 0,$ consider the set $\mathcal{H}_s:=\left\{ r\in[m/2]: \forall j\in J_s(w_s^{(t)},x), |\langle w_{r,s}^{(0)},x^{(j)}\rangle|\ge2\tau\right\}$\\\\
    Now, for every $t\le \cfrac{\tau}{\Upsilon\eta}$,
    $|\langle w_{r,s}^{(t)}-w_{r,s}^{(0)},x^{(j)}\rangle|\le \tau$\\ Which implies for every $r\in\mathcal{H}_s$, $t\le \cfrac{\tau}{\Upsilon\eta}$ and $j\in J_s(w_s^{(t)},x)$, $|\langle w_{r,s}^{(t)}, x^{(j)}\rangle|\ge\tau$\\\\
    Therefore, for every $r\in\mathcal{H}_s$, $t=\Tilde{O}(\frac{\tau}{\eta})$ and $j\in J_s(w_s^{(t)},x)$, $1_{\langle w_{r,s}^{(t)}, x^{(j)}\rangle\ge0}=1_{\langle w_{r,s}^{(0)}, x^{(j)}\rangle\ge0}$
    and hence, $\cfrac{\partial \mathcal{L}(\theta^{(t)},x,y)}{\partial w_{r,s}^{(t)}}=\cfrac{\overset{\sim}{\partial} \mathcal{L}(\theta^{(t)},x,y)}{\partial w_{r,s}^{(t)}}$\\\\
    Now, we will find the lower bound of $|\mathcal{H}_s|:$\\\\
    As, $w_{r,s}^{(0)}\sim\mathcal{N}(0,\sigma^2\mathbb{I}_{d\times d}),
    \forall j\in J_s(w_s^{(t)},x), \langle w_{r,s}^{(0)},x^{(j)}\rangle \sim \mathcal{N}(0,\sigma^2)$\\\\
    Hence, $\mathbb{P}[|\langle w_{r,s}^{(0)},x^{(j)}\rangle|\le2\tau]\le\cfrac{2e\tau}{\sigma}$\\ 
    Now as $|J_s(w_s^{(t)},x)|=l$, $\mathbb{P}[\forall j\in J_s(w_s^{(t)},x), |\langle w_{r,s}^{(0)},x^{(j)}\rangle|\ge2\tau]\ge1-\cfrac{2e\tau l}{\sigma}$\\\\
    Therefore, $|\mathcal{H}_s|\ge\left( 1-\cfrac{2e\tau l}{\sigma}\right)\frac{m}{2}$\\\\
\end{proof}
Using the following two lemmas we show that when $v_1^{(t)}$ is large, the expected pseudo-gradient of the loss function w.r.t. the hidden nodes of the \textit{expert 1} is large. Similar thing happens for \textit{expert 2} when $v_2^{(t)}$ is large.
We prove the first of these two lemmas for a fixed set $\{v^{(t)}(\theta^{(t)},x,y):(x,y)\sim\mathcal{D}\}$ which does not depend on the random initialization of the hidden nodes of the experts (i.e., on $\{w_{r,s}^{(0)}\}$). In the second of these two lemmas we remove the dependency on fixed set by means of a sampling trick introduced in \citep{li2018learning} to take a union bound over an epsilon-net on the set $\{v^{(t)}(\theta^{(t)},x,y):(x,y)\sim\mathcal{D}\}$.\\
\begin{lemma}\label{a_lemma_5}
    For any possible fixed set $\{v^{(t)}(\theta^{(t)},x,y):(x,y)\sim\mathcal{D}\}$ (that does not depend on $w_{r,s}^{(0)}$) such that $v_s^{(t)}=v_1^{(t)}$ for $s=1$ and $v_s^{(t)}=v_2^{(t)}$ for $s=2$ we have for every $l\ge l^*$:\\\\
    \centerline{$\mathbb{P}\left[||P_{r,s}^{(t)}||=\overset{\sim}{\Omega}\left(\cfrac{v_s^{(t)}}{lp\sqrt{\delta}}\right)\right]=\Omega\left(\cfrac{1}{p\sqrt{\delta}}\right)$}
\end{lemma}
\begin{proof}
        WLOG, let's assume $s=1$. Now,
    \begin{align*}
        P_{r,1}^{(t)}&=\mathbb{E}_{\mathcal{D}|y=+1}\left[v^{(t)}(\theta^{(t)},x,y)\left(\cfrac{1}{l}\underset{j\in J_1(w_1^{(t)},x)}{\sum}x^{(j)}1_{\langle w_{r,1}^{(0)},x^{(j)}\rangle\ge0}\right)\Big\vert y=+1\right]\\
        &\hspace{1.5cm}-\mathbb{E}_{\mathcal{D}|y=-1}\left[v^{(t)}(\theta^{(t)},x,y)\left(\cfrac{1}{l}\underset{j\in J_1(w_1^{(t)},x)}{\sum}x^{(j)}1_{\langle w_{r,1}^{(0)},x^{(j)}\rangle\ge0}\right)\Big\vert y=-1\right]
    \end{align*}
    Then,
    \begin{align*}
        h(w_{r,1}^{(0)})&:=\langle P_{r,1}, w_{r,1}^{(0)}\rangle\\
        &=\mathbb{E}_{\mathcal{D}|y=+1}\left[v^{(t)}(\theta^{(t)},x,y)\left(\cfrac{1}{l}\underset{j\in J_1(w_1^{(t)},x)}{\sum}\textbf{ReLU}\left(\langle w_{r,1}^{(0)},x^{(j)}\rangle\right)\right)\Big\vert y=+1\right]\\
        &\hspace{1.2cm}-\mathbb{E}_{\mathcal{D}|y=-1}\left[v^{(t)}(\theta^{(t)},x,y)\left(\cfrac{1}{l}\underset{j\in J_1(w_1^{(t)},x)}{\sum}\textbf{ReLU}\left(\langle w_{r,1}^{(0)},x^{(j)}\rangle\right)\right)\Big\vert y=-1\right]
    \end{align*}
    Now, let us decompose $w_{r,1}^{(0)}=\alpha o_1 +\beta$, where $\beta\perp o_1$\\\\
    Then,
    \begin{align*}
        h(w_{r,1}^{(0)})&=\mathbb{E}_{\mathcal{D}|y=+1}\left[v^{(t)}(\theta^{(t)},x,y)\left(\cfrac{1}{l}\underset{j\in J_1(w_1^{(t)},x)}{\sum}\textbf{ReLU}\left(\alpha\langle o_1,x^{(j)}\rangle+\langle \beta,x^{(j)}\rangle\right)\right)\Big\vert y=+1\right]\\
        &\hspace{0.2cm}-\mathbb{E}_{\mathcal{D}|y=-1}\left[v^{(t)}(\theta^{(t)},x,y)\left(\cfrac{1}{l}\underset{j\in J_1(w_1^{(t)},x)}{\sum}\textbf{ReLU}\left(\alpha\langle o_1,x^{(j)}\rangle+\langle \beta,x^{(j)}\rangle\right)\right)\Big\vert y=-1\right]\\
        &=\phi(\alpha)-l(\alpha)
    \end{align*}
    Where,
    \begin{align*}
    \phi(\alpha):=&\mathbb{E}_{\mathcal{D}|y=+1}\left[v^{(t)}(\theta^{(t)},x,y)\left(\cfrac{1}{l}\underset{j\in J_1(w_1^{(t)},x)}{\sum}\textbf{ReLU}\left(\alpha\langle o_1,x^{(j)}\rangle+\langle \beta,x^{(j)}\rangle\right)\right)\Big\vert y=+1\right]
    \end{align*}
    and,
    \begin{align*}
        l(\alpha):=\mathbb{E}_{\mathcal{D}|y=-1}\left[v^{(t)}(\theta^{(t)},x,y)\left(\cfrac{1}{l}\underset{j\in J_1(w_1^{(t)},x)}{\sum}\textbf{ReLU}\left(\alpha\langle o_1,x^{(j)}\rangle+\langle \beta,x^{(j)}\rangle\right)\right)\Big\vert y=-1\right]
    \end{align*}
    Note that, $\phi(\alpha)$ and $l(\alpha)$ both are convex functions.\\\\
    Now for $l\ge l^*$, using Lemma \ref{a_corollary_1}, we can express $\phi(\alpha)$ as follows:
    \begin{align*}
    &\phi(\alpha)=\cfrac{v_1^{(t)}}{l}\textbf{ReLU}\left(\alpha\right)\\
    &\hspace{0.2cm}+\mathbb{E}_{\mathcal{D}|y=+1}\left[v^{(t)}(\theta^{(t)},x,y)\left(\cfrac{1}{l}\underset{j\in J_1(x)/\underset{j\in J_1(x)}{\text{arg}}x^{(j)}=o_1}{\sum}\textbf{ReLU}\left(\alpha\langle o_1,x^{(j)}\rangle+\langle \beta,x^{(j)}\rangle\right)\right)\Big\vert y=+1\right]
    \end{align*}
    Now, for any class-irrelevant pattern set $S_i$ where $i\in[p]$, let us define $q_i^*\in S_i$ such that $q_i^*=\cfrac{\mathbb{E}_{S_i}[q]}{||\mathbb{E}_{S_i}[q]||}$. Also, let us define the set, $\mathcal{H}:=\{q_i^*:i\in[p]\}\cup\{o_2\}$\\\\
    Now let us define the event $\mathcal{E}_\tau: (i) \hspace{0.1cm}|\alpha|\le\tau; \hspace{0.1cm}(ii)\hspace{0.1cm} \forall q^\prime\in\mathcal{H}: |\langle\beta,q^\prime\rangle|\ge4\tau$\\\\
    Now, as $\alpha\sim\mathcal{N}(0,\sigma^2)$, for every $q^\prime\in\mathcal{H}, \langle\beta,q^\prime\rangle\sim\mathcal{N}\left(0,\left(1-\langle o_1,q^\prime\rangle^2\right)\sigma^2\right)$\\\\
    Now, $1-\langle o_1,q^\prime\rangle^2\ge\frac{1}{\delta}$. Hence, $\mathbb{P}[\exists q^\prime\in\mathcal{H}: |\langle\beta, q^\prime\rangle|\le4\tau]\le\cfrac{4e\tau p\sqrt{\delta}}{\sigma}$\\\\
    Therefore, $\mathbb{P}[\forall q^\prime\in\mathcal{H}: |\langle\beta, q^\prime\rangle|\ge4\tau]\ge1-\cfrac{4e\tau p\sqrt{\delta}}{\sigma}$.\\\\
    Picking, $\tau\le\cfrac{\sigma}{8ep\sqrt{\delta}}$ gives, $\mathbb{P}[\forall q^\prime\in\mathcal{H}: |\langle\beta, q^\prime\rangle|\ge4\tau]\ge\cfrac{1}{2}$.\\\\
    On the other hand, $\mathbb{P}[|\alpha|\le\tau]=\Omega\left(\cfrac{\tau}{\sigma}\right)$. Therefore, $\mathbb{P}[\mathcal{E}_\tau]=\Omega\left(\cfrac{\tau}{\sigma}\right)$\\\\
    Now, $\forall i\in[p]$ s.t. $q\in S_i, \mathbb{E}[|\langle w_{r,1}^{(0)},q-q_i^*\rangle|]\le\mathbb{E}_{\mathcal{N}(0,\sigma^2\mathbb{I}_{d\times d})}[||w_{r,1}^{(0)}||]\mathbb{E}_{S_i}[||q-q_i^*||]\le\tau$,
    where the last inequality comes from the bound of the diameter of the pattern sets and the fact that for any $X\sim\mathcal{N}(0,\sigma^2\mathbb{I}_{d\times d}), \mathbb{E}[||X||]\le 4\sigma\sqrt{d}$.\\\\
    Therefore, using Markov's inequality $\forall i\in[p]$ s.t. $q\in S_i,\mathbb{P}[|\langle w_{r,1}^{(0)},q-q_i^*\rangle|\le 2\tau]\ge\frac{1}{2}$\\\\
    Now,\\
    $\forall i\in[p],$ s.t. $q\in S_i, \textbf{ReLU}\left(\alpha\langle o_1,q\rangle+\langle\beta, q\rangle\right)=\textbf{ReLU}\left(\alpha\langle o_1,q_i^*\rangle+\langle\beta, q_i^*\rangle+\langle w_{r,1}^{(0)},q-q_i^*\rangle\right)$\\\\
    Now, conditioned on the event $\mathcal{E}_\tau$, for a fixed $\beta$ and $\alpha$ is the only random variable,\\
    $\forall i\in[p]$ s.t. $q\in S_i, \textbf{ReLU}\left(\alpha\langle o_1,q\rangle+\langle\beta, q\rangle\right)=\left(\alpha\langle o_1,q\rangle+\langle\beta, q\rangle\right)1_{\langle\beta, q_i^*\rangle\ge0}$ 
    which is a linear function of $\alpha\in[-\tau,\tau]$ with probability at least $\frac{1}{2}$ and, $\textbf{ReLU}\left(\alpha\langle o_1,o_2\rangle+\langle\beta, o_2\rangle\right)=\left(\alpha\langle o_1,o_2\rangle+\langle\beta, o_2\rangle\right)1_{\langle\beta, o_2\rangle\ge0}$ which is a linear function of $\alpha\in[-\tau,\tau]$ with probability $1$.\\\\
    Now, let us define $\{\partial l(\alpha)\}$ and $\{\partial \phi(\alpha)\}$ as the set of sub-gradient at the point $\alpha$ for $l(\alpha)$ and $\phi(\alpha)$ respectively such that $\partial_{\text{max}}l(\alpha)=\text{max}\{\partial l(\alpha)\}$, $\partial_{\text{max}}\phi(\alpha)=\text{max}\{\partial \phi(\alpha)\}$, $\partial_{\text{min}}l(\alpha)=\text{min}\{\partial l(\alpha)\}$ and $\partial_{\text{min}}\phi(\alpha)=\text{min}\{\partial \phi(\alpha)\}$.\\\\
    Then, using the above argument, conditioned on the event $\mathcal{E}_\tau$, $\partial_{\text{max}}l(\tau)-\partial_{\text{min}}l(-\tau)=0$.\\
    On the other hand, $\partial_{\text{max}}\phi(\tau/2)-\partial_{\text{min}}\phi(-\tau/2)=\cfrac{v_1^{(t)}}{l}$.\\\\
    Now using Lemma \ref{convex_li_liang_18}, conditioned on the event $\mathcal{E}_\tau$, $\underset{\alpha\sim U(-\tau,\tau)}{\mathbb{P}}\left[|\phi(\alpha)-l(\alpha)|\ge\cfrac{v_1^{(t)}\tau }{512l}\right]\ge\cfrac{1}{64}$.\\\\
    Now, for $\tau\le\cfrac{\sigma}{8ep\sqrt{\delta}}$, conditioned on $\mathcal{E}_\tau$, the density $p(\alpha)\in\left[\cfrac{1}{e\tau},\cfrac{e}{\tau}\right]$, which implies that,
    \begin{equation}\label{eq:1}
        \mathbb{P}\left[h(w_{r,1}^{(0)})\ge\cfrac{v_1^{(t)}\tau }{128l}\right]\ge\mathbb{P}\left[h(w_{r,1}^{(0)})\ge\cfrac{v_1^{(t)}\tau }{128l}\big|\mathcal{E}_\tau\right]\mathbb{P}\left[\mathcal{E}_\tau\right]=\Omega\left(\cfrac{\tau}{\sigma}\right)
    \end{equation}
    Now, as $v_1^{(t)}$ does not depends on $w_{r,1}^{(0)}$, $\langle P_{r,1}^{(t)}, w_{r,1}^{(0)}\rangle\sim\mathcal{N}(0,\sigma^2||P_{r,1}^{(t)}||^2)$.\\\\
    Now, using a concentration bound of Gaussian RV (i.e., $\mathbb{P}[X\ge\sigma x]\le e^{-x^2/2}$),
    \begin{equation}\label{eq:2}
        \mathbb{P}[\langle P_{r,1}^{(t)}, w_{r,1}^{(0)}\rangle\ge (\sigma||P_{r,1}^{(t)}||)10c]\le e^{-50c^2}; \text{ here }c>10.
    \end{equation}
    Now, taking $c=100\sqrt{\log{\cfrac{p\sqrt{\delta}}{\sigma}}}$ in (\ref{eq:2}) we get,
    \begin{equation}\label{eq:3}
        \mathbb{P}[\langle P_{r,1}^{(t)}, w_{r,1}^{(0)}\rangle=\Tilde{\Omega}(\sigma||P_{r,1}^{(t)}||)]=o(1)
    \end{equation}
    On the other hand, picking $\tau=\Theta(\frac{\sigma}{p\sqrt{\delta}})$ and plugging in at (\ref{eq:1}) gives,
    \begin{equation}\label{eq:4}
        \mathbb{P}\left[\langle P_{r,1}^{(t)},w_{r,1}^{(0)}\rangle=\Omega\left(\sigma\cfrac{v_1^{(t)}}{lp\sqrt{\delta}}\right)\right]=\Omega\left(\frac{1}{p\sqrt{\delta}}\right)
    \end{equation}
    Comparing (\ref{eq:3}) and (\ref{eq:4}) we get,
    $\mathbb{P}\left[||P_{r,1}^{(t)}||=\Tilde{\Omega}\left(\cfrac{v_1^{(t)}}{lp\sqrt{\delta}}\right)\right]=\Omega\left(\cfrac{1}{p\sqrt{\delta}}\right)$\\\\
\end{proof}
\begin{lemma}\label{a_lemma_6}
    Let $v_s^{(t)}=v_1^{(t)}$ for $s=1$ and $v_s^{(t)}=v_2^{(t)}$ for $s=2$. Then, for every $v_s^{(t)}>0$, for $m=\Tilde{\Omega}\left(\cfrac{l^2p^3\delta^{3/2}}{(v_s^{(t)})^2}\right)$, for every possible set $\{v^{(t)}(\theta^{(t)},x,y):(x,y)\sim\mathcal{D}\}$ (that depends on $w_{r,s}^{(0)}$), there exist at least $\Omega\left(\frac{1}{p\sqrt{\delta}}\right)$ fraction of $r\in [m/2]$ of the expert $s\in\{1,2\}$ such that for every $l\ge l^*$,\\\\
    \centerline{$\left|\left|\cfrac{\overset{\sim}{\partial}\Hat{\mathcal{L}}(\theta^{(t)})}{\partial w_{r,s}^{(t)}}\right|\right|=\Tilde{\Omega}\left(\cfrac{v_s^{(t)}}{lp\sqrt{\delta}}\right)$}\\\\
\end{lemma}
\begin{proof}
    Let us pick $S$ samples to form $\textbf{S}=\{(x_i,y_i)\}_{i=1}^S$ with $S/2$ many samples from $y=+1$ and $S/2$ many samples from $y=-1$. Let us denote the subset of samples with $y=+1$ as $\textbf{S}_{1}$ and the subset of samples with $y=-1$ as $\textbf{S}_{2}$.
    Therefore, $|\textbf{S}_{1}|=|\textbf{S}_{2}|=S/2$. Let us denote the corresponding value function of $i$-th sample of $\textbf{S}$ as $v^{(t)}(\theta^{(t)},x_i,y_i)$. Since, each $v^{(t)}(\theta^{(t)},x_i,y_i)\in[0,1]\,$ using Hoeffding's inequality we know that w.h.p. :
    \begin{align*}
        \left|v_s^{(t)}-\cfrac{1}{S/2}\underset{(x_i,y_i)\in\textbf{S}_s}{\sum}v^{(t)}(\theta^{(t)},x_i,y_i)\right|=\Tilde{O}\left(\cfrac{1}{\sqrt{S}}\right)
    \end{align*}
    This implies that, as long as $S=\Tilde{\Omega}\left(\cfrac{1}{(v_s^{(t)})^2}\right)$, we will have that,
    \begin{align*}
        &\cfrac{1}{S/2} \underset{(x_i,y_i)\in\textbf{S}_{s}}{\sum} v^{(t)}(\theta^{(t)},x_i,y_i) \in\left[\cfrac{1}{2}v_s^{(t)},\cfrac{3}{2}v_s^{(t)}\right]
    \end{align*}
    Now, the average pseudo-gradient over the set $\textbf{S}$,
    \begin{align*}
        \displaystyle\cfrac{1}{S}\underset{(x_i,y_i)\in \textbf{S}}{\sum}\cfrac{\overset{\sim}{\partial}\mathcal{L}(\theta^{(t)},x_i,y_i)}{\partial w_{r,s}^{(t)}}&=\cfrac{1}{S}\underset{(x_i,y_i)\in \textbf{S}}{\sum}-ya_{r,s}v^{(t)}(\theta^{(t)},x_i,y_i)\left(\cfrac{1}{l}\underset{j\in J_s(w_s^{(t)},x_i)}{\sum}x_i^{(j)}1_{\langle w_{r,s}^{(0)},x_i^{(j)}\rangle\ge0}\right)\\
        &=-\cfrac{a_{r,s}}{2}P_{r,s}^{(t)}(\textbf{S})
    \end{align*}
    where,
    \begin{align*}
        P_{r,s}^{(t)}(\textbf{S})&=\cfrac{1}{S/2}\underset{(x_i,y_i)\in \textbf{S}_{1}}{\sum}v^{(t)}(\theta^{(t)},x_i,y_i)\left(\cfrac{1}{l}\underset{j\in J_s(w_s^{(t)},x_i)}{\sum}x_i^{(j)}1_{\langle w_{r,s}^{(0)},x_i^{(j)}\rangle\ge0}\right)\\
        &\hspace{1cm}-\cfrac{1}{S/2}\underset{(x_i,y_i)\in \textbf{S}_{2}}{\sum}v^{(t)}(\theta^{(t)},x_i,y_i)\left(\cfrac{1}{l}\underset{j\in J_s(w_s^{(t)},x_i)}{\sum}x_i^{(j)}1_{\langle w_{r,s}^{(0)},x_i^{(j)}\rangle\ge0}\right)
    \end{align*}
    Now as $a_{r,s}\sim\mathcal{N}(0,1)$, $\mathbb{P}\left[\left|\left|\cfrac{1}{S}\underset{(x_i,y_i)\in \textbf{S}}{\sum}\cfrac{\overset{\sim}{\partial}\mathcal{L}(\theta^{(t)},x_i,y_i)}{\partial w_{r,s}^{(t)}}\right|\right|=\left|\left|\cfrac{a_{r,s}}{2}P_{r,s}^{(t)}(\textbf{S})\right|\right|\ge\cfrac{1}{2}\left|\left|P_{r,s}^{(t)}(\textbf{S})\right|\right|\right]\ge\cfrac{1}{e}$\\\\
    Now for a fixed set $\{v^{(t)}(\theta^{(t)},x_i,y_i):(x_i,y_i)\in \textbf{S}\}$ as long as $S=\Tilde{\Omega}\left(\cfrac{1}{v_s^2{(t)}}\right)$, for every $l\ge l^*$ using Lemma \ref{a_lemma_5},
    \begin{align*}
        \mathbb{P}\left[||P_{r,s}^{(t)}(\textbf{S})||=\Tilde{\Omega}\left(\cfrac{v_s^{(t)}}{lp\sqrt{\delta}}\right)\right]=\Omega\left(\cfrac{1}{p\sqrt{\delta}}\right)
    \end{align*}
    Hence, for a fixed set $\{v^{(t)}(\theta^{(t)},x_i,y_i):(x_i,y_i)\in\textbf{S}\}$, the probability that there are less than $O\left(\frac{1}{p\sqrt{\delta}}\right)$ fraction of $r\in[m/2]$ such that $\left|\left|\cfrac{1}{S}\underset{(x_i,y_i)\in \textbf{S}}{\sum}\cfrac{\overset{\sim}{\partial}\mathcal{L}(\theta^{(t)},x_i,y_i)}{\partial w_{r,s}^{(t)}}\right|\right|$ is $\Tilde{\Omega}\left(\cfrac{v_s^{(t)}}{lp\sqrt{\delta}}\right)$ is no more than $p_{\text{fix}}$ where, $p_{\text{fix}}\le \exp{\left(-\Omega\left(\frac{m}{p\sqrt{\delta}}\right)\right)}$.\\\\
    Moreover, for every $\bar{\varepsilon}>0$, for two different $\{v^{(t)}(\theta^{(t)},x_i,y_i):(x_i,y_i)\in \textbf{S}\}$, $\{v^{\prime(t)}(\theta^{(t)},x_i,y_i):(x_i,y_i)\in \textbf{S}\}$ such that $\forall (x_i,y_i)\in \textbf{S}$, $|v^{(t)}(\theta^{(t)},x_i,y_i)-v^{\prime(t)}(\theta^{(t)},x_i,y_i)|\le\bar{\varepsilon}$, since w.h.p. $|a_{r,s}|=\Tilde{O}(1)$,
    \begin{align*}
        &\left|\left|\cfrac{1}{S}\underset{(x_i,y_i)\in \textbf{S}}{\sum}-ya_{r,s}(v^{(t)}(\theta^{(t)},x_i,y_i)-v^{\prime(t)}(\theta^{(t)},x_i,y_i))\left(\cfrac{1}{l}\underset{j\in J_s(w_s^{(t)},x_i)}{\sum}x_i^{(j)}1_{\langle w_{r,s}^{(0)},x_i^{(j)}\rangle\ge0}\right)\right|\right|\\&=\Tilde{O}(\bar{\varepsilon})
    \end{align*}
    which implies that we can take $\bar{\varepsilon}$-net with $\bar{\varepsilon}=\Tilde{\Theta}\left(\cfrac{v_s^{(t)}}{lp\sqrt{\delta}}\right)$.\\\\
    Thus, the probability that there exists $\{v^{(t)}(\theta^{(t)},x_i,y_i):(x_i,y_i)\in\textbf{S}\}$ such that there are no more than $O\left(\frac{1}{p\sqrt{\delta}}\right)$ fraction of $r\in [m/2]$ with $\left|\left|\cfrac{1}{S}\underset{(x_i,y_i)\in \textbf{S}}{\sum}\cfrac{\overset{\sim}{\partial}\mathcal{L}(\theta^{(t)},x_i,y_i)}{\partial w_{r,s}^{(t)}}\right|\right|=\Tilde{\Omega}\left(\cfrac{v_s^{(t)}}{lp\sqrt{\delta}}\right)$ is no more than,
    $p\le p_{\text{fix}}\left(\cfrac{v_s^{(t)}}{\bar{\varepsilon}}\right)^{S}\le \exp{\left(-\Omega\left(\frac{m}{p\sqrt{\delta}}\right)+S\log{\left(\cfrac{v_s^{(t)}}{\bar{\varepsilon}}\right)}\right)}$.\\\\
    Hence, for $m=\overset{\sim}{\Omega}\left(Sp\sqrt{\delta}\right)$ with $S=\Tilde{\Omega}\left(\cfrac{1}{v_s^2{(t)}}\right)$, w.h.p. for every possible choice of $\{v^{(t)}(\theta^{(t)},x_i,y_i):(x_i,y_i)\in\textbf{S}\}$, there are at least $\Omega\left(\frac{1}{p\sqrt{\delta}}\right)$ fraction of $r\in[m/2]$ such that,
    \begin{align*}
        \left|\left|\cfrac{1}{S}\underset{(x_i,y_i)\in \textbf{S}}{\sum}\cfrac{\overset{\sim}{\partial}\mathcal{L}(\theta^{(t)},x_i,y_i)}{\partial w_{r,s}^{(t)}}\right|\right|=\Tilde{\Omega}\left(\cfrac{v_s^{(t)}}{lp\sqrt{\delta}}\right)
    \end{align*}
    Now, we consider the difference between the sample gradient and the expected gradient. Since, $\left|\left|\cfrac{\overset{\sim}{\partial}\mathcal{L}(\theta^{(t)},x_i,y_i)}{\partial w_{r,s}^{(t)}}\right|\right|=\overset{\sim}{O}(1)$, by using the Hoeffding's inequality, we know that for every $r \in [m/2]$:
    \begin{align*}
        \left|\left|\cfrac{1}{S}\underset{(x_i,y_i)\in S}{\sum}\cfrac{\overset{\sim}{\partial}\mathcal{L}(\theta^{(t)},x_i,y_i)}{\partial w_{r,s}^{(t)}}-\cfrac{\overset{\sim}{\partial}\Hat{\mathcal{L}}(\theta^{(t)})}{\partial w_{r,s}^{(t)}}\right|\right|=\Tilde{O}\left(\cfrac{1}{\sqrt{S}}\right)
    \end{align*}
    This implies that as long as $S=\Tilde{\Omega}\left(\left(\cfrac{lp\sqrt{\delta}}{v_s^{(t)}}\right)^2\right)$ and hence for $m=\Tilde{\Omega}\left(\cfrac{l^2p^3\delta^{3/2}}{(v_s^{(t)})^2}\right)$, such $r \in [m/2]$ also have:
    \begin{align*}
        \left|\left|\cfrac{\overset{\sim}{\partial}\Hat{\mathcal{L}}(\theta^{(t)})}{\partial w_{r,s}^{(t)}}\right|\right|=\Tilde{\Omega}\left(\cfrac{v_s^{(t)}}{lp\sqrt{\delta}}\right)
    \end{align*}
\end{proof}
\begin{lemma}\label{a_lemma_7}
    Let us define $v^{(t)}:=\sqrt{\underset{s\in\{1,2\}}{\sum}(v_s^{(t)})^2}$ where $v_s^{(t)}=v_1^{(t)}$ for $s=1$ and $v_s^{(t)}=v_2^{(t)}$ for $s=2$; $\gamma:=\Omega\left(\frac{1}{p\sqrt{\delta}}\right)$. 
    Then, by selecting learning rate $\eta=\Tilde{O}\left(\cfrac{\gamma^3(v^{(t)})^2}{ml^2}\right)$ and batch size $B=\Tilde{\Omega}\left(\cfrac{l^4}{\gamma^6(v^{(t)})^4}\right)$, at each iteration $t$ of the Step-3 of Algorithm \ref{alg:1} such that $t=\Tilde{O}\left(\cfrac{\sigma\gamma^3(v^{(t)})^2}{\eta l^3}\right)$, w.h.p. we can ensure that for every $l\ge l^*$,
    \begin{align*}
        \Delta L(\theta^{(t)},\theta^{(t+1)})\ge \cfrac{\eta m \gamma^3}{l^2} \Tilde{\Omega}\left((v^{(t)})^2\right)
    \end{align*}
\end{lemma}
\begin{proof}
    For every $l\ge l^*$, from Lemma \ref{a_lemma_6}, for at least $\gamma$ fraction of $r\in[m/2]$ of expert $s$:
    \begin{align*}
        \left|\left|\cfrac{\overset{\sim}{\partial}\Hat{\mathcal{L}}(\theta^{(t)})}{\partial w_{r,s}^{(t)}}\right|\right|=\Tilde{\Omega}\left(\cfrac{v_s^{(t)}}{lp\sqrt{\delta}}\right)
    \end{align*}
    Now w.h.p., $\left|\left|\cfrac{\Tilde{\partial}\mathcal{L}(\theta^{(t)},x,y)}{\partial w_{r,s}^{(t)}}\right|\right|=\overset{\sim}{O}(1)$. Therefore, w.h.p. over a randomly sampled batch from $\mathcal{D}$ at iteration $t$ denoted as $\mathcal{B}_t$ of size $B$:
    \begin{align*}
         \left|\left|\cfrac{1}{B}\underset{(x,y)\in \mathcal{B}_t}{\sum}\cfrac{\overset{\sim}{\partial}\mathcal{L}(\theta^{(t)},x,y)}{\partial w_{r,s}^{(t)}}-\cfrac{\overset{\sim}{\partial}\Hat{\mathcal{L}}(\theta^{(t)})}{\partial w_{r,s}^{(t)}}\right|\right|=\Tilde{O}\left(\cfrac{1}{\sqrt{B}}\right)
    \end{align*}
    This implies, by selecting batch-size of $B=\Omega\left(\cfrac{l^2p^2\delta}{(v_s^{(t)})^2}\right)$, for these $\gamma$ fraction of $r\in[m/2]$ of expert $s$ we can ensure that:
    \begin{align*}
        \left|\left|\cfrac{1}{B}\underset{(x_i,y_i)\in \mathcal{B}_t}{\sum}\cfrac{\overset{\sim}{\partial}\mathcal{L}(\theta^{(t)},x,y)}{\partial w_{r,s}^{(t)}}\right|\right|=\Tilde{\Omega}\left(\cfrac{v_s^{(t)}}{lp\sqrt{\delta}}\right)
    \end{align*}
    Now using Lemma \ref{a_lemma_4}, for a fixed $(x,y)\in\mathcal{B}_t$, by selecting $\tau=\cfrac{\sigma\gamma}{4elB}$ we have $\left(1-\cfrac{\gamma}{2B}\right)$ fraction of $r\in[m/2]$ of the expert $s$:
    \begin{align*}
        \cfrac{\partial \mathcal{L}(\theta^{(t)},x,y)}{\partial w_{r,s}^{(t)}}=\cfrac{\overset{\sim}{\partial} \mathcal{L}(\theta^{(t)},x,y)}{\partial w_{r,s}^{(t)}}
    \end{align*}
    Therefore, at least $(1-\gamma/2)$ fraction of $r\in[m/2]$ of the expert $s$:
    \begin{align*}
        \cfrac{\partial \mathcal{L}(\theta^{(t)},x,y)}{\partial w_{r,s}^{(t)}}=\cfrac{\overset{\sim}{\partial} \mathcal{L}(\theta^{(t)},x,y)}{\partial w_{r,s}^{(t)}} \hspace{1cm}\forall (x,y)\in\mathcal{B}_t
    \end{align*}
    Recall our definition of loss-function for SGD at iteration $t$ with mini-batch $\mathcal{B}_t$, $\mathcal{L}(\theta^{(t)})=\cfrac{1}{B}\sum_{(x,y)\in\mathcal{B}_t}\log{(1+e^{-yf_M(\theta^{(t)}, x)})}=\cfrac{1}{B}\sum_{(x,y)\in\mathcal{B}_t}\mathcal{L}(\theta^{(t)},x,y)$ and the corresponding batch-gradient at iteration $t$, $\cfrac{\partial\mathcal{L}(\theta^{(t)})}{\partial w_{r,s}^{(t)}}=\cfrac{1}{B}\sum_{(x,y)\in\mathcal{B}_t}\cfrac{\partial\mathcal{L}(\theta^{(t)},x,y)}{\partial w_{r,s}^{(t)}}$. Therefore, there are at least $\gamma/2$ fraction of $r\in[m/2]$ of the expert $s$:
    \begin{align*}
        \left|\left|\cfrac{\partial\mathcal{L}(\theta^{(t)})}{\partial w_{r,s}^{(t)}}\right|\right|=\Tilde{\Omega}\left(\cfrac{v_s^{(t)}}{lp\sqrt{\delta}}\right)
    \end{align*}
    Now for any $(x^\prime,y^\prime)\sim\mathcal{D}$, according to Lemma \ref{a_lemma_4}, w.h.p. there are at least $1-\cfrac{2e\tau l}{\sigma}$ fraction of $r \in [m/2]$ of the expert $s$ such that $\forall j\in J_s(w_s^{(t)},x^\prime), |\langle w_{r,s}^{(t)}, x^{\prime(j)}\rangle|\ge\tau$. Let us denote the set of these $r$'s of $s$ as $\mathcal{S}_{r,s}$. Therefore, on the set $\underset{s\in\{1,2\}}{\bigcup}\mathcal{S}_{r,s}$, the loss function $\mathcal{L}(\theta^{(t)}, x^{\prime}, y^{\prime})$ is $\overset{\sim}{O}(1)$ -smooth and  $\overset{\sim}{O}(1)$ -Lipschitz smooth.\\\\
    On the other hand, the update rule of SGD at the iteration $t$ is, $\theta^{(t+1)}=\theta^{(t)}-\eta\cfrac{\partial\mathcal{L}(\theta^{(t)})}{\partial w_{r,s}^{(t)}}$\\\\
    Therefore, using Lemma \ref{converge_li_liang_18},
    \begin{align*}
        &\Delta L(\theta^{(t)},\theta^{(t+1)},x^\prime,y^\prime):=\mathcal{L}(\theta^{(t)},x^\prime,y^\prime)-\mathcal{L}({\theta^{(t+1)},x^\prime,y^\prime})\\
        &\ge\eta\underset{r \in \underset{s\in[2]}{\bigcup}\mathcal{S}_{r,s}}{\sum} \left\langle\cfrac{\partial\mathcal{L}(\theta^{(t)})}{\partial w_{r,s}^{(t)}},\cfrac{\partial\mathcal{L}(\theta^{(t)},x^\prime,y^\prime)}{\partial w_{r,s}^{(t)}}\right\rangle-\underset{s\in[2],r\in[m/2]\backslash\underset{s\in[2]}{\cup}\mathcal{S}_{r,s}}{\sum}\overset{\sim}{O}(\eta)-\overset{\sim}{O}\left(\eta^2m^2\right)\\
        &\ge\eta\underset{r \in [m/2],s\in[2]}{\sum} \left\langle\cfrac{\partial\mathcal{L}(\theta^{(t)})}{\partial w_{r,s}^{(t)}},\cfrac{\partial\mathcal{L}(\theta^{(t)},x^\prime,y^\prime)}{\partial w_{r,s}^{(t)}}\right\rangle-\overset{\sim}{O}\left(\cfrac{\eta \tau lm}{\sigma}\right)-\overset{\sim}{O}\left(\eta^2m^2\right)
    \end{align*}
    Let us denote the event,
    \begin{align*}
        \mathcal{E}_0:&\\
        &\Delta L(\theta^{(t)},\theta^{(t+1)},x^\prime,y^\prime)\\
        &\ge\eta\underset{r \in [m/2],s\in[2]}{\sum} \left\langle\cfrac{\partial\mathcal{L}(\theta^{(t)})}{\partial w_{r,s}^{(t)}},\cfrac{\partial\mathcal{L}(\theta^{(t)},x^\prime,y^\prime)}{\partial w_{r,s}^{(t)}}\right\rangle-\overset{\sim}{O}\left(\cfrac{\eta \tau lm}{\sigma}\right)-\overset{\sim}{O}\left(\eta^2m^2\right)
    \end{align*}
    Then, $\mathbb{P}\left[\mathcal{E}_0\right]\ge 1-\cfrac{1}{poly(m,n,p,\delta,\frac{1}{\epsilon})}$ (i.e., w.h.p.) and hence $\mathbb{P}\left[\neg\mathcal{E}_0\right]\le \cfrac{1}{poly(m,n,p,\delta,\frac{1}{\epsilon})}$\\\\
    Also, let us define the event,
    \begin{align*}
        \mathcal{E}_1:&\\
        &\left|\mathcal{L}(\theta^{(t)},x^\prime,y^\prime)\right|=\Tilde{O}(m),\left|\left|\cfrac{\partial\mathcal{L}(\theta^{(t)},x^\prime,y^\prime)}{\partial w_{r,s}^{(t)}}\right|\right|=\Tilde{O}(1)\text{ and }\left|\left|\cfrac{\partial\mathcal{L}(\theta^{(t)})}{\partial w_{r,s}^{(t)}}\right|\right|=\Tilde{O}(1)
    \end{align*}
    Then, $\mathbb{P}\left[\mathcal{E}_1\right]\ge 1-\cfrac{1}{poly(m,n,p,\delta,\frac{1}{\epsilon})}$ and hence $\mathbb{P}\left[\neg\mathcal{E}_1\right]\le \cfrac{1}{poly(m,n,p,\delta,\frac{1}{\epsilon})}$\\\\
    Now, the expected gradient at iteration $t$, $\cfrac{{\partial}\Hat{\mathcal{L}}(\theta^{(t)})}{\partial w_{r,s}^{(t)}}:=\mathbb{E}_\mathcal{D}\left[\cfrac{\partial\mathcal{L}(\theta^{(t)},x^\prime,y^\prime)}{\partial w_{r,s}^{(t)}}\right]$\\\\
    Therefore condition on $\mathcal{E}_1$,
    \begin{align*}
        \cfrac{{\partial}\Hat{\mathcal{L}}(\theta^{(t)})}{\partial w_{r,s}^{(t)}}&=\mathbb{E}_\mathcal{D}\left[\cfrac{\partial\mathcal{L}(\theta^{(t)},x^\prime,y^\prime)}{\partial w_{r,s}^{(t)}}\big|\mathcal{E}_1\right]\\
        &=\mathbb{E}_\mathcal{D}\left[\cfrac{\partial\mathcal{L}(\theta^{(t)},x^\prime,y^\prime)}{\partial w_{r,s}^{(t)}}\big|\mathcal{E}_0,\mathcal{E}_1\right]\mathbb{P}\left[\mathcal{E}_0\big|\mathcal{E}_1\right]+\mathbb{E}_\mathcal{D}\left[\cfrac{\partial\mathcal{L}(\theta^{(t)},x^\prime,y^\prime)}{\partial w_{r,s}^{(t)}}\big|\neg\mathcal{E}_0,\mathcal{E}_1\right]\mathbb{P}\left[\neg\mathcal{E}_0\big|\mathcal{E}_1\right]
    \end{align*}
    Which implies,
    \begin{align*}
        \left|\left|\cfrac{{\partial}\Hat{\mathcal{L}}(\theta^{(t)})}{\partial w_{r,s}^{(t)}}-\mathbb{E}_\mathcal{D}\left[\cfrac{\partial\mathcal{L}(\theta^{(t)},x^\prime,y^\prime)}{\partial w_{r,s}^{(t)}}\big|\mathcal{E}_0,\mathcal{E}_1\right]\right|\right|\le\cfrac{\Tilde{O}(1)}{poly(m,n,p,\delta,\frac{1}{\epsilon})}
    \end{align*}
    Again, condition on $\mathcal{E}_1$,
    \begin{align*}
        &\Delta L(\theta^{(t)},\theta^{(t+1)}):=\mathbb{E}_\mathcal{D}\left[\mathcal{L}(\theta^{(t)},x^\prime,y^\prime)-\mathcal{L}(\theta^{(t+1)},x^\prime,y^\prime)\big|\mathcal{E}_1\right]\\
        &=\mathbb{E}_\mathcal{D}\left[\mathcal{L}(\theta^{(t)},x^\prime,y^\prime)-\mathcal{L}(\theta^{(t+1)},x^\prime,y^\prime)\big|\mathcal{E}_0,\mathcal{E}_1\right]\mathbb{P}\left[\mathcal{E}_0\big|\mathcal{E}_1\right]\\
        &\hspace{1cm}+\mathbb{E}_\mathcal{D}\left[\mathcal{L}(\theta^{(t)},x^\prime,y^\prime)-\mathcal{L}(\theta^{(t+1)},x^\prime,y^\prime)\big|\neg\mathcal{E}_0,\mathcal{E}_1\right]\mathbb{P}\left[\neg\mathcal{E}_0\big|\mathcal{E}_1\right]\\
        &\ge\eta\underset{r \in \left[m/2\right],s\in[2]}{\sum} \left\langle\cfrac{\partial\mathcal{L}(\theta^{(t)})}{\partial w_{r,s}^{(t)}},\mathbb{E}_\mathcal{D}\left[\cfrac{\partial\mathcal{L}(\theta^{(t)},x^\prime,y^\prime)}{\partial w_{r,s}^{(t)}}\big|\mathcal{E}_0,\mathcal{E}_1\right]\right\rangle-\overset{\sim}{O}\left(\cfrac{\eta \tau lm}{\sigma}\right)-\overset{\sim}{O}\left(\eta^2m^2\right)\\
        &\hspace{1cm}-\cfrac{\Tilde{O}(m)}{poly(m,n,p,\delta,\frac{1}{\epsilon})}\\
        &\ge\eta\underset{r \in \left[m/2\right],s\in[2]}{\sum} \left\langle\cfrac{\partial\mathcal{L}(\theta^{(t)})}{\partial w_{r,s}^{(t)}},\cfrac{{\partial}\Hat{\mathcal{L}}(\theta^{(t)})}{\partial w_{r,s}^{(t)}}\right\rangle-\overset{\sim}{O}\left(\cfrac{\eta \tau lm}{\sigma}\right)-\overset{\sim}{O}\left(\eta^2m^2\right)-\cfrac{\Tilde{O}(m)}{poly(m,n,p,\delta,\frac{1}{\epsilon})}\\
        &\hspace{1cm}-\cfrac{\Tilde{O}(\eta m)}{poly(m,n,p,\delta,\frac{1}{\epsilon})}\\
        &\ge\eta\underset{r \in \left[m/2\right],s\in[2]}{\sum} \left\langle\cfrac{\partial\mathcal{L}(\theta^{(t)})}{\partial w_{r,s}^{(t)}},\cfrac{{\partial}\Hat{\mathcal{L}}(\theta^{(t)})}{\partial w_{r,s}^{(t)}}\right\rangle-\overset{\sim}{O}\left(\cfrac{\eta \tau lm}{\sigma}\right)-\overset{\sim}{O}\left(\eta^2m^2\right)
    \end{align*}
    Now, w.h.p.
    \begin{align*}
        \left|\left|\cfrac{\partial\mathcal{L}(\theta^{(t)})}{\partial w_{r,s}^{(t)}}-\cfrac{{\partial}\Hat{\mathcal{L}}(\theta^{(t)})}{\partial w_{r,s}^{(t)}}\right|\right|=\Tilde{O}\left(\cfrac{1}{\sqrt{B}}\right)
    \end{align*}
    Therefore,
    \begin{align*}
        \left\langle\cfrac{\partial\mathcal{L}(\theta^{(t)})}{\partial w_{r,s}^{(t)}},\cfrac{{\partial}\Hat{\mathcal{L}}(\theta^{(t)})}{\partial w_{r,s}^{(t)}}\right\rangle\ge\left|\left|\cfrac{\partial\mathcal{L}(\theta^{(t)})}{\partial w_{r,s}^{(t)}}\right|\right|^2-\Tilde{O}\left(\cfrac{1}{\sqrt{B}}\right)
    \end{align*}
    Therefore,
    \begin{align*}
        \Delta L(\theta^{(t)},\theta^{(t+1)})&\ge\eta\underset{r \in \left[m\right],s\in[2]}{\sum} \left|\left|\cfrac{\partial\mathcal{L}(\theta^{(t)})}{\partial w_{r,s}^{(t)}}\right|\right|^2-\Tilde{O}\left(\cfrac{\eta \tau lm}{\sigma}\right)-\Tilde{O}\left(\eta^2m^2\right)-\eta\Tilde{O}\left(\cfrac{m}{\sqrt{B}}\right)\\
        &\ge \cfrac{\eta m \gamma^3}{l^2} \Tilde{\Omega}\left(\underset{s\in[2]}{\sum}(v_s^{(t)})^2\right)-\Tilde{O}\left(\cfrac{\eta \tau lm}{\sigma}\right)-\Tilde{O}\left(\eta^2m^2\right)-\eta\Tilde{O}\left(\cfrac{m}{\sqrt{B}}\right)\\
        &\ge \cfrac{\eta m \gamma^3}{l^2} \overset{\sim}{\Omega}\left((v^{(t)})^2\right)-\Tilde{O}\left(\cfrac{\eta \tau lm}{\sigma}\right)-\Tilde{O}\left(\eta^2m^2\right)-\eta\Tilde{O}\left(\cfrac{m}{\sqrt{B}}\right)
    \end{align*}
    Now selecting, $\eta=\Tilde{O}\left(\cfrac{\gamma^3(v^{(t)})^2}{ml^2}\right)$, $B=\Tilde{\Omega}\left(\cfrac{l^4}{\gamma^6(v^{(t)})^4}\right)$, $\tau=\Tilde{O}\left(\cfrac{\sigma\gamma^3(v^{(t)})^2}{l^3}\right)$ and hence for\\ $t=\Tilde{O}\left(\cfrac{\sigma\gamma^3(v^{(t)})^2}{\eta l^3}\right)$, we get,
    \begin{align*}
        \Delta L(\theta^{(t)},\theta^{(t+1)})\ge \cfrac{\eta m \gamma^3}{l^2} \Tilde{\Omega}\left((v^{(t)})^2\right)
    \end{align*}
\end{proof}

\section{Lemmas Used to Prove the Theorem \ref{Thm_mcnn_j}}\label{aux_joint_pmoe_thm_proof}
In joint-training pMoE i.e., for any iteration $t$ of the Step-2 of Algorithm \ref{alg:2}, the gradient of the loss for single-sample with respect to the hidden nodes of the experts:
\begin{equation}\label{grad_joint_pmoe}
    \frac{\partial \mathcal{L}(\theta^{(t)},x,y)}{\partial w_{r,s}^{(t)}}=-ya_{r,s}v^{(t)}(\theta^{(t)},x,y)\left(\cfrac{1}{l}\underset{j\in J_s(w_s^{(t)},x)}{\sum}G_{j,s}(w_s^{(t)},x)x^{(j)}1_{\langle w_{r,s}^{(t)},x^{(j)}\rangle\ge0}\right)
\end{equation}
and the corresponding \textit{\textbf{pseudo-gradient}}:
\begin{equation}\label{pseudo_grad_joint_moe}
    \frac{\overset{\sim}{\partial} \mathcal{L}(\theta^{(t)},x,y)}{\partial w_{r,s}^{(t)}}=-ya_{r,s}v^{(t)}(\theta^{(t)},x,y)\left(\cfrac{1}{l}\underset{j\in J_s(w_s^{(t)},x)}{\sum}G_{j,s}(w_s^{(t)},x)x^{(j)}1_{\langle w_{r,s}^{(0)},x^{(j)}\rangle\ge0}\right)
\end{equation}
and the expected pseudo-gradient:
\begin{align*}
    &\frac{\overset{\sim}{\partial}\Hat{\mathcal{L}}(\theta^{(t)})}{\partial w_{r,s}^{(t)}}=\mathbb{E}_{\mathcal{D}}\left[\frac{\overset{\sim}{\partial} \mathcal{L}(\theta^{(t)},x,y)}{\partial w_{r,s}^{(t)}}\right]\\
    &=-\cfrac{a_{r,s}}{2}\left(\mathbb{E}_{\mathcal{D}|y=+1}\left[v^{(t)}(\theta^{(t)},x,y)\left(\cfrac{1}{l}\underset{j\in J_s(w_s^{(t)},x)}{\sum}G_{j,s}(w_s^{(t)},x)x^{(j)}1_{\langle w_{r,s}^{(0)},x^{(j)}\rangle\ge0}\right)\Big\vert y=+1\right]\right.\\
    &\left.\hspace{1cm}-\mathbb{E}_{\mathcal{D}|y=-1}\left[v^{(t)}(\theta^{(t)},x,y)\left(\cfrac{1}{l}\underset{j\in J_s(w_s^{(t)},x)}{\sum}G_{j,s}(w_s^{(t)},x)x^{(j)}1_{\langle w_{r,s}^{(0)},P_jx\rangle\ge0}\right)\Big\vert y=-1\right]\right)\\
    &=-\cfrac{a_{r,s}}{2}P_{r,s}^{(t)}
\end{align*}
with,
\begin{align*}
    P_{r,s}^{(t)}&=\mathbb{E}_{\mathcal{D}|y=+1}\left[v^{(t)}(\theta^{(t)},x,y)\left(\cfrac{1}{l}\underset{j\in J_s(w_s^{(t)},x)}{\sum}G_{j,s}(w_s^{(t)},x)x^{(j)}1_{\langle w_{r,s}^{(0)},x^{(j)}\rangle\ge0}\right)\Big\vert y=+1\right]\\
    &\hspace{1cm}-\mathbb{E}_{\mathcal{D}|y=-1}\left[v^{(t)}(\theta^{(t)},x,y)\left(\cfrac{1}{l}\underset{j\in J_s(w_s^{(t)},x)}{\sum}G_{j,s}(w_s^{(t)},x)x^{(j)}1_{\langle w_{r,s}^{(0)},x^{(j)}\rangle\ge0}\right)\Big\vert y=-1\right]
\end{align*}
\begin{lemma}\label{a_lemma_8}
    W.h.p. over the random initialization of the hidden nodes of the experts defined in (\ref{eqn:expert_ini}), for every $(x,y)\sim\mathcal{D}$ and for every $\tau>0$, for every $t=\Tilde{O}\left(\cfrac{\tau l}{\eta}\right)$ of the Step-2 of Algorithm \ref{alg:2}, we have that for at least $\left(1-\cfrac{2e\tau n}{\sigma}\right)$ fraction of $r\in[m/k]$ of the expert $s\in [k]$:\\\\
    \centerline{$\cfrac{\partial \mathcal{L}(\theta^{(t)},x,y)}{\partial w_{r,s}^{(t)}}=\cfrac{\overset{\sim}{\partial} \mathcal{L}(\theta^{(t)},x,y)}{\partial w_{r,s}^{(t)}}$ and $|\langle w_{r,s}^{(t)}, x^{(j)}\rangle|\ge\tau, \forall j\in [n]$}
\end{lemma}
\begin{proof}
    Using similar argument as in Lemma \ref{a_lemma_4} and as $\sum_{j\in J_s(w_s^{(t)},x)}G_{j,s}(w_s^{(t)},x)=1$ w.h.p. $\left|\left|\frac{\partial \mathcal{L}(\theta^{(t)},x,y)}{\partial w_{r,s}^{(t)}}\right|\right|=\Tilde{O}\left(\frac{1}{l}\right)$ so as the mini-batch gradient, $\left|\left|\frac{\partial \mathcal{L}(\theta^{(t)})}{\partial w_{r,s}^{(t)}}\right|\right|=\Tilde{O}\left(\frac{1}{l}\right)$.\\\\
    Therefore, $\left|\left|w_{r,s}^{(t)}-w_{r,s}^{(0)}\right|\right|=\Tilde{O}\left(\frac{\eta t}{l}\right)$.\\\\
    Now, for every $\tau > 0,$ considering the set $\mathcal{H}_s:=\left\{ r\in[m/k]: \forall j\in [n], |\langle w_{r,s}^{(0)},x^{(j)}\rangle|\ge2\tau\right\}$ and following the same procedure as in Lemma \ref{a_lemma_4} we can complete the proof.\\\\
\end{proof}
\begin{lemma}\label{a_lemma_9}
    For the expert $s\in[k]$ and any possible fixed set $\{v^{(t)}(\theta^{(t)},x,y)G_{j,s}(w_s^{(t)},x):(x,y)\sim\mathcal{D},j\in J_s(w_s^{(t)},x)\}$ (that does not depend on $w_{r,s}^{(0)}$) such that $v_s^{(t)}=v_{1,s}^{(t)}=\max\{v_{1,s}^{(t)},v_{2,s}^{(t)}\}$, we have:\\\\
    \centerline{$\mathbb{P}\left[||P_{r,s}^{(t)}||=\overset{\sim}{\Omega}\left(\cfrac{v_s^{(t)}}{lp\sqrt{\delta}}\right)\right]=\Omega\left(\cfrac{1}{p\sqrt{\delta}}\right)$}
\end{lemma}
\begin{proof}
    We know that,
    \begin{align*}
        P_{r,s}^{(t)}&=\mathbb{E}_{\mathcal{D}|y=+1}\left[v^{(t)}(\theta^{(t)},x,y)\left(\cfrac{1}{l}\underset{j\in J_s(w_s^{(t)},x)}{\sum}G_{j,s}(w_s^{(t)},x)x^{(j)}1_{\langle w_{r,s}^{(0)},x^{(j)}\rangle\ge0}\right)\Big\vert y=+1\right]\\
        &\hspace{1cm}-\mathbb{E}_{\mathcal{D}|y=-1}\left[v^{(t)}(\theta^{(t)},x,y)\left(\cfrac{1}{l}\underset{j\in J_s(w_s^{(t)},x)}{\sum}G_{j,s}(w_s^{(t)},x)x^{(j)}1_{\langle w_{r,s}^{(0)},x^{(j)}\rangle\ge0}\right)\Big\vert y=-1\right]
    \end{align*}
    Therefore,
    \begin{align*}
        &h(w_{r,s}^{(0)}):=\langle P_{r,s}, w_{r,s}^{(0)}\rangle\\
        &=\mathbb{E}_{\mathcal{D}|y=+1}\left[v^{(t)}(\theta^{(t)},x,y)\left(\cfrac{1}{l}\underset{j\in J_s(w_s^{(t)},x)}{\sum}G_{j,s}(w_s^{(t)},x)\textbf{ReLU}\left(\langle w_{r,s}^{(0)},x^{(j)}\rangle\right)\right)\Big\vert y=+1\right]\\
        &\hspace{1cm}-\mathbb{E}_{\mathcal{D}|y=-1}\left[v^{(t)}(\theta^{(t)},x,y)\left(\cfrac{1}{l}\underset{j\in J_s(w_s^{(t)},x)}{\sum}G_{j,s}(w_s^{(t)},x)\textbf{ReLU}\left(\langle w_{r,s}^{(0)},x^{(j)}\rangle\right)\right)\Big\vert y=-1\right]
    \end{align*}
    Now, decomposing $w_{r,s}^{(0)}=\alpha o_1 +\beta$ with $\beta\perp o_1$ we get,
    \begin{align*}
        &h(w_{r,s}^{(0)})=\cfrac{v_{1,s}^{(t)}}{l}\textbf{ReLU}\left(\alpha\right)\\
        &+\mathbb{E}_{\mathcal{D}|y=+1,\mathcal{E}_{1,s}^{(t)}}\left[p_{1,s}^{(t)}v^{(t)}(\theta^{(t)},x,y)\left(\cfrac{1}{l}\underset{j\in J_s(x)/j_{o_1}}{\sum}G_{j,s}\textbf{ReLU}\left(\alpha\langle o_1,x^{(j)}\rangle+\langle \beta,x^{(j)}\rangle\right)\right)\right]\\
        &+\mathbb{E}_{\mathcal{D}|y=+1,\neg\mathcal{E}_{1,s}^{(t)}}\left[(1-p_{1,s}^{(t)})v^{(t)}(\theta^{(t)},x,y)\left(\cfrac{1}{l}\underset{j\in J_s(x)}{\sum}G_{j,s}\textbf{ReLU}\left(\alpha\langle o_1,x^{(j)}\rangle+\langle \beta,x^{(j)}\rangle\right)\right)\right]\\
        &-\mathbb{E}_{\mathcal{D}|y=-1}\left[v^{(t)}(\theta^{(t)},x,y)\left(\cfrac{1}{l}\underset{j\in J_s(x)}{\sum}G_{j,s}\textbf{ReLU}\left(\alpha\langle o_1,x^{(j)}\rangle+\langle \beta,x^{(j)}\rangle\right)\right)\right]\\
        &=\phi(\alpha)-l(\alpha)
    \end{align*}
    where,
    \begin{align*}
        &\phi(\alpha):=\cfrac{v_{1,s}^{(t)}}{l}\textbf{ReLU}\left(\alpha\right)\\
        &+\mathbb{E}_{\mathcal{D}|y=+1,\mathcal{E}_{1,s}^{(t)}}\left[p_{1,s}^{(t)}v^{(t)}(\theta^{(t)},x,y)\left(\cfrac{1}{l}\underset{j\in J_s(x)/j_{o_1}}{\sum}G_{j,s}\textbf{ReLU}\left(\alpha\langle o_1,x^{(j)}\rangle+\langle \beta,x^{(j)}\rangle\right)\right)\right]\\
        &+\mathbb{E}_{\mathcal{D}|y=+1,\neg\mathcal{E}_{1,s}^{(t)}}\left[(1-p_{1,s}^{(t)})v^{(t)}(\theta^{(t)},x,y)\left(\cfrac{1}{l}\underset{j\in J_s(x)}{\sum}G_{j,s}\textbf{ReLU}\left(\alpha\langle o_1,x^{(j)}\rangle+\langle \beta,x^{(j)}\rangle\right)\right)\right]
    \end{align*}
    and
    \begin{align*}
        &l(\alpha):=\mathbb{E}_{\mathcal{D}|y=-1}\left[v^{(t)}(\theta^{(t)},x,y)\left(\cfrac{1}{l}\underset{j\in J_s(x)}{\sum}G_{j,s}\textbf{ReLU}\left(\alpha\langle o_1,x^{(j)}\rangle+\langle \beta,x^{(j)}\rangle\right)\right)\right]
    \end{align*}
    Now as $\phi(\alpha)$ and $l(\alpha)$ both are convex functions, using the same procedure as in Lemma \ref{a_lemma_4} we can complete the proof.
\end{proof}
\begin{lemma}\label{a_lemma_10}
    Let $v_s^{(t)}=\max\{v_{1,s}^{(t)},v_{2,s}^{(t)}\}$. Then, for every $v_s^{(t)}>0$, for $m=\Tilde{\Omega}\left(\cfrac{klp^3\delta^{3/2}}{(v_s^{(t)})^2}\right)$, for every possible set $\{v^{(t)}(\theta^{(t)},x,y)G_{j,s}(w_s^{(t)},x):(x,y)\sim\mathcal{D},j\in J_s(w_s^{(t)},x)\}$ (that depends on $w_{r,s}^{(0)}$), there exist at least $\Omega\left(\frac{1}{p\sqrt{\delta}}\right)$ fraction of $r\in [m/k]$ of the expert $s\in[k]$ such that,\\\\
    \centerline{$\left|\left|\cfrac{\overset{\sim}{\partial}\Hat{\mathcal{L}}(\theta^{(t)})}{\partial w_{r,s}^{(t)}}\right|\right|=\Tilde{\Omega}\left(\cfrac{v_s^{(t)}}{lp\sqrt{\delta}}\right)$}
\end{lemma}
\begin{proof}
    Let us pick $S$ samples to form $\textbf{S}=\{(x_i,y_i)\}_{i=1}^S$ with $S/2$ many samples from $y=+1$ such that $\frac{1}{2}p_{1,s}^{(t)}S$ many samples of them satisfy the event $\mathcal{E}_{1,s}^{(t)}$ and $S/2$ many samples from $y=-1$ such that $\frac{1}{2}p_{2,s}^{(t)}S$ many samples of them satisfy the event $\mathcal{E}_{2,s}^{(t)}$. We denote the subset of $\textbf{S}$ satisfying the event $\mathcal{E}_{1,s}^{(t)}$ by $\textbf{S}_{1}$ and the subset of $\textbf{S}$ satisfying the event $\mathcal{E}_{2,s}^{(t)}$ by $\textbf{S}_{2}$. Therefore, $\left|\textbf{S}_{1}\right|=\frac{1}{2}p_{1,s}^{(t)}S$ and $\left|\textbf{S}_{2}\right|=\frac{1}{2}p_{2,s}^{(t)}S$. Now, w.h.p. :
    \begin{align*}
        &\left|v_{1,s}^{(t)}-\cfrac{2}{p_{1,s}^{(t)}S}\underset{(x_i,y_i)\in\textbf{S}_{1}}{\sum}p_{1,s}^{(t)}G_{j_{o_1},s}^{(t)}(x_i)v^{(t)}(\theta^{(t)},x_i,y_i)\right|=\Tilde{O}\left(\cfrac{1}{\sqrt{p_{1,s}^{(t)}S}}\right)\text{ and }\\
        &\left|v_{2,s}^{(t)}-\cfrac{2}{p_{2,s}^{(t)}S}\underset{(x_i,y_i)\in\textbf{S}_{2}}{\sum}p_{2,s}^{(t)}G_{j_{o_2},s}^{(t)}(x_i)v^{(t)}(\theta^{(t)},x_i,y_i)\right|=\Tilde{O}\left(\cfrac{1}{\sqrt{p_{2,s}^{(t)}S}}\right)
    \end{align*}
    This implies that, as long as $S=\Tilde{\Omega}\left(\cfrac{1}{(v_s^{(t)})^2}\right)$, we will have that,
    \begin{align*}
        &\max\left\{\cfrac{2}{p_{1,s}^{(t)}S}\underset{(x_i,y_i)\in\textbf{S}_{1}}{\sum}p_{1,s}^{(t)}G_{j_{o_1},s}^{(t)}(x_i)v^{(t)}(\theta^{(t)},x_i,y_i),\right.\\
        &\left.\hspace{2cm}\cfrac{2}{p_{2,s}^{(t)}S}\underset{(x_i,y_i)\in\textbf{S}_{2}}{\sum}p_{2,s}^{(t)}G_{j_{o_2},s}^{(t)}(x_i)v^{(t)}(\theta^{(t)},x_i,y_i)\right\}\in\left[\cfrac{1}{2}v_s^{(t)},\cfrac{3}{2}v_s^{(t)}\right]
    \end{align*}
    Now using the same procedure as in Lemma \ref{a_lemma_6} and using Lemma \ref{a_lemma_9} we can show that, for a fixed set $\{v^{(t)}(\theta^{(t)},x_i,y_i)G_{j,s}(w_s^{(t)},x_i):(x_i,y_i)\in\textbf{S},j\in J_s(w_s^{(t)},x_i)\}$ as long as $S=\Tilde{\Omega}\left(\cfrac{1}{(v_s^{(t)})^2}\right)$, the probability that there are less than $O\left(\frac{1}{p\sqrt{\delta}}\right)$ fraction of $r\in[m/k]$ such that $\left|\left|\cfrac{1}{S}\underset{(x_i,y_i)\in \textbf{S}}{\sum}\cfrac{\overset{\sim}{\partial}\mathcal{L}(\theta^{(t)},x_i,y_i)}{\partial w_{r,s}^{(t)}}\right|\right|$ is $\Tilde{\Omega}\left(\cfrac{v_s^{(t)}}{lp\sqrt{\delta}}\right)$ is no more than $p_{\text{fix}}$ where, $p_{\text{fix}}\le \exp{\left(-\Omega\left(\frac{m}{kp\sqrt{\delta}}\right)\right)}$.\\\\
    Now, for every $\bar{\varepsilon}>0$, for two different $\{v^{(t)}(\theta^{(t)},x_i,y_i)G_{j,s}(w_s^{(t)},x_i):(x_i,y_i)\in\textbf{S},j\in J_s(w_s^{(t)},x_i)\}$, $\{v^{\prime(t)}(\theta^{(t)},x_i,y_i)G_{j,s}^\prime(w_s^{(t)},x_i):(x_i,y_i)\in\textbf{S},j\in J_s(w_s^{(t)},x_i)\}$ such that $\forall (x_i,y_i)\in \textbf{S}, j\in J_s(w_s^{(t)},x_i)$, $|v^{(t)}(\theta^{(t)},x_i,y_i)G_{j,s}(w_s^{(t)},x_i)-v^{\prime(t)}(\theta^{(t)},x_i,y_i)G_{j,s}^\prime(w_s^{(t)},x_i)|\le\bar{\varepsilon}$, w.h.p.,
    \begin{align*}
        &\left|\left|\cfrac{1}{S}\underset{(x_i,y_i)\in \textbf{S}}{\sum}\cfrac{-ya_{r,s}}{l}\underset{j\in J_s(w_s^{(t)},x_i)}{\sum}\left(v^{(t)}(\theta^{(t)},x_i,y_i)G_{j,s}(w_s^{(t)},x_i)\right.\right.\right.\\
        &\left.\left.\left.\hspace{5cm}-v^{\prime(t)}(\theta^{(t)},x_i,y_i)G_{j,s}^\prime(w_s^{(t)},x_i)\right)x_i^{(j)}1_{\langle w_{r,s}^{(0)},x_i^{(j)}\rangle\ge0}\right|\right|=\Tilde{O}(\bar{\varepsilon})
    \end{align*}
    Therefore taking $\Bar{\varepsilon}$-net with $\bar{\varepsilon}=\Tilde{\Theta}\left(\cfrac{v_s^{(t)}}{lp\sqrt{\delta}}\right)$ we can show that the probability that there exists $\{v^{(t)}(\theta^{(t)},x_i,y_i)G_{j,s}(w_s^{(t)},x_i):(x_i,y_i)\in\textbf{S},j\in J_s(w_s^{(t)},x_i)\}$ such that there are no more than $O\left(\frac{1}{p\sqrt{\delta}}\right)$ fraction of $r\in [m/k]$ with $\left|\left|\cfrac{1}{S}\underset{(x_i,y_i)\in \textbf{S}}{\sum}\cfrac{\overset{\sim}{\partial}\mathcal{L}(\theta^{(t)},x_i,y_i)}{\partial w_{r,s}^{(t)}}\right|\right|=\Tilde{\Omega}\left(\cfrac{v_s^{(t)}}{lp\sqrt{\delta}}\right)$ is no more than,
    $p\le p_{\text{fix}}\left(\cfrac{v_s^{(t)}}{\bar{\varepsilon}}\right)^{Sl}\le \exp{\left(-\Omega\left(\frac{m}{kp\sqrt{\delta}}\right)+Sl\log{\left(\cfrac{v_s^{(t)}}{\bar{\varepsilon}}\right)}\right)}$.\\\\
    Hence, for $m=\overset{\sim}{\Omega}\left(kSlp\sqrt{\delta}\right)$ with $S=\Tilde{\Omega}\left(\cfrac{1}{(v_s^{(t)})^2}\right)$, w.h.p. for every possible choice of $\{v^{(t)}(\theta^{(t)},x_i,y_i)G_{j,s}(w_s^{(t)},x_i):(x_i,y_i)\in\textbf{S},j\in J_s(w_s^{(t)},x_i)\}$, there are at least $\Omega\left(\frac{1}{p\sqrt{\delta}}\right)$ fraction of $r\in[m/k]$ such that,
    \begin{align*}
        \left|\left|\cfrac{1}{S}\underset{(x_i,y_i)\in \textbf{S}}{\sum}\cfrac{\overset{\sim}{\partial}\mathcal{L}(\theta^{(t)},x_i,y_i)}{\partial w_{r,s}^{(t)}}\right|\right|=\Tilde{\Omega}\left(\cfrac{v_s^{(t)}}{lp\sqrt{\delta}}\right)
    \end{align*}
    Now as $\left|\left|\cfrac{\overset{\sim}{\partial}\mathcal{L}(\theta^{(t)},x_i,y_i)}{\partial w_{r,s}^{(t)}}\right|\right|=\overset{\sim}{O}(1/l)$, using the same procedure as in Lemma \ref{a_lemma_6} we can complete the proof which gives us $m=\Tilde{\Omega}\left(\cfrac{klp^3\delta^{3/2}}{(v_s^{(t)})^2}\right)$.
\end{proof}
\begin{lemma}\label{a_lemma_11}
    Let us define $v^{(t)}:=\sqrt{\underset{s\in[k]}{\sum}(v_s^{(t)})^2}$ where $v_s^{(t)}=\max\{v_{1,s}^{(t)},v_{2,s}^{(t)}\}$ for all $s\in[k]$; $\gamma:=\Omega\left(\frac{1}{p\sqrt{\delta}}\right)$. 
    Then, by selecting learning rate $\eta=\Tilde{O}\left(\cfrac{\gamma^3(v^{(t)})^2l^3}{mk^2}\right)$ and batch size $B=\Tilde{\Omega}\left(\cfrac{k^2}{\gamma^6(v^{(t)})^4}\right)$, at each iteration $t$ of the Step-2 of Algorithm \ref{alg:2} such that $t=\Tilde{O}\left(\cfrac{\sigma\gamma^3(v^{(t)})^2l^2}{\eta nk}\right)$, w.h.p. we can ensure that,
    \begin{align*}
        \Delta L(\theta^{(t)},\theta^{(t+1)})\ge \cfrac{\eta m \gamma^3}{l^2} \Tilde{\Omega}\left((v^{(t)})^2\right)
    \end{align*}
\end{lemma}
\begin{proof}
    As w.h.p. $\left|\left|\cfrac{\Tilde{\partial}\mathcal{L}(\theta^{(t)},x,y)}{\partial w_{r,s}^{(t)}}\right|\right|=\overset{\sim}{O}(1/l)$, for a randomly sampled batch $\mathcal{B}_t$ of size $B$, by selecting $\tau=\cfrac{\sigma\gamma}{4enB}$ in Lemma \ref{a_lemma_8} and using the same procedure as in Lemma \ref{a_lemma_7}, we can show that for at least $\gamma/2$ fraction of $r\in[m/k]$ of expert $s\in[k]$:
    \begin{align*}
        \left|\left|\cfrac{\partial\mathcal{L}(\theta^{(t)})}{\partial w_{r,s}^{(t)}}\right|\right|=\Tilde{\Omega}\left(\cfrac{v_s^{(t)}}{lp\sqrt{\delta}}\right)
    \end{align*}
    Now, for any $(x^\prime,y^\prime)\sim\mathcal{D}$, from Lemma \ref{a_lemma_8} we know that for at least $1-\cfrac{2e\tau n}{\sigma}$ fraction of $r\in[m/k]$ of any expert $s\in[k]$, the loss function is $\Tilde{O}(1/l)$-Lipschitz smooth and also $\Tilde{O}(1/l)$-smooth.\\\\
    Therefore, using same procedure as in Lemma \ref{a_lemma_7} we can complete the proof.
\end{proof}

\section{Lemmas Used to Prove the Theorem \ref{Thm_single_cnn}}\label{aux_cnn_thm_proof}
For the single CNN model, as all the patches of an input $(x,y)\sim\mathcal{D}$ are sent to the model (i.e., there is no router), the gradient of the single sample loss function w.r.t. hidden node $r\in[m]$,
\begin{align}\label{grad_cnn}
    \frac{\partial \mathcal{L}(\theta^{(t)},x,y)}{\partial w_{r}^{(t)}}=-ya_{r}v^{(t)}(\theta^{(t)},x,y)\left(\cfrac{1}{n}\underset{j\in [n]}{\sum}x^{(j)}1_{\langle w_{r}^{(t)},x^{(j)}\rangle\ge0}\right)
\end{align}
the corresponding \textbf{pseudo-gradient},
\begin{align*}
    \frac{\partial \mathcal{L}(\theta^{(t)},x,y)}{\partial w_{r}^{(t)}}=-ya_{r}v^{(t)}(\theta^{(t)},x,y)\left(\cfrac{1}{n}\underset{j\in [n]}{\sum}x^{(j)}1_{\langle w_{r}^{(0)},x^{(j)}\rangle\ge0}\right)
\end{align*}
and the expected pseudo-gradient,
\begin{align*}
    \frac{\overset{\sim}{\partial}\Hat{\mathcal{L}}(\theta^{(t)})}{\partial w_{r}^{(t)}}&=\mathbb{E}_{\mathcal{D}}\left[\frac{\overset{\sim}{\partial} \mathcal{L}(\theta^{(t)},x,y)}{\partial w_{r}^{(t)}}\right]\\
    &=-\cfrac{a_{r}}{2}\left(\mathbb{E}_{\mathcal{D}|y=+1}\left[v^{(t)}(\theta^{(t)},x,y)\left(\cfrac{1}{n}\underset{j\in [n]}{\sum}x^{(j)}1_{\langle w_{r}^{(0)},x^{(j)}\rangle\ge0}\right)\Big\vert y=+1\right]\right.\\
    &\left.\hspace{1.5cm}-\mathbb{E}_{\mathcal{D}|y=-1}\left[v^{(t)}(\theta^{(t)},x,y)\left(\cfrac{1}{n}\underset{j\in [n]}{\sum}x^{(j)}1_{\langle w_{r}^{(0)},P_jx\rangle\ge0}\right)\Big\vert y=-1\right]\right)\\
    &=-\cfrac{a_{r}}{2}P_{r}^{(t)}
\end{align*}
where,
\begin{align*}
    P_{r}^{(t)}&=\mathbb{E}_{\mathcal{D}|y=+1}\left[v^{(t)}(\theta^{(t)},x,y)\left(\cfrac{1}{n}\underset{j\in [n])}{\sum}x^{(j)}1_{\langle w_{r}^{(0)},x^{(j)}\rangle\ge0}\right)\Big\vert y=+1\right]\\
    &\hspace{1.5cm}-\mathbb{E}_{\mathcal{D}|y=-1}\left[v^{(t)}(\theta^{(t)},x,y)\left(\cfrac{1}{n}\underset{j\in [n]}{\sum}x^{(j)}1_{\langle w_{r}^{(0)},x^{(j)}\rangle\ge0}\right)\Big\vert y=-1\right]
\end{align*}
\begin{lemma}\label{a_lemma_12}
    W.h.p. over the random initialization, for every $(x,y)\sim\mathcal{D}$ and for every $\tau>0$, for every iteration $t=\Tilde{O}\left(\cfrac{\tau}{\eta}\right)$ of the minibatch SGD, we have that for at least $\left(1-\cfrac{2e\tau n}{\sigma}\right)$ fraction of $r\in[m]$:\\\\
    \centerline{$\cfrac{\partial \mathcal{L}(\theta^{(t)},x,y)}{\partial w_{r}^{(t)}}=\cfrac{\overset{\sim}{\partial} \mathcal{L}(\theta^{(t)},x,y)}{\partial w_{r}^{(t)}}$ and $|\langle w_{r}^{(t)}, x^{(j)}\rangle|\ge\tau, \forall j\in [n]$}
\end{lemma}
\begin{proof}
    Using similar argument as in Lemma \ref{a_lemma_4} we can show that w.h.p., $\left|\left|\frac{\partial \mathcal{L}(\theta^{(t)},x,y)}{\partial w_{r}^{(t)}}\right|\right|=\Tilde{O}\left(1\right)$ so as the mini-batch gradient, $\left|\left|\frac{\partial \mathcal{L}(\theta^{(t)})}{\partial w_{r}^{(t)}}\right|\right|=\Tilde{O}\left(1\right)$.\\\\
    Therefore, $\left|\left|w_{r}^{(t)}-w_{r}^{(0)}\right|\right|=\Tilde{O}\left(1\right)$.\\\\
    Now, for every $\tau > 0,$ considering the set $\mathcal{H}:=\left\{ r\in[m]: \forall j\in [n], |\langle w_{r}^{(0)},x^{(j)}\rangle|\ge2\tau\right\}$ and following the same procedure as in Lemma \ref{a_lemma_4} we can complete the proof.\\\\
\end{proof}
Recall, $v_1^{(t)}:=\mathbb{E}_{\mathcal{D}|y=+1}[v^{(t)}(\theta^{(t)},x,y)|y=+1]$ and $v_2^{(t)}:=\mathbb{E}_{\mathcal{D}|y=-1}[v^{(t)}(\theta^{(t)},x,y)|y=-1]$.
\begin{lemma}\label{a_lemma_13}
    For any possible fixed set $\{v^{(t)}(\theta^{(t)},x,y):(x,y)\sim\mathcal{D}\}$ (that does not depend on $w_{r}^{(0)}$) such that $v^{(t)}=v_1^{(t)}=\max\{v_1^{(t)},v_2^{(t)}\}$, we have:\\\\
    \centerline{$\mathbb{P}\left[||P_{r}^{(t)}||=\overset{\sim}{\Omega}\left(\cfrac{v^{(t)}}{np\sqrt{\delta}}\right)\right]=\Omega\left(\cfrac{1}{p\sqrt{\delta}}\right)$}
\end{lemma}
\begin{proof}
     We know that,
    \begin{align*}
        P_{r}^{(t)}&=\mathbb{E}_{\mathcal{D}|y=+1}\left[v^{(t)}(\theta^{(t)},x,y)\left(\cfrac{1}{n}\underset{j\in [n]}{\sum}x^{(j)}1_{\langle w_{r}^{(0)},x^{(j)}\rangle\ge0}\right)\Big\vert y=+1\right]\\
        &\hspace{1cm}-\mathbb{E}_{\mathcal{D}|y=-1}\left[v^{(t)}(\theta^{(t)},x,y)\left(\cfrac{1}{n}\underset{j\in [n]}{\sum}x^{(j)}1_{\langle w_{r}^{(0)},x^{(j)}\rangle\ge0}\right)\Big\vert y=-1\right]
    \end{align*}
    Therefore,
    \begin{align*}
        &h(w_{r}^{(0)}):=\langle P_{r}, w_{r}^{(0)}\rangle\\
        &=\mathbb{E}_{\mathcal{D}|y=+1}\left[v^{(t)}(\theta^{(t)},x,y)\left(\cfrac{1}{n}\underset{j\in [n]}{\sum}\textbf{ReLU}\left(\langle w_{r}^{(0)},x^{(j)}\rangle\right)\right)\Big\vert y=+1\right]\\
        &\hspace{1cm}-\mathbb{E}_{\mathcal{D}|y=-1}\left[v^{(t)}(\theta^{(t)},x,y)\left(\cfrac{1}{n}\underset{j\in [n]}{\sum}\textbf{ReLU}\left(\langle w_{r}^{(0)},x^{(j)}\rangle\right)\right)\Big\vert y=-1\right]
    \end{align*}
    Now, decomposing $w_{r}^{(0)}=\alpha o_1 +\beta$ with $\beta\perp o_1$ we get,
    \begin{align*}
        &h(w_{r}^{(0)})=\cfrac{v^{(t)}}{n}\textbf{ReLU}\left(\alpha\right)\\
        &+\mathbb{E}_{\mathcal{D}|y=+1}\left[v^{(t)}(\theta^{(t)},x,y)\left(\cfrac{1}{n}\underset{j\in [n]/j_{o_1}}{\sum}\textbf{ReLU}\left(\alpha\langle o_1,x^{(j)}\rangle+\langle \beta,x^{(j)}\rangle\right)\right)\right]\\
        &-\mathbb{E}_{\mathcal{D}|y=-1}\left[v^{(t)}(\theta^{(t)},x,y)\left(\cfrac{1}{n}\underset{j\in [n]}{\sum}\textbf{ReLU}\left(\alpha\langle o_1,x^{(j)}\rangle+\langle \beta,x^{(j)}\rangle\right)\right)\right]\\
        &=\phi(\alpha)-l(\alpha)
    \end{align*}
    where,
    \begin{align*}
        &\phi(\alpha):=\cfrac{v^{(t)}}{n}\textbf{ReLU}\left(\alpha\right)\\
        &\hspace{1.2cm}+\mathbb{E}_{\mathcal{D}|y=+1}\left[v^{(t)}(\theta^{(t)},x,y)\left(\cfrac{1}{n}\underset{j\in [n]/j_{o_1}}{\sum}\textbf{ReLU}\left(\alpha\langle o_1,x^{(j)}\rangle+\langle \beta,x^{(j)}\rangle\right)\right)\right]
    \end{align*}
    and
    \begin{align*}
        &l(\alpha):=\mathbb{E}_{\mathcal{D}|y=-1}\left[v^{(t)}(\theta^{(t)},x,y)\left(\cfrac{1}{n}\underset{j\in [n]}{\sum}\textbf{ReLU}\left(\alpha\langle o_1,x^{(j)}\rangle+\langle \beta,x^{(j)}\rangle\right)\right)\right]
    \end{align*}
    Now as $\phi(\alpha)$ and $l(\alpha)$ both are convex functions, using the same procedure as in Lemma \ref{a_lemma_4} we can complete the proof.
\end{proof}
\begin{lemma}\label{a_lemma_14}
    Let $v^{(t)}=\max\{v_{1}^{(t)},v_{2}^{(t)}\}$. Then, for every $v^{(t)}>0$, for $m=\Tilde{\Omega}\left(\cfrac{n^2p^3\delta^{3/2}}{(v^{(t)})^2}\right)$, for every possible set $\{v^{(t)}(\theta^{(t)},x,y)(w_s^{(t)},x):(x,y)\sim\mathcal{D}\}$ (that depends on $w_{r}^{(0)}$), there exist at least $\Omega\left(\frac{1}{p\sqrt{\delta}}\right)$ fraction of $r\in [m]$ such that,\\\\
    \centerline{$\left|\left|\cfrac{\overset{\sim}{\partial}\Hat{\mathcal{L}}(\theta^{(t)})}{\partial w_{r}^{(t)}}\right|\right|=\Tilde{\Omega}\left(\cfrac{v^{(t)}}{np\sqrt{\delta}}\right)$}\\\\
\end{lemma}
\begin{proof}
    Similar as in the proof of Lemma \ref{a_lemma_6}, by picking $S$ samples from the distribution $\mathcal{D}$ to form the set $\textbf{S}=\{(x_i,y_i)\}_{i=1}^S$ such that $S/2$ many samples from $y=+1$ (denoting the sub-set by $\textbf{S}_{+1}$) and $S/2$ many samples from $y=-1$ (denoting the sub-set by $\textbf{S}_{-1}$), we can show that w.h.p.,
    \begin{align*}
        &\left|v_1^{(t)}-\cfrac{1}{S/2}\underset{(x_i,y_i)\in\textbf{S}_{+1}}{\sum}v^{(t)}(\theta^{(t)},x_i,y_i)\right|=\Tilde{O}\left(\cfrac{1}{\sqrt{S}}\right)\text{ and }\\
        &\left|v_2^{(t)}-\cfrac{1}{S/2}\underset{(x_i,y_i)\in\textbf{S}_{-1}}{\sum}v^{(t)}(\theta^{(t)},x_i,y_i)\right|=\Tilde{O}\left(\cfrac{1}{\sqrt{S}}\right)
    \end{align*}
    This implies that, as long as $S=\Tilde{\Omega}\left(\cfrac{1}{(v^{(t)})^2}\right)$ we have,
    \begin{align*}
        \max\left\{\cfrac{1}{S/2}\underset{(x_i,y_i)\in\textbf{S}_{+1}}{\sum}v^{(t)}(\theta^{(t)},x_i,y_i),\cfrac{1}{S/2}\underset{(x_i,y_i)\in\textbf{S}_{-1}}{\sum}v^{(t)}(\theta^{(t)},x_i,y_i)\right\}\in\left[\cfrac{1}{2}v^{(t)},\cfrac{3}{2}v^{(t)}\right]
    \end{align*}
    Now using Lemma \ref{a_lemma_13} and following similar procedure as in Lemma \ref{a_lemma_6} we can complete the proof.
\end{proof}
\begin{lemma}\label{a_lemma_15}
    With $v^{(t)}=\max\{v_1^{(t)},v_2^{(t)}\}$ and $\gamma=\Omega\left(\frac{1}{p\sqrt{\delta}}\right)$, by selecting learning rate $\eta=\Tilde{O}\left(\cfrac{\gamma^3(v^{(t)})^2}{mn^2}\right)$ and batch-size $B=\Tilde{\Omega}\left(\cfrac{n^4}{\gamma^6(v^{(t)})^4}\right)$, for $t=\Tilde{O}\left(\cfrac{\sigma\gamma^3(v^{(t)})^2}{\eta n^3}\right)$ iterations of SGD, w.h.p. we can ensure that,
    \begin{align*}
        \Delta L(\theta^{(t)},\theta^{(t+1)})\ge \cfrac{\eta m\gamma^3}{n^2} \Tilde{\Omega}\left((v^{(t)})^2\right)
    \end{align*}
\end{lemma}
\begin{proof}
    Using Lemma \ref{a_lemma_12} and \ref{a_lemma_14} and following similar technique as in Lemma \ref{a_lemma_7}, the proof can be completed.
\end{proof}

\section{Auxiliary Lemmas}
\begin{lemma}\label{convex_li_liang_18}{\citep{li2018learning}}
    Let $\psi: \mathbb{R}\rightarrow\mathbb{R}$ and $\zeta:\mathbb{R}\rightarrow\mathbb{R}$ are convex functions. Let $\{\partial\psi(x)\}$ and $\{\partial\zeta(x)\}$ are the sets of sub-gradient of $\psi$ and $\zeta$ at $x$ respectively such that $\partial_{\text{max}}\psi(x)=\text{max}\{\partial\psi(x)\}$ , $\partial_{\text{max}}\zeta(x)=\text{max}\{\partial\zeta(x)\}$, $\partial_{\text{min}}\psi(x)=\text{min}\{\partial\psi(x)\}$ and $\partial_{\text{min}}\zeta(x)=\text{min}\{\partial\zeta(x)\}$. Then for any $\tau\ge0$ such that $\gamma=(\partial_{\text{max}}\psi(\tau/2)-\partial_{\text{min}}\psi(-\tau/2)-(\partial_{\text{max}}\zeta(\tau/2)-\partial_{\text{min}}\zeta(-\tau/2)$,
    \begin{center}
        $\mathbb{P}_{\alpha\sim U(-\tau,\tau)}\left[|\psi(\alpha)-\zeta(\alpha)|\ge\frac{\tau\gamma}{512}\right]\ge\frac{1}{64}$
    \end{center}
\end{lemma}
\begin{lemma}\label{converge_li_liang_18}\citep{li2018learning}
    Let for any $i\in[m]$, the function $h_i:\mathbb{R}^d\rightarrow\mathbb{R}$ is $L$-Lipschitz smooth and there exists $r\in[m]$ such that for all $i\in[m-r]$ the function $h_i$ is also $L$-smooth. Furthermore, let us assume that the function $g:\mathbb{R}\rightarrow\mathbb{R}$ is both $L$-Lipschitz smooth and $L$-smooth. Let define $f(w):=g\left(\sum_{i\in[m]}h_i(w_i)\right)$ where $w\in\mathbb{R}^{dm}$ such that $w_i\in\mathbb{R}^d$. Then for every $\xi\in\mathbb{R}^{dm}$ such that $\xi_i\in\mathbb{R}^d$ with $\left|\left|\xi_i\right|\right|\le\rho$, we have:
    \begin{align*}
        g\left(\sum_{i\in[m]}h_i(w_i+\xi_i)\right)-g\left(\sum_{i\in[m]}h_i(w_i)\right)\le\sum_{i\in[m-r]}\left\langle\cfrac{\partial f(w)}{\partial w_i},\xi_i\right\rangle+L^3m^2\rho^2+L^2r\rho
    \end{align*}
\end{lemma}

\section{Proof of the Non-linear Separability of the Data-model}\label{non_linearly_separable}
\begin{lemma}\label{non_linear_data_model}
    As long as $l^*=\Omega(1)$, the distribution $\mathcal{D}$ is \textbf{NOT} linearly separable.
\end{lemma}
\begin{proof}
    We will prove the Lemma by contradiction.\\\\
    Now, if the distribution, $\mathcal{D}$ is linearly separable, then there exists a hyperplane $h=\left[h^{(1)T}, h^{(2)T}, ..., h^{(n)T}\right]$ with $||h||=1$ (here, $h^{(j)}$ represents the $j$-th patch of the hyperplane for $j\in[n]$) such that,
    \begin{equation}\label{linear_separable_eqn}
        \forall(x_1,y=+1)\sim\mathcal{D} \text{ and } (x_2,y=-1)\sim\mathcal{D},
        x_1^{T}h-x_2^Th\ge0
    \end{equation}
    Now, as the class-discriminative patterns $o_1$ and $o_2$ can occur at any position $j\in[n]$, $||h^{(j)}||^2=\Theta(\frac{1}{n});\hspace{0.1cm}\forall j\in[n]$.\\\\
    Now, $\forall j\in[n]$, we can decompose $h^{(j)}$ as $h^{(j)}=a_jo_1+b_jo_2$.\\\\ Then, $|a_j|=|b_j|=\Theta\left(\cfrac{1}{\sqrt{n(1-\delta_d)}}\right)$, $\forall j\in[n]$ as $||o_1||=||o_2||=1$.\\\\
    Now,
    \begin{align*}
        x_1^Th-x_2^Th=\langle o_1,h^{(j_{o_1})}\rangle-\langle o_2,h^{(j_{o_2})}\rangle+\sum_{j\in[n]/j_{o_1}}\langle x_1^{(j)},h^{(j)}\rangle-\sum_{j\in[n]/j_{o_2}}\langle x_2^{(j)},h^{(j)}\rangle
    \end{align*}
    Now,
    \begin{align*}
        &\langle o_1,h^{(j_{o_1})}\rangle-\langle o_2,h^{(j_{o_2})}\rangle=(a_{j_{o_1}}-b_{j_{o_2}})-(a_{j_{o_2}}-b_{j_{o_1}})\delta_d\\
        &\le|a_{j_{o_1}}-b_{j_{o_2}}|-|a_{j_{o_1}}-b_{j_{o_2}}|\delta_d\hspace{1cm}\text{[WLOG, let assume $\delta_d<0$]}\\
        &=O\left(\sqrt{\frac{1-\delta_d}{n}}\right)
    \end{align*}
    Now,
    \begin{align*}
        &\sum_{j\in[n]/j_{o_2}}\langle x_2^{(j)},h^{(j)}\rangle-\sum_{j\in[n]/j_{o_1}}\langle x_1^{(j)},h^{(j)}\rangle=\sum_{j\in[n]/j_{o_2}}\langle x_2^{(j)},a_jo_1+b_jo_2\rangle-\sum_{j\in[n]/j_{o_1}}\langle x_1^{(j)},a_jo_1+b_jo_2\rangle\\
        &=O\left(l^*\sqrt{\frac{(1-\delta_d)}{n}}\right)
    \end{align*}
    Therefore, for $l^*=\Omega(1)$ there is contradiction with (\ref{linear_separable_eqn}).
\end{proof}
\clearpage
\section{WRN and WRN-pMoE Architectures Implemented in the Experiments}
\begin{figure}[ht]
\vskip 0.1in
    \begin{minipage}{0.25\linewidth}
        \centering
        \includegraphics[width=0.70\linewidth]{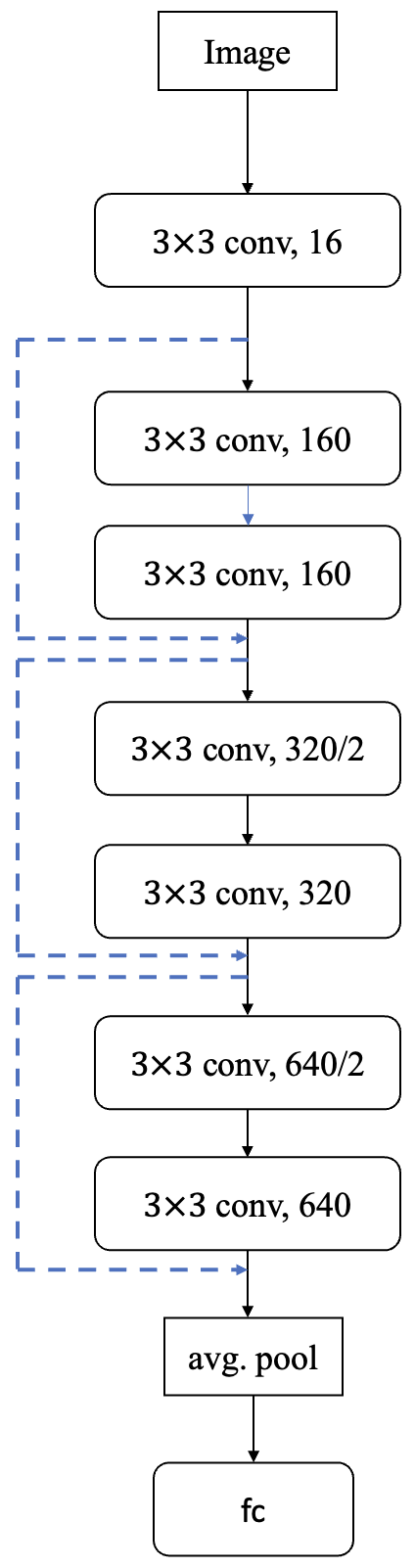}
        \caption{The WRN architecture implemented to learn CelebA dataset}
        \label{wrn}
    \end{minipage}
    ~
    \begin{minipage}{0.75\linewidth}
        \centering
        \includegraphics[width=\linewidth]{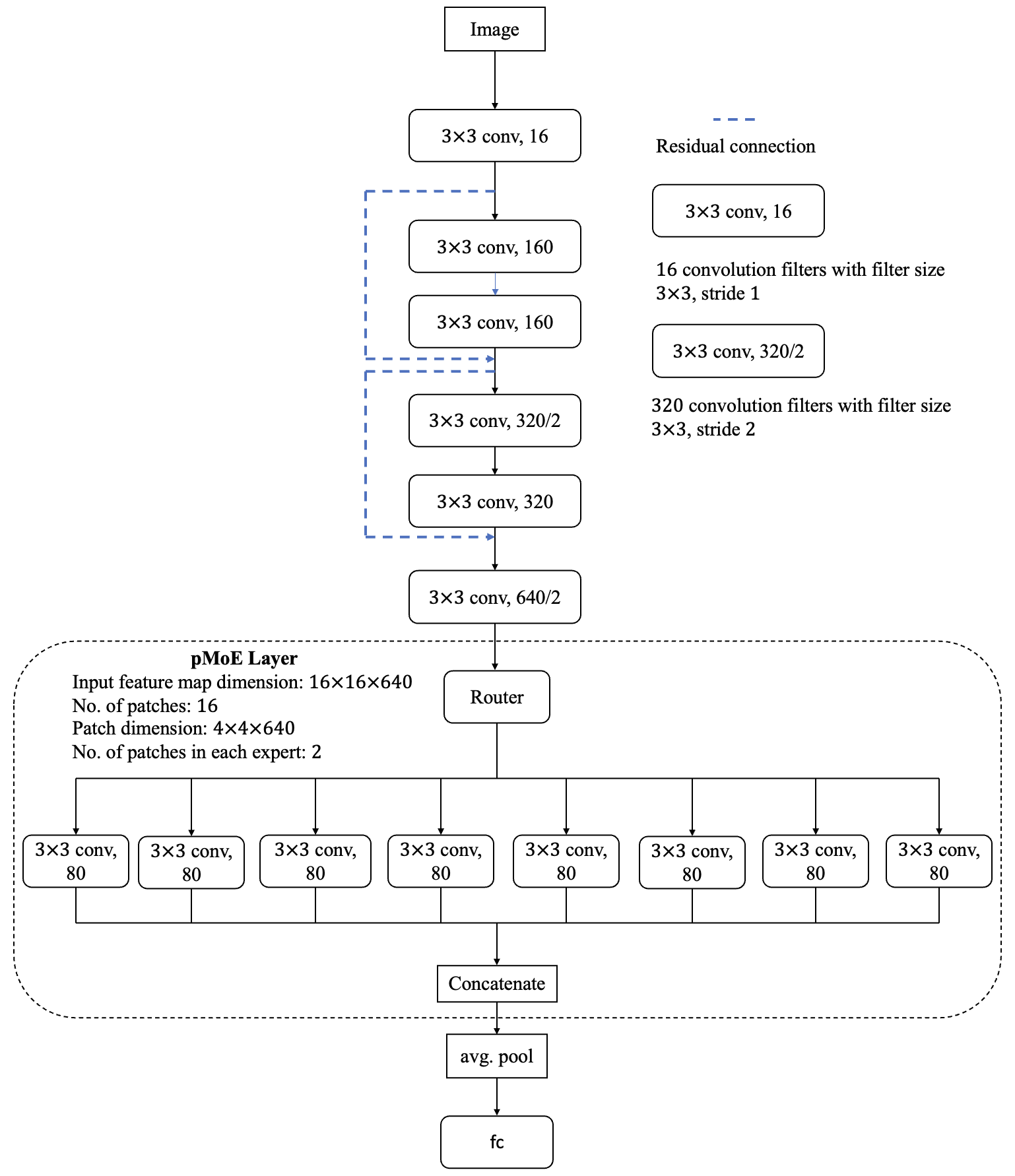}
        \caption{The WRN-pMoE architecture implemented to learn CelebA dataset}
        \label{wrn_pmoe}
    \end{minipage}
    \vskip -0.1in
\end{figure}

\section{Extension to Multi-class Classification}\label{multi_class_extension}

Let us consider $c$-class classification problem where $c>2$. Then, we have $(x,y)\sim\mathcal{D}_c$ where $y\in\{1,2, ..., c\}$ for the multi-class distribution $\mathcal{D}_c$. 

\textbf{The multi-class data model:}\\
Now, according to the data model presented in section \ref{data_model}, we have $\{o_1,o_2, ..., o_c\}$ as class-discriminative pattern set. $\forall j,j^\prime\in[c]$ such that $j\neq j^\prime$, we define $\delta_{d_{j,j^\prime}}:=\langle o_j,o_{j^\prime}\rangle$. We further define $\delta_d:=\max\{\delta_{d_{j,j^\prime}}\}$. Then,
\begin{align*}
    \delta=\cfrac{1}{(1-\max\{\delta^2_{d_{j,j^\prime}},\delta^2_r\}_{j,j^\prime\in[c],j\neq j^\prime})}
\end{align*}
\clearpage
\textbf{The multi-class pMoE model:}\\
The pMoE model for multi-class case is given by,
\begin{equation}\label{multiclass_pmoe}
    \forall i\in[c], f_{M_i}(\theta, x)=\overset{k}{\underset{s=1}{\sum}}\overset{\frac{m}{k}}{\underset{r=1}{\sum}}\cfrac{a_{r,s,i}}{l}\underset{j \in J_s(w_s,x)}{\sum}\textbf{ReLU}(\langle w_{r,s},x^{(j)}\rangle)G_{j,s}(w_s,x)
\end{equation}
An illustration of (\ref{multiclass_pmoe}) is given in Figure \ref{arch_eqn_pmoe_multiclass}.\\
\begin{figure}[h]
\vskip 0.1in
    \centering
    \includegraphics[width=0.5\linewidth]{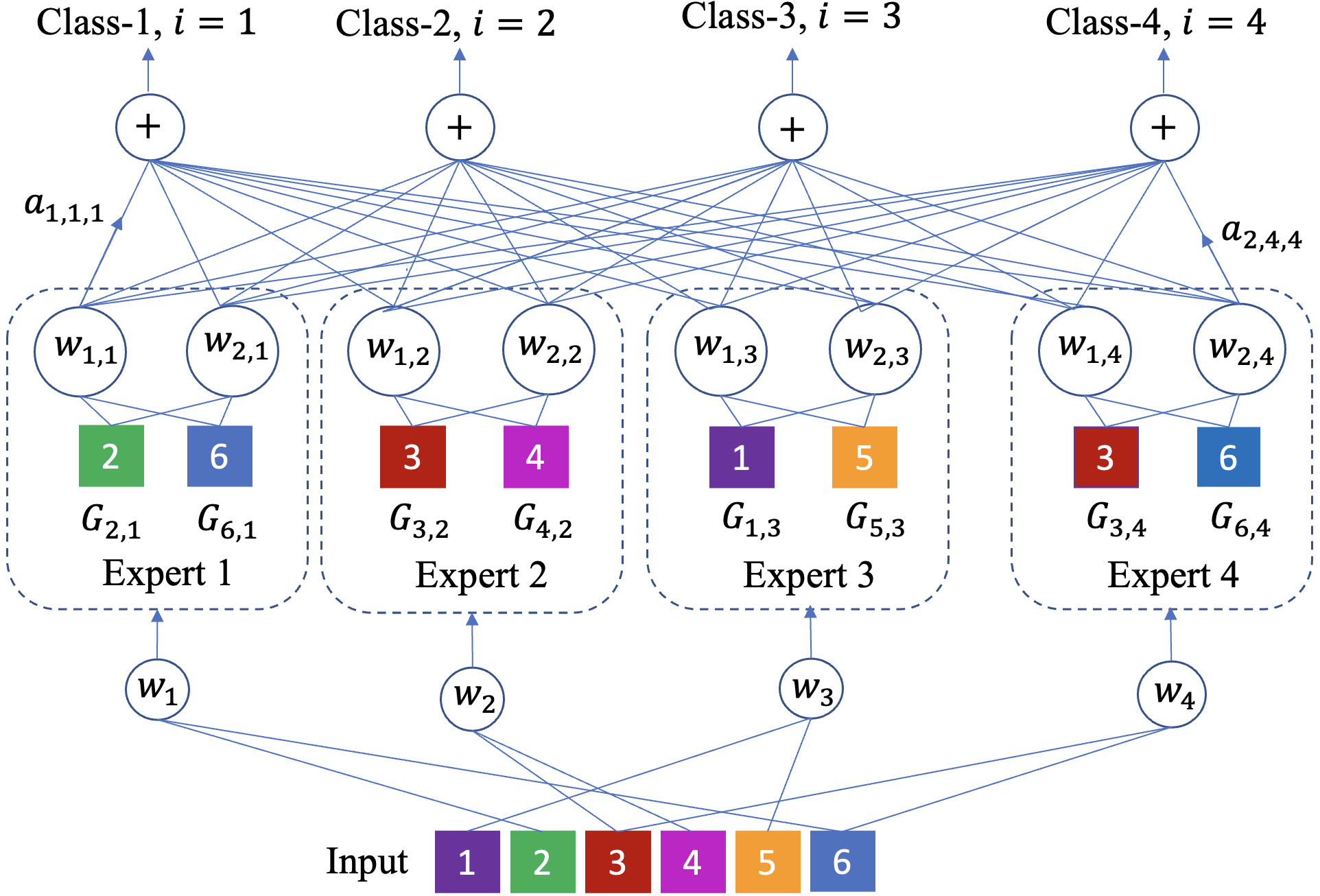}
    \caption{An illustration of the pMoE model in (\ref{multiclass_pmoe}) with $c=4, k=4, m=8, n=6$ and $l=2$.}
    \label{arch_eqn_pmoe_multiclass}
    \vskip -0.1in
\end{figure}

For mult-class case, we replace the logistic loss function by the softmax loss function (also known as cross-entropy loss). For the training dataset $\{x_j,y_j\}_{j=1}^{N}$, we minimize the following empirical risk minimization problem:
\begin{equation}\label{erm_multiclass}
    \displaystyle\underset{\theta}{\text{min}}:\hspace{0.4cm}L(\theta)=\cfrac{1}{N}\overset{N}{\underset{j=1}{\sum}}\log{\cfrac{\sum_{i=1}^c\exp{(f_{M_i}(\theta,x_j))}}{\exp{(f_{M_{y_j}}(\theta,x_j))}}}
\end{equation}

\subsection{The Multi-class Separate-training pMoE}
\textbf{Number of experts:} For the multi-class separate-training pMoE, we take $k=c$, i.e. number of experts is equal to the number of classes.\\\\
\textbf{Training algorithm:}\\
\textbf{Input} : Training data $\{(x_i,y_i)\}_{i=1}^N$, learning rates $\eta_r$ and $\eta$, number of iterations $T_r$ and $T$, batch- \\
\indent\hspace{1.2cm}sizes $B_r$ and $B$\\
\textbf{Step-1}: Initialize $w_s^{(0)}, w_{r,s}^{(0)}, a_{r,s}, \forall s\in\{1,2\}, r\in[m/k]$ according to (\ref{eqn:router_ini}) and (\ref{eqn:expert_ini})\\
\textbf{Step-2:} (Pair-wise router training) We train the router, i.e. the gating-kernels $w_1, w_2, ..., w_c$ using pair-wise training describe below:
\begin{enumerate}
    \item At first, we separate the training set of $N_r$ samples into $c$ disjoint subsets $\{N_{r,1}, N_{r,2}, ..., N_{r,c}\}$ according to the class-labels.
    \item Now, we prepare $c/2$ pairs of training sets $\{(N_{r,1}, N_{r,2}), (N_{r,3}, N_{r,4}), ..., (N_{r,c-1}, N_{r,c})\}$ (here WLOG we assume that $c$ is even).
    \item Under each pair $(N_{r,i}, N_{r,i+1})$, we re-define the label as $y=+1$ and $y=-1$ for the class $i$ and $i+1$ respectively and train the gating-kernels $w_i$ and $w_{i+1}$ by minimizing (\ref{router_erm}) for $T_r$ iterations
    \item After the end of pair-wise training for all the pairs $\{(N_{r,1}, N_{r,2}), (N_{r,3}, N_{r,4}), ..., (N_{r,c-1}, N_{r,c})\}$, we receive $w_1^{(T_r)}, w_2^{(T_r)}, ..., w_c^{(T_r)}$ as the learned gating-kernels.
\end{enumerate}
\textbf{Step-3:}(Expert training)\\
Using the learned gating-kernels $w_1^{(T_r)}, w_2^{(T_r)}, ..., w_c^{(T_r)}$ in Step-2 and using the same procedure as in Step-3 of Algorithm \ref{alg:1} we train the experts.\\

\textbf{The multi-class counterpart of the Lemma \ref{router_lemma}:}\\
Now, using the same proof techniques as for Lemma \ref{router_lemma} (i.e. following same procedures as in section \ref{router_lemma_proof} and \ref{aux_router_lemma_proof}) we can show that, we need $N_r=\Omega(\cfrac{c^2n^2}{(1-\delta_d)^2})$ training samples to ensure,
\begin{align*}
    \text{arg}_{j\in[n]}(x^{(j)}=o_i)\in J_i(w_i^{(T_r)},x)\hspace{1cm}\forall (x,y=i)\sim\mathcal{D}_c\hspace{0.1cm}\text{and } \forall i\in[c]
\end{align*}

\textbf{The multi-class counterpart of the Theorem
\ref{Thm_mcnn_bin}:}\\
We redefine the \textbf{value-function} for each class $i\in[c]$ as,
\begin{gather}\label{multi_class_value_function}
    v_{i,a}^{(t)}(\theta^{(t)},x,y=a):=
    \begin{cases}
    \cfrac{\underset{j\ne a}{\sum}e^{f_{M_j}(\theta^{(t)},x)}}{\overset{c}{\underset{j=1}{\sum}}e^{f_{M_j}(\theta^{(t)},x)}}; \hspace{0.2cm}\text{if, } i=a\\\\
    -\cfrac{e^{f_{M_i}(\theta^{(t)},x)}}{\overset{c}{\underset{j=1}{\sum}}e^{f_{M_j}(\theta^{(t)},x)}}; \hspace{0.2cm} \text{otherwise}
    \end{cases}
\end{gather}
Now using similar techniques as in the proof of Theorem \ref{Thm_mcnn_bin} (i.e. following same procedure as in the proof of Theorem \ref{a_theorem_1} and section \ref{aux_sep_pmoe_thm_proof}) we can show that for every $\epsilon>0$, we need number of hidden nodes $m\ge M_S=\Omega\left(l^{10}p^{12}\delta^6c^{11}\big/\epsilon^{16}\right)$, batch-size $B=\Omega\left(l^4p^6\delta^3c^6\big/\epsilon^{8}\right)$ for $T=O\left(l^4p^6\delta^3c^6\big/\epsilon^{8}\right)$ iterations (i.e. $N_S=\Omega(l^8 p^{12}\delta^6c^{12}/\epsilon^{16})$) to ensure,
\begin{align*}
    \underset{(x,y)\sim\mathcal{D}_c}{\mathbb{P}}\left[\forall j\in[c],j\neq y, f_{M_y}(\theta^{(T)},x)>f_{M_j}(\theta^{(T)},x)\right]\ge1-\epsilon
\end{align*}

\subsection{The Multi-class Joint-training pMoE}
\textbf{Training algorithm:} Same as the Algorithm \ref{alg:2} except that for multi-class case the loss function is softmax instead of logistic loss.\\
\textbf{The multi-class counterpart of the Theorem
\ref{Thm_mcnn_j}:}\\
Using the value-function define in (\ref{multi_class_value_function}) and as long as the Assumption \ref{router_asmptn} satisfied for all the classes $i\in[c]$, following the similar techniques as in the proof of Theorem \ref{Thm_mcnn_j} (i.e. following same procedure as in the proof of Theorem \ref{a_theorem_2} and section \ref{aux_joint_pmoe_thm_proof}), we can show that for every $\epsilon>0$, we need number of hidden nodes $m\ge M_J=\Omega\left(k^3n^2l^{6}p^{12}\delta^6c^8\big/\epsilon^{16}\right)$, batch-size $B=\Omega\left(k^2l^4p^6\delta^3c^4\big/\epsilon^{8}\right)$ for $T=O\left(k^2l^2p^6\delta^3c^4\big/\epsilon^{8}\right)$ iterations (i.e. $N_J=\Omega(k^4l^6p^{12}\delta^6c^8/\epsilon^{16})$) to ensure,
\begin{align*}
    \underset{(x,y)\sim\mathcal{D}_c}{\mathbb{P}}\left[\forall j\in[c],j\neq y, f_{M_y}(\theta^{(T)},x)>f_{M_j}(\theta^{(T)},x)\right]\ge1-\epsilon
\end{align*}

\section{Details of the Results in Table \ref{table1}}\label{table1_details}

\textbf{Complexity in forward pass.} The computational complexity of a non-overlapping convolution operation by a filter of dimension $d$ on an input sample of $n$ patches (of same dimension as the filter) is $O(nd)$ \citep{vaswani2017attention}. Therefore, the complexity of forward pass of a batch of size $B$ through a convolution layer of $m$ neurons is $O(Bmnd)$. Similarly, the forward pass complexity of a the batch through the experts (of same total number of neurons as in the convolution layer) of a pMoE layer is $O(Bmld)$. The operations in a pMoE router includes convolution (with complexity $O(nd)$), softmax operation (with complexity $O(1)$) and TOP-$l$ operation (with complexity $O(nl)$ when $l\ll n$). Therefore, the overall forward pass complexity of a pMoE router with $k$ expert is $O(Bknd)$.

\textbf{Complexity in backward pass.} The gradient of neurons in convolution layer for an input sample is given in (\ref{grad_cnn}), which implies that the complexity of the gradient calculation is $O(nd)$ (addition of $n$ vectors of dimension $d$) and hence the backward pass complexity of CNN is $O(Bmnd)$. Similarly, the backward pass complexity of pMoE experts is $O(Bmld)$. Now the gradient of gating kernels in pMoE router is given in (\ref{grad_router_joint}), which implies that the complexity of the gradient calculation is $O(l^2d)$ (addition of $l^2$ vectors of dimension $d$) and hence the backward pass complexity of pMoE router is $O(Bkl^2d)$.
\begin{align}\label{grad_router_joint}
    \frac{\partial \mathcal{L}(\theta^{(t)},x,y)}{\partial w_{s}^{(t)}}=-yv^{(t)}(\theta^{(t)},x,y)\left(\sum_{r\in[m]}a_{r,s}\left(\cfrac{1}{l}\sum_{j\in J_s(x)}\textbf{ReLU}(\langle w_{r,s}^{(t)},x^{(j)}\rangle)G_{j,s}(\sum_{i\in J_s(x)/j}(x^{(j)}-x^{(i)})G_{i,s})\right)\right)
\end{align}
\textbf{Complexity to achieve $\epsilon$ generalization error.} From Theorem \ref{Thm_mcnn_j}, to achieve $\epsilon$ generalization error we need $O(k^2l^2/\epsilon^8)$ iterations of training in pMoE, which implies that the computational complexity to achieve $\epsilon$ error in pMoE is $O(Bmk^2l^3d/\epsilon^8)$. Similarly, using the results from Theorem \ref{Thm_single_cnn}, the corresponding complexity in CNN is $O(Bmn^5d/\epsilon^8)$.

\end{document}